%% file: main.tex
\documentclass[a4paper,fleqn]{cas-sc}

\usepackage[authoryear,longnamesfirst]{natbib}

\usepackage{color}
\usepackage[utf8]{inputenc}
\usepackage[english]{babel}
\usepackage[T1]{fontenc}
\usepackage{pgfplots}
\usepackage{graphicx}
\graphicspath{{Images/}}
\usepackage{eso-pic} 
\usepackage{caption} 
\usepackage{subcaption}
\usepackage{transparent}
\usepackage{wrapfig} 
\usepackage{amsmath}
\usepackage{amsthm}
\usepackage{bm}
\usepackage[overload]{empheq}
\usepackage{dsfont} 
\usepackage{algorithm2e}
\usepackage{tabularx}
\usepackage{longtable}
\usepackage{colortbl}
\usepackage{multirow}
\usepackage{makecell} 
\usepackage{adjustbox}
\usepackage{pifont}
\usepackage{array}
\usepackage{float}
\usepackage{appendix}
\usepackage{placeins}
\usepackage{enumitem}
\usepackage{amsthm,thmtools}
\usepackage{comment} 
\usepackage{fancyhdr} 
\usepackage{lipsum} 
\usepackage{tcolorbox}
\usepackage{mymacros}

\newcommand{\bea}{\begin{eqnarray}} 
\newcommand{\eea}{\end{eqnarray}}

\newcommand{\mathbbm}[1]{\text{\usefont{U}{bbm}{m}{n}#1}}

\newcommand{\cmark}{\textcolor{brightGreen}{\ding{51}}}
\newcommand{\xmark}{\textcolor{brightRed}{\ding{55}}}

\def\tsc#1{\csdef{#1}{\textsc{\lowercase{#1}}\xspace}}
\tsc{WGM}
\tsc{QE}
\tsc{EP}
\tsc{PMS}
\tsc{BEC}
\tsc{DE}

\begin{document}
\let\WriteBookmarks\relax
\def\floatpagepagefraction{1}
\def\textpagefraction{.001}

\shorttitle{Learning Deterministic Policies with Policy Gradients in Constrained Markov Decision Processes}

\shortauthors{A. Montenegro et~al.}

\title [mode = title]{Learning Deterministic Policies with Policy Gradients in Constrained Markov Decision Processes} 

\allowdisplaybreaks[4]

\pgfplotsset{compat=1.18}

\author[1]{Alessandro Montenegro}[
    orcid=0009-0000-2034-7106
]
\cormark[1]
\ead{alessandro.montenegro@polimi.it}

\affiliation[1]{organization={Politecnico di Milano},
    addressline={Piazza Leonardo Da Vinci 32}, 
    city={Milan},
    postcode={20133}, 
    country={Italy}}

\author[1]{Leonardo Cesani}[
    orcid=0009-0009-9329-5349
]
\ead{leonardo.cesani@polimi.it}

\author[1]{Marco Mussi}[
    orcid=0000-0001-8356-6744
]
\ead{marco.mussi@polimi.it}

\author[1]{Matteo Papini}[
    orcid=0000-0002-3807-3171
]
\ead{matteo.papini@polimi.it}

\author[1]{Alberto Maria Metelli}[
    orcid=0000-0002-3424-5212
]
\ead{albertomaria.metelli@polimi.it}

\cortext[cor1]{Corresponding author}

\input{contents/00_abs}


\maketitle
\input{contents/01_intro}
\input{contents/02_setting}
\input{contents/03_convergence}
\input{contents/04_deterministic}

\input{contents/07_related}
\input{contents/06_exp}
\input{contents/08_concl}

\clearpage

\appendix
\input{apx/proofs}
\input{apx/add_exp}

\clearpage

\bibliographystyle{cas-model2-names}
\bibliography{biblio}

\end{document}

%% file: contents/00_abs.tex
\begin{abstract}
    \emph{Constrained Reinforcement Learning} (CRL) addresses sequential decision-making problems where agents are required to achieve goals by maximizing the expected return while meeting domain-specific constraints. In this setting, \emph{policy-based} methods are widely used thanks to their advantages when dealing with continuous-control problems. These methods search in the policy space with an \emph{action-based} or a \emph{parameter-based} exploration strategy, depending on whether they learn the parameters of a stochastic policy or those of a stochastic hyperpolicy.  
    We introduce an exploration-agnostic algorithm, called \cpg, which enjoys \emph{global last-iterate convergence} guarantees under gradient domination assumptions.
    Furthermore, under specific noise models where the (hyper)policy is expressed as a \emph{stochastic perturbation} of the actions or of the parameters of an underlying deterministic policy, we additionally establish global last-iterate convergence guarantees of \cpg to the \emph{optimal deterministic policy}. This holds when learning a stochastic (hyper)policy and subsequently \emph{switching off the stochasticity} at the \emph{end} of training, thereby deploying a deterministic policy.
    Finally, we empirically validate both the action-based (\cpgae) and parameter-based (\cpgpe) variants of \cpg on constrained control tasks, and compare them against state-of-the-art baselines, demonstrating their effectiveness, in particular when deploying deterministic policies after training.\footnote{This work extends the preliminary version presented in~\citep{montenegro2024constraints} by providing a theoretical analysis and empirical evaluation of deterministic policy deployment.}
\end{abstract}

%% file: contents/01_intro.tex
\section{Introduction} \label{sec:intro}
\noindent When applying Reinforcement Learning~\citep[RL,][]{sutton2018reinforcement} to real-world scenarios, we aim at solving large-scale continuous control problems where, in addition to reaching a goal, it is necessary to meet structural or utility-based constraints. 
For instance, an autonomous-driving car has its main objective of getting to the desired destination (i.e., goal) while avoiding collisions, ensuring the safety of people on the streets, adhering to traffic rules, and respecting the physical requirements of the engine to avoid damaging it (i.e., constraints)~\citep{likmeta2020autonomous}.
To pursue such an objective, it is necessary to extend the RL problem formulation with the possibility to account for constraints.
Constrained Reinforcement Learning~\citep[CRL,][]{uchibe2007constrained} aims at solving this family of problems by employing RL techniques to tackle Constrained Markov Decision Processes~\citep[CMDPs,][]{altman1999constrained}, which provide an established and widely-used framework for modeling constrained control tasks.
The conventional CRL framework primarily focuses on constraints related directly to \emph{expected costs}~\citep{stooke2020responsive,ding2020natural,ying2022dual,ding2024last}. 


Among the RL methods applicable to CMDPs, Policy Gradients~\citep[PGs,][]{deisenroth13survey} are particularly appealing. 
Indeed, PGs have demonstrably achieved impressive results in continuous-control problems due to several advantages that make them well-suited for real-world applications. 
These advantages include the ability to handle continuous state and action spaces~\citep{peters2006policy}, resilience to sensor and actuator noise~\citep{gravell2020learning}, robustness in partially-observable environments~\citep{azizzadenesheli2018policy}, and the possibility of incorporating expert knowledge during the policy design phase~\citep{ghavamzadeh2006bayesian}, thus improving the efficacy, safety, and interpretability of the learned policy~\citep{likmeta2020autonomous}. 
PGs can be categorized into two key families depending on the way exploration is carried out in the policy space
~\citep{montenegro2024learning}. Following their taxonomy, we distinguish between the \emph{action-based} (AB) and the \emph{parameter-based} (PB) exploration paradigms. The former, employed by \reinforce~\citep{williams1992simple} and \gpomdp~\citep{baxter2001infinite}, focuses on directly learning the parameters of a parametric stochastic \emph{policy}. The latter, employed by \pgpe~\citep{SEHNKE2010551}, is tasked with learning the parameters of a parametric stochastic \emph{hyperpolicy} from which the parameters of the actual policy (often deterministic) are sampled.

PGs have gained significant popularity in solving constrained control problems~\citep{achiam2017constrained}.
Within this field, algorithms are primarily developed using \emph{primal-dual} methods~\citep{chow2017risk,tessler2018reward,ding2020natural,ding2021provably,bai2022achieving}, which can be formulated through \emph{Lagrangian optimization} of the primal (i.e., policy or hyperpolicy parameters) and dual variable (i.e., Lagrange multipliers).
Even though the distinction between the exploration paradigms is well known in the PG methods literature, the current state of the art in policy-based CRL focuses only on the \emph{action-based} exploration approach~\citep{achiam2017constrained,stooke2020responsive,bai2023achieving}, while the \emph{parameter-based} one remains unexplored~\citep{montenegro2024constraints}.
A critical challenge for policy-based Lagrangian optimization algorithms is ensuring convergence guarantees. Existing works have spent a notable effort in this direction~\citep{ying2022dual,gladin2023algorithm,ding2024last}. 
Recently,~\citet{ying2022dual}, \citet{gladin2023algorithm}, and~\citet{ding2024last} manage to ensure \emph{global last-iterate} convergence guarantees. However, these approaches are affected by some notable limitations: ($i$) the provided convergence rates depend on the cardinality of the state and action spaces, limiting their applicability to tabular CMPDs and preventing scaling to realistic continuous control problems; ($ii$) they focus on \emph{softmax} policies only, disregarding other more realistic policy models (\eg Gaussian ones); ($iii$) ensure convergence when a single constraint only is present \citep{ding2024last, rozada2024deterministic}.

Real-world problems not only require RL algorithms to produce policies satisfying constraints, but they also often demand the resulting policy to be deterministic to meet reliability, safety, and traceability requirements. To this end, PGs remain a recommended choice.
Considering unconstrained scenarios, the challenge of learning deterministic policies was first addressed by~\citep{silver2014deterministic}, that introduced the \emph{deterministic policy gradient} (DPG) method, later inspiring successful deep RL algorithms such as DDPG~\citep{lillicrap2016deterministic,fujimoto2018deterministic}. However, DPG-based approaches present notable drawbacks due to their off-policy nature, which makes the theoretical analysis complex and limits local convergence guarantees to restrictive assumptions~\citep{xiong2022deterministic}. 
More recently, \citet{montenegro2024learning} proposed a unified framework for deterministic policy deployment that bridges action-based and parameter-based exploration paradigms. Their approach is grounded on specific noise models that represent stochastic policies and hyperpolicies as \emph{perturbations} of the actions or the parameters of an underlying deterministic policy. The core idea is to \emph{train stochastic} (hyper)policies via policy gradient algorithms and \emph{deploy} their \emph{deterministic} counterparts by \emph{switching off} the stochasticity, thereby offering a principled methodology for incorporating deterministic policies within the policy gradient framework.
The latter contribution is a recent advancement in the unconstrained setting, while in the constrained one, despite the advancements in policy-based CRL, the integration of deterministic policies within this framework remains mostly unexplored.
A recent contribution in this direction is presented by~\citet{rozada2024deterministic}, who introduce the \emph{Deterministic Policy Gradient Primal-Dual} (D-PGPD) algorithm, a novel approach designed to directly learn deterministic policies in CMDPs with continuous state and action spaces. D-PGPD is a primal-dual algorithm that incorporates entropy regularization \wrt the policy and ridge regularization \wrt the dual variable. 
Unlike traditional stochastic policy-based CRL methods, which introduce exploration through policy randomness, D-PGPD learns a fully deterministic policy and relies solely on the inherent stochasticity of the environment. While this design ensures stable and consistent policy execution, it may face limitations in environments where intrinsic stochasticity is insufficient to explore the state-action space effectively.
Moreover, the algorithm is developed for single-constraint settings, which may reduce its applicability to real-world problems involving multiple constraints. In the considered setup, and under additional assumptions, such as boundedness of the action space, the sample-based version of D-PGPD requires a sample complexity of order $\cO(\epsilon^{-18})$ to the optimal feasible deterministic policy in the last iterate. This unsatisfactory rate highlights the need for further research on more efficient methods.


\paragraph{Original Contribution.}
The goal of this work is to introduce a framework for solving constrained continuous control problems using policy-based primal-dual algorithms that operate in both the \emph{action-based} and \emph{parameter-based} policy gradient exploration scenarios, while providing global last-iterate convergence guarantees with general (hyper)policy parameterization to both optimal stochastic (hyper)policies and deterministic policies. Specifically, the main contributions of this work can be summarized as follows:
\begin{itemize}
    \item In Section~\ref{sec:setting}, we introduce a general constrained optimization problem, which is agnostic \wrt both the \emph{action-based} and \emph{parameter-based} exploration paradigm.
    \item In Section~\ref{sec:alg_general}, we introduce \cpg, a general policy-based primal-dual algorithm optimizing the regularized Lagrangian function associated with the general constrained optimization problem shown in Section~\ref{sec:setting}. We show that, under (weak) gradient domination assumptions, it simultaneously achieves the following: ($i$) \emph{last-iterate} convergence guarantees to a globally optimal feasible policy (\ie satisfying all the constraints); ($ii$) compatibility with CMDPs having \emph{continuous state and action spaces}; ($iii$) the ability to handle \emph{multiple constraints}.
    \item In Section~\ref{sec:deterministic}, we restrict action-based and parameter-based exploration paradigms as white-noise perturbations applied to the actions or parameters of an underlying parametric deterministic policy. Based on this characterization, we define \emph{deterministic deployment} as the process of \emph{switching off} the noise~\citep{montenegro2024learning} in the learned stochastic (hyper)policies. Within this framework, we derive all the conditions required to ensure the last-iterate global convergence of \cpg, as presented in Section~\ref{sec:alg_general}. Finally, we show that this approach guarantees last-iterate convergence to the optimal deterministic policy in the constrained setting.
\end{itemize}
In Section~\ref{sec:experiments}, we numerically validate the parameter-based and the action-based variants of \cpg against state-of-the-art baselines in constrained control problems. 
Related work is discussed in Section~\ref{apx:related}.
Omitted proofs and additional technical results are reported in Appendix~\ref{apx:proofs}. 

%% file: contents/02_setting.tex
\section{Preliminaries} \label{sec:setting}

In this section, we present the notation we will use throughout this manuscript and the preliminaries needed to understand its content. Moreover, after having introduced the the \AB and \PB exploration paradigms, we introduce the exploration-agnostic constrained optimization problem we aim at solving with the introduced method.

\paragraph{Notation.}
For a measurable set $\mathcal{X}$, we denote as $\Delta(\mathcal{X})$ the set of probability measures over $\mathcal{X}$. For $P \in \Delta(\mathcal{X})$, we denote with $p$ its density function \wrt a reference measure that we assume to exist whenever needed. With a little abuse of notation, we will interchangeably use $x \sim P$ or $x \sim p$ to express that random variable $x$ is distributed according to $P$. 
For $n, m \in \mathbb{N}$ with $n \le m$, we denote $\dsb{n} \coloneqq \left\{1,2, \ldots , n\right\}$ and with $\dsb{n,m} \coloneqq \left\{n,n+1, \ldots , m\right\}$.
For a vector $\bm{x} \in \mR^{d}$, we denote as $x_i$ the $i$-th component of $\bm{x}$.
For $a \in \mR$, we define $(a)^+ \coloneqq \max\{0, a\}$ and we extend the notation to vectors as $(\bm{x})^+=((x_1)^+,\dots,(x_d)^+)^\top$. Given a set $\mathcal{X} \subseteq \mathbb{R}^d$, we denote with $\Pi_{\mathcal{X}}$ the Euclidean-norm projection, i.e., $\Pi_{\mathcal{X}} \bm{x} \in \argmin_{\bm{y} \in \mathcal{X}} \|\bm{y} - \bm{x} \|_2$ for any $\bm{x} \in \mathbb{R}^d$. For two vectors $\bm{x},\bm{y} \in \mathbb{R}^d$, we denote with $\langle\bm{x},\bm{y} \rangle$ their inner product. A function $f:\mathbb{R}^d \rightarrow \mathbb{R}$ is $L_1$-Lipschitz Continuous ($L_1$-LC) if $|f(\bm{x}) - f(\bm{x}')| \le L_1 \| \bm{x}-\bm{x}'\|_2$ and $L_2$-Lipschitz Smooth ($L_2$-LS) if it is differentiable and $\|\nabla_{\bm{x}}f(\bm{x}) - \nabla_{\bm{x}}f(\bm{x}')\|_2 \le L_2 \| \bm{x}-\bm{x}'\|_2$ for every $\bm{x},\bm{x}'\in \mathbb{R}^d$.

\paragraph{Constrained Markov Decision Processes.}
A Constrained Markov Decision Process~\citep[CMDP,][]{altman1999constrained} with $U \in \mathbb{N}$ constraints is represented by $\cM_{\cC} \coloneqq \left( \cS, \cA, p, r, \{c_i\}_{i\in\dsb{U}}, \{b_i\}_{i\in\dsb{U}}, \phi_{0}, \gamma \right)$, where $\cS \subseteq \mR^{d_{\cS}}$ and $\cA \subseteq \mR^{d_{\cA}}$ are the measurable state and action spaces; $p: \cS \times \cA \to \Delta\left(\cS\right)$ is the transition model, where $p(\bm{s}' | \bm{s}, \bm{a})$ is the probability density of getting to state $\bm{s}'\in\cS$ given that action $\bm{a}\in\cA$ is taken in state $\bm{s}\in\cS$; $r: \cS \times \cA \to [-1, 0]$ is the reward function, where $r(\bm{s},\bm{a})$ is the instantaneous reward obtained by playing action $\bm{a}$ in state $\bm{s}$; $c_i: \cS \times \cA \to [0, 1]$ is the $i$-th cost function, where $c_i(\bm{s},\bm{a})$ is the $i$-th instantaneous cost obtained by playing action \hll{$\bm{a}$ in state $\bm{s}$; $b_i \in [0,J_{\max}]$ is the threshold for the $i$-th cost for every $i \in \dsb{U}$; $\phi_{0} \in \Delta(\cS)$ is the initial state distribution; and $\gamma \in [0,1]$ is the discount factor.
A trajectory $\tau$ of length $T \in \mathbb{N} \cup \{+\infty\}$\footnote{We admit $\gamma=1$ just when $T<+\infty$.} is a sequence of $T$ state-action pairs: $\tau = \left(\bm{s}_{\tau, 0}, \bm{a}_{\tau, 0}, \dots, \bm{s}_{\tau, T-1}, \bm{a}_{\tau, T-1}\right)$. 
The \emph{discounted return} over a trajectory $\tau$ is $R(\tau) \coloneqq \sum_{t=0}^{T-1} \gamma^t r(\bm{s}_{\tau, t}, \bm{a}_{\tau, t})$, while the $i$-th \emph{discounted cumulative cost} is $C_{i}(\tau) \coloneqq \sum_{t=0}^{T-1} \gamma^t c_{i}(\bm{s}_{\tau, t}, \bm{a}_{\tau, t})$.  We define the additional cost function $c_0(\bm{s}, \bm{a}) \coloneqq - r(\bm{s}, \bm{a}) \in [0,1]$ and $C_0(\tau) \coloneqq - R(\tau)$ just for presentation purposes. \hll{Note that, with $J_{\max} \coloneqq \frac{1-\gamma^T}{1-\gamma}$,} $R(\tau) \in [-J_{\max}, 0]$ and $C_i(\tau) \in [0,J_{\max}]$, for every $i \in \dsb{U}$ and trajectory $\tau$.} Our goal is to minimize $\E [C_0(\tau)]$ subject to the constraints $\E[C_i(\tau)] \le b_{i}$ for every $i \in \dsb{U}$.

\paragraph{\textcolor{vibrantBlue}{Action-based Policy Gradients.}}
Action-based (AB) PG methods focus on learning the parameters $\vtheta \in \Theta \subseteq \mR^{d_{\Theta}}$ of a parametric stochastic policy $\pi_{\vtheta}: \cS \to \Delta(\cA)$, where $\pi_{\vtheta}(\bm{a} | \bm{s})$ represents the probability density of selecting action $\bm{a}\in\cA$ being in state $\bm{s}\in\cS$.
At each step $t$ of the interaction with the environment, the stochastic policy is employed to sample an action $\bm{a}_t \sim \pi_{\vtheta_{t}}(\cdot | \bm{s}_t)$.
To assess the performance of $\pi_{\vtheta}$ \wrt the $i$-th cost function, with $i \in \dsb{0,U}$, we employ the \emph{AB performance index} $\Ja:\Theta \to \mR$, which is defined as $\Ja(\vtheta) \coloneqq \E_{\tau \sim \pA(\cdot | \vtheta)} \left[ C_{i}(\tau) \right]$, 
where $\pA(\tau, \vtheta) \coloneqq \phi_{0}(\bm{s}_{\tau, 0}) \prod_{t=0}^{T-1} \pi_{\vtheta}(\bm{a}_{\tau, t} | \bm{s}_{\tau, t}) p(\bm{s}_{\tau, t+1} | \bm{s}_{\tau, t}, \bm{a}_{\tau, t})$ is the density of trajectory $\tau$ induced by policy $\pi_{\vtheta}$.

\paragraph{\textcolor{vibrantRed}{Parameter-based Policy Gradients.}}
Parameter-based (PB) PG methods focus on learning the parameters $\vrho \in \mathcal{R} \subseteq \mR^{d_{\mathcal{R}}}$ of a parametric stochastic hyperpolicy $\nu_{\vrho} \in \Delta(\Theta)$. 
The hyperpolicy $\nu_{\vrho}$ is used to sample parameter configurations $\vtheta \sim \nu_{\vrho}$ to be plugged into an underlying parametric policy $\pi_{\vtheta}$, that will be then used for the interaction with the environment. Notice that $\pi_{\vtheta}$ can also be deterministic.
To assess the performance of $\nu_{\vrho}$ \wrt the $i$-th cost function, with $i \in \dsb{0,U}$, we employ the \emph{PB performance index} $\Jp:\mathcal{R} \to \mR$, which is defined as $\Jp(\vrho) \coloneqq \E_{\vtheta \sim \nu_{\vrho}} \left[\E_{\tau \sim \pA(\cdot | \vtheta)} \left[ C_{i}(\tau) \right]\right]$.

\paragraph{Constrained Optimization Problem.}
Having introduced the \textcolor{vibrantBlue}{AB} and \textcolor{vibrantRed}{PB} performance indices, we formulate a \emph{constrained optimization problem} (COP), which is agnostic \wrt the exploration paradigm:
\begin{align}\label{eq:gen_opt_prob}
    \min_{\vupsilon \in \mathcal{V}} J_{\dagger, 0}(\vupsilon) \quad \sucht \quad J_{\dagger, i}(\vupsilon) \le b_i , \; \; \forall i \in \dsb{U}, 
\end{align}
where $\dagger \in \{\text{\textcolor{vibrantBlue}{A}}, \text{\textcolor{vibrantRed}{P}}\}$ and $\vupsilon$ is a generic parameter vector belonging to the parameter space $\mathcal{V}$. 
When $\dagger = \text{\textcolor{vibrantBlue}{A}}$, we are considering the \textcolor{vibrantBlue}{AB} exploration paradigm, so $\mathcal{V} = \Theta$.
On the other hand, when $\dagger = \text{\textcolor{vibrantRed}{P}}$, we are in the \textcolor{vibrantRed}{PB} exploration paradigm, thus $\mathcal{V} = \mathcal{R}$.

%% file: contents/03_convergence.tex
\section{Last-Iterate Global Convergence of \cpg}
\label{sec:alg_general}

In this section, we present \cpg, a general primal-dual algorithm that optimizes a regularized version of the Lagrangian function (Section~\ref{sec:reglag}) associated with the COP of Equation~(\ref{eq:gen_opt_prob}). After having introduced the necessary assumptions (Section~\ref{sec:ass}), we show that \cpg exhibit \emph{dimension-free}\footnote{The \emph{dimension-free} property \citep{liu2021policy,ding2020natural,ding2022convergence,ding2024last} is achieved when the convergence rates do not depend on the cardinality of the state and/or action spaces.} \emph{last-iterate global} convergence guarantees (Section~\ref{sec:conv}). 
For notational convenience, in the rest of this section, we use $J_i$ in place of $J_{\dagger,i}$.

\subsection{Regularized Lagrangian Approach}\label{sec:reglag}
To solve the COP of Equation~(\ref{eq:gen_opt_prob}) we resort to the method of Lagrange multipliers~\citep{bertsekas2014constrained} introducing the Lagrangian function $\cL_{0} (\vupsilon, \vlambda) \coloneqq J_{ 0}(\vupsilon) + \sum_{i=1}^{U} \lambda_{i} \left( J_i(\vupsilon) - b_i \right) = J_0(\vupsilon) +\langle \vlambda, \mathbf{J}(\vupsilon) - \mathbf{b}\rangle$, where $\vupsilon \in \mathcal{V}$ is the primal variable and $\vlambda {\in \mathbb{R}^U_{\ge 0}}$ are the Lagrangian multipliers or dual variable,  $\mathbf{J} = (J_1,\dots,J_U)^\top$, and $\mathbf{b} = (b_1,\dots,b_U)^\top$. This allows rephrasing the COP in Equation~\eqref{eq:gen_opt_prob} as a min-max optimization problem $\min_{\vupsilon \in \mathcal{V}} \max_{\vlambda {\in \mathbb{R}^U_{\ge 0}}}  \cL_{0} (\vupsilon, \vlambda)$ and we denote with $H_0(\vupsilon) \coloneqq  \max_{\vlambda {\in \mathbb{R}^{U}_{\ge 0}}}  \cL_{0} (\vupsilon, \vlambda)$ the \emph{primal function} and its optimum with $H^*_0 \coloneqq \min_{\vupsilon \in \mathcal{V}} H_0(\vupsilon)$.
To obtain a \emph{last-iterate} convergence guarantee, we make use of a regularization approach. Specifically, let $\omega > 0$ be a regularization parameter, we define the \emph{$\omega$-regularized Lagrangian function} as follows:
\begin{align*}
    \cL_{\omega} (\vupsilon, \vlambda) \coloneqq  J_{0}(\vupsilon) + \sum_{i=1}^{U} \lambda_{i} \left( J_i(\vupsilon) - b_i \right) - \frac{\omega}{2} \left\| \vlambda \right\|_2^2 = J_{0}(\vupsilon) + \langle \vlambda,  \mathbf{J}(\vupsilon) - \mathbf{b} \rangle - \frac{\omega}{2} \left\| \vlambda \right\|_2^2 = \cL_{0}(\vupsilon,\vlambda) - \frac{\omega}{2} \left\| \vlambda \right\|_2^2.
\end{align*}
The ridge regularization makes $\cL_{\omega} (\vupsilon, \vlambda)$ a strongly concave function of $\vlambda$ at the price of a bias that is quantified in Lemmas~\ref{lemma:boundLagrangeMult}, \ref{lemma:regularizationBias}, and~\ref{lemma:33}. Thus, we address the $\omega$-regularized min-max optimization problem $\min_{\vupsilon \in \mathcal{V}} \max_{\vlambda {\in \Lambda}}  \cL_{\omega} (\vupsilon, \vlambda)$, where \hll{$\Lambda \coloneqq \{\vlambda \in \mathbb{R}_{\ge 0}^U \,:\, \|\vlambda \|_2\le \omega^{-1} \sqrt{U} J_{\max}\}$}, in replacement of the original (non-regularized) one. 
We stress that this choice of $\Lambda$ guarantees that the optimal Lagrange multipliers $\vlambda_{\omega}^{*}$ lie within $\Lambda$.
For this problem, we introduce the primal function $H_\omega(\vupsilon) \coloneqq \max_{\vlambda {\in \Lambda}} \cL_{\omega} (\vupsilon, \vlambda)$, that, thanks to the ridge regularization, admits the closed-form expression: 
\begin{align*}
    H_\omega(\vupsilon) =  J_{0}(\vupsilon) + \frac{1}{2\omega} \sum_{i=1}^{U}  \left(\left( J_i(\vupsilon) - b_i \right)^+\right)^2 =  J_{0}(\vupsilon) + \frac{1}{2\omega} \| (\mathbf{J}(\vupsilon) - \mathbf{b})^+\|^2_2, 
\end{align*}
where the optimal values of the Lagrange multipliers are given by: 
\begin{align*}
    \vlambda^*(\vupsilon) = \Pi_{\Lambda} \left( \frac{1}{\omega} (\mathbf{J}(\vupsilon) - \mathbf{b}) \right) = \frac{1}{\omega} (\mathbf{J}(\vupsilon) - \mathbf{b})^+ ,
\end{align*}
that is guaranteed to have norm smaller than $\omega^{-1} \sqrt{U} J_{\max}$. Furthermore, we define $H^*_\omega \coloneqq \min_{\vupsilon \in \mathcal{V}} H_\omega(\vupsilon)$. \cpg updates the parameters $(\vupsilon_k,\vlambda_k)$ with an \emph{alternate gradient descent-ascent} scheme for every iterate $k \in \mathbb{N}$:

\begin{tabular}{ll}
    \textbf{Primal Update}: & $\vupsilon_{k+1} \leftarrow \Pi_{\mathcal{V}} \left(\vupsilon_{k} - \zeta_{\vupsilon,k} \widehat{\nabla}_{\vupsilon} \cL_{\omega}(\vupsilon_{k},\vlambda_{k})\right),$ \\
    \textbf{Dual Update}: & $\vlambda_{k+1} \leftarrow \Pi_{\Lambda} \left(\vlambda_{k} + \zeta_{\vlambda,k} \widehat{\nabla}_{\vlambda} \cL_{\omega}(\vupsilon_{k+1},\vlambda_{k})\right),$
\end{tabular}

\noindent where $\zeta_{\vupsilon,k},\zeta_{\vlambda,k}>0$ are the learning rates and $\widehat{\nabla}_{\vupsilon} \cL_{\omega}(\vupsilon_{k},\vlambda_{k}), \widehat{\nabla}_{\vlambda} \cL_{\omega}(\vupsilon_{k},\vlambda_{k})$ are (unbiased) estimators of the gradients ${\nabla}_{\vupsilon} \cL_{\omega}(\vupsilon_{k},\vlambda_{k}), {\nabla}_{\vlambda} \cL_{\omega}(\vupsilon_{k},\vlambda_{k})$ of the regularized Lagrangian function. Notice that \cpg performs \emph{alternate} descent-ascent, as the update value for the dual variable is performed employing the \emph{already updated} primal variable.

\subsection{Assumptions}\label{sec:ass}
Before diving into the study of the convergence guarantees of \cpg, %
we list and motivate the assumptions necessary for our analysis.

\begin{ass}[Existence of Saddle Points]\label{asm:assunzione}
	There exist $\vupsilon^*_0 \in \mathcal{V}$ and $\vlambda^*_0 \in \mathbb{R}^U_{\ge 0}$ such that $ \cL_{0} (\vupsilon^*_0, \vlambda^*_0) = \min_{\vupsilon \in \mathcal{V}} \max_{\vlambda {\in \mathbb{R}^U_{\ge 0}}}  \cL_{0} (\vupsilon, \vlambda)$.
\end{ass}

Assumption~\ref{asm:assunzione} ensures that the value of the min-max problem is attained by a pair of primal-dual values $\vupsilon^*_0 \in \mathcal{V}$ and $\vlambda^*_0 \in \mathbb{R}^U_{\ge 0}$ which, consequently, satisfy $\cL_{0} (\vupsilon^*_0, \vlambda) \le \cL_{0} (\vupsilon^*_0, \vlambda^*_0) \le \cL_{0} (\vupsilon, \vlambda^*_0) $ for every $\vupsilon \in \mathcal{V}$ and $\vlambda \in \mathbb{R}^U_{\ge 0}$. Analogous assumptions have been considered by~\citet{yang2020minimax} and~\citet{ying2022dual}. Thus, $(\vupsilon^*_0,\vlambda^*_0)$ is a saddle point of the Lagrangian function $\cL_{0}$ and, consequently, \emph{strong duality} holds. Alternatively, as commonly requested in CRL works, assuming \emph{Slater's condition} combined with the requirement that the policy space covers all Markovian policies ensures strong duality~\citep[e.g.,][]{paternain2019constrained, ding2020natural, ding2024last}.\footnote{Assumption~\ref{asm:assunzione} combined with Slater's condition, i.e., the existence of a parametrization $\widetilde{\vupsilon} \in \mathcal{V}$ for which there exists $\xi > 0$ such that $J_i(\widetilde{\vupsilon}) - b < -\xi$ for all $i \in \dsb{U}$ (strictly feasible), allows providing an upper bound to the Lagrange multipliers $\|\vlambda^*_0\|_2 \le \xi^{-1}(J_0(\widetilde{\vupsilon}) - J_0(\vupsilon^*_0))$ using standard arguments~\citep[see][]{ying2022dual}.}
\begin{ass}[Weak $\psi$-Gradient Domination] \label{asm:wgd}
    Let $\psi \in [1,2]$. There exist $\alpha_{1} \in \mR_{>0}$ and $\beta_{1} \in \mR_{\ge0}$ such that, for every $\vupsilon \in \mathcal{V}$ and $\vlambda {\in \Lambda}$, it holds that:
    \begin{align}\label{eq:wgd}
        \left\| \nabla_{\vupsilon} \cL_{0}(\vupsilon, \vlambda) \right\|_2^{\psi} \ge \alpha_{1} \Big( \cL_{0}(\vupsilon, \vlambda) - \min_{\vupsilon' \in \mathcal{V}} \cL_{0}(\vupsilon', \vlambda) \Big) - \beta_{1}.
    \end{align}
\end{ass}

Assumption~\ref{asm:wgd} is customary in the convergence analysis of policy gradient methods and it is usually enforced on the objective $J_0$ only~\citep{yuan2022general,masiha2022stochastic,fatkhullin2023stochastic}. In particular, when $\beta_1=0$, we speak of strong $\psi$-gradient domination. In this form, for a generic exponent $\psi \in [1,2]$, this assumption has been employed by~\cite{masiha2022stochastic}. Particular cases are $\psi=1$, which corresponds to the standard weak \emph{gradient domination} (GD), while for $\psi=2$ we have the so-called \emph{Polyak-\L ojasiewicz} (PL) condition. Notice that Assumption~\ref{asm:wgd} is enforced on the non-regularized Lagrangian function $\cL_0$ (i.e., $\omega=0$). However, it is easy to realize that it holds for the regularized one $\cL_\omega$ by simply replacing $\cL_0$ with $\cL_\omega$ in Equation~\eqref{eq:wgd}.

\begin{remark}[When does Assumption~\ref{asm:wgd} holds?]\label{remark:interp}
    As remarked by~\cite{ding2024last}, the Lagrangian function, for a fixed value of $\vlambda$ can be regarded as the return of a new reward function $-C_0 - \langle \vlambda, \mathbf{C} \rangle$, where $\mathbf{C}=(C_1,\dots,C_U)^\top$. As a consequence, a sufficient condition for Assumption~\ref{asm:wgd} is when the selected class of policies guarantees the $\psi$-gradient domination \emph{regardless} of the reward function. For instance, in tabular environments with natural policy parametrization, i.e., $\pi_{\vtheta}(s) = \vtheta_s$ for every $s \in \mathcal{S}$, the PL condition ($\psi=2$ and $\beta_1=0$) holds~\citep{bhandari2024global}. Moreover, in tabular environments with softmax policy, i.e., $\pi_{\vtheta}(a|s) \propto \exp(\theta(s,a))$, GD ($\psi=1$ and $\beta_1=0$) holds~\citep{mei2020global}. This enables a meaningful comparison of our results when resorting to softmax policies~\citep[e.g.,][]{ding2020natural,gladin2023algorithm,ding2024last}. More in general, when ($i$) the Fisher information matrix induced by policy $\pi_{\vtheta}$ is non-degenerate for every $\vtheta \in \Theta$, i.e.,  $\mathbf{F}(\vtheta)= \E_{\pi_{\vtheta}}[\nabla_{\vtheta} \log \pi_{\vtheta}(\bm{a}|\bm{s})\nabla_{\vtheta} \log \pi_{\vtheta}(\bm{a}|\bm{s})^\top]\succeq \mu_{\text{F}} \mathbf{I}$ for some   $\mu_{\text{F}} > 0$
    and ($ii$) a compatible function approximation bias bound holds, i.e., $\E_{\pi_{\vtheta^*}}[(A^{\pi_{\vtheta}}(\bm{s},\bm{a}) - (1-\gamma) \bm{u}^\top \nabla_{\vtheta} \log \pi_{\vtheta}(\bm{a}|\bm{s}))^2] \le \epsilon_{\text{bias}}$ being $\bm{u} = \mathbf{F}(\vtheta)^{\dagger} \nabla_{\vtheta} J_0(\vtheta)$ and the advantage function $A^{\pi_{\vtheta}}$ computed \wrt reward $-c_0 - \langle \vlambda, \mathbf{c} \rangle$, the weak GD ($\psi=1$) holds with $\alpha_1 = G \mu^{-1}_F$ and $\beta_1 = (1-\gamma)^{-1}\sqrt{\epsilon_{\text{bias}}}$, where $G$ is such that $\|\nabla_{\vtheta} \log \pi_{\vtheta}(\bm{a}|\bm{s}) \|_2 \le G$~\citep{masiha2022stochastic}. 
\end{remark}

In principle, we could have enforced Assumption~\ref{asm:wgd} on the primal function $H_\omega(\vupsilon)$ only. However, this would come with two drawbacks: ($i$) the assumption would now depend explicitly on $\omega$; ($ii$) the considerations of Remark~\ref{remark:interp} would no longer hold. Nevertheless, in Lemma~\ref{lemma:wgdH}, we prove that Assumption~\ref{asm:wgd} induces an analogous property on the primal function $H_\omega(\vupsilon)$ in the regularized case.

\begin{ass}[Regularity of the Regularized Lagrangian $\cL_{0}$] \label{asm:L_grad_lip}
    There exist $L_1,L_{2},L_3 \in \mR_{>0}$ such that, for every $\vupsilon,\vupsilon' \in \mathcal{V}$, and for every $\vlambda,\vlambda' {\in \Lambda}$, the following holds:
    \begin{align}
        & \text{$\nabla_{\vlambda} \cL_0(\cdot, \vlambda)$ $L_{1}$-Lipschitz w.r.t. $\vupsilon$:} && \left\| \nabla_{\vlambda} \cL_{0}(\vupsilon, \vlambda) -  \nabla_{\vlambda} \cL_{0}(\vupsilon', \vlambda) \right\|_{2} \le L_1 \! \left\| \vupsilon - \vupsilon'\right\|_{2},\label{eq:lip} \\
        & \text{$\cL_0(\cdot,\vlambda)$ $L_{2}$-Smooth w.r.t. $\vupsilon$}: && \left\| \nabla_{\vupsilon} \cL_{0}(\vupsilon, \vlambda) - \nabla_{\vupsilon} \cL_{0}(\vupsilon', \vlambda) \right\|_{2} \le L_2 \| \vupsilon - \vupsilon'\|_{2},\label{eq:smooth}\\
        & \text{$\nabla_{\vupsilon}\cL_0(\vupsilon,\cdot)$ $L_{3}$-Lipschitz w.r.t $\vlambda$:} &&  \left\| \nabla_{\vupsilon} \cL_{0}(\vupsilon, \vlambda) - \nabla_{\vupsilon} \cL_{0}(\vupsilon, \vlambda') \right\|_{2} \le L_3 \| \vlambda - \vlambda'\|_{2}.\label{eq:lip2}
    \end{align}
\end{ass}

Notice that, similarly to Assumption~\ref{asm:wgd}, we realize that if Assumption~\ref{asm:L_grad_lip} holds for the non-regularized Lagrangian $\cL_{0}$, it also holds (with the same constants) for the regularized one $\cL_\omega$ for every $\omega>0$. The regularity conditions of Assumption~\ref{asm:L_grad_lip} are common in the literature~\citep{yang2020minimax} and mild when regarded from the policy optimization perspective.
Equation~\eqref{eq:lip} is satisfied whenever the constraint functions $J_i$ are Lipschitz continuous w.r.t. $\vupsilon$. Indeed, $\left\| \nabla_{\vlambda} \cL_{0}(\vupsilon, \vlambda) -  \nabla_{\vlambda} \cL_{0}(\vupsilon', \vlambda) \right\|_{2} = \left\| \mathbf{J}(\vupsilon) -   \mathbf{J}(\vupsilon') \right\|_{2}$. Equation~\eqref{eq:smooth} is fulfilled when the objective function $J_0$ and the constraint functions $J_i$  are smooth w.r.t. $\vupsilon$ and the Lagrange multipliers are bounded (guaranteed thanks to the projection $\Pi_\Lambda$), since $\left\| \nabla_{\vupsilon} \cL_{0}(\vupsilon, \vlambda) - \nabla_{\vupsilon} \cL_{0}(\vupsilon', \vlambda) \right\|_{2} \le |\nabla_{\vupsilon}J_0(\vupsilon) -  \nabla_{\vupsilon}J_0(\vupsilon')| + \sum_{i=1}^U \lambda_i |\nabla_{\vupsilon} J_i(\vupsilon) - \nabla_{\vupsilon} J_i(\vupsilon)|$. Finally, Equation~\eqref{eq:lip2} is fulfilled whenever functions $J_i$ admit bounded gradients, since $\left\| \nabla_{\vupsilon} \cL_{0}(\vupsilon, \vlambda) - \nabla_{\vupsilon} \cL_{0}(\vupsilon, \vlambda') \right\|_{2} \le \| \nabla_{\vupsilon} \mathbf{J}(\vupsilon) ( \vlambda - \vlambda')\|_2 $. 
 \hll{It is worth noting that $L_2$ depends on the norm of the Lagrange multipliers and, consequently, due to the projection operator $\Pi_{\Lambda}$, we have that $L_2 = \cO(\omega^{-1})$, whereas $L_1$ and $L_3$ are independent on $\omega$.\footnote{We highlight the dependencies on $\omega$ since, as we shall see later, we will set $\omega = \cO(\epsilon)$ having, consequently, an effect on the convergence rate.}}
Explicit conditions on the constitutive elements of the MDP and (hyper)policies to ensure Lipshitzness and smoothness of these quantities are reported in~\citep[][Appendix E]{montenegro2024learning} for both the \textcolor{vibrantBlue}{AB} and \textcolor{vibrantRed}{PB} cases.
These regularity properties enforced on $\cL_\omega$ are inherited by the primal function $H_\omega$ which results to be $\left( L_2 + L_1^2 \omega^{-1}\right)$-LS (Lemma~\ref{lemma:smoothH}).
Concerning the regularity of $\cL_\omega$ w.r.t. $\vlambda$, we observe that it is a quadratic function and, therefore, it is $\omega$-smooth and satisfies the PL condition, i.e., Assumption~\ref{asm:wgd} with $\psi=2$, $\beta_1=0$, and with $\alpha_1=\omega$ (Lemma~\ref{lemma:propLLambda}).

\begin{ass}[Bounded Estimator Variance]\label{ass:boundedVariance}
    For every $\vupsilon \in \mathcal{V}$ and $\vlambda \in \Lambda$, the estimators $\widehat{\nabla}_{\vupsilon} \cL_\omega(\vupsilon,\vlambda)$ and $\widehat{\nabla}_{\vlambda} \cL_\omega(\vupsilon,\vlambda)$ are unbiased for ${\nabla}_{\vupsilon} \cL_\omega(\vupsilon,\vlambda) =\nabla_{\vupsilon}J_0(\vupsilon) + \sum_{i=1}^U \lambda_i \nabla_{\vupsilon} J_i(\vupsilon)$ and ${\nabla}_{\vlambda} \cL_\omega(\vupsilon,\vlambda) = \mathbf{J}(\vupsilon) - \mathbf{b} - \omega \vlambda$ with bounded variance, i.e., there exist $V_{\vupsilon},V_{\vlambda} \in \mR_{\ge 0}$ such that:
    \begin{align*}
        \Var[\widehat{\nabla}_{\vupsilon} \cL_\omega(\vupsilon,\vlambda)] \le V_{\vupsilon}, \qquad \Var[\widehat{\nabla}_{\vlambda} \cL_\omega(\vupsilon,\vlambda)] \le V_{\vlambda}.
    \end{align*}
\end{ass}
\hll{Note that $V_{\vupsilon}$ typically depends on the Lagrange multipliers and, for standard sample-mean estimators, it is of order $V_{\vupsilon} = \cO(\omega^{-2})$ thanks to the projection operator. In contrast, $V_{\vlambda}$ is usually not affected by $\omega$ since the term $\omega \vlambda$ is not estimated and, thus, it does not affect the variance of the sample mean estimator.}
The variance of such estimators can be easily controlled by leveraging on the properties of the score function as done in previous works (see~\citealt{papini2022smoothing} and \citealt{montenegro2024learning}, Appendix~E).

\subsection{Convergence Analysis}\label{sec:conv}
We are now ready to tackle the convergence analysis of \cpg to the global optimum of the COP of Equation~\eqref{eq:gen_opt_prob}. To this end, we study the \emph{potential function} defined as $\mathcal{P}_k(\chi) \coloneqq a_k + \chi b_k$, where $a_k \coloneqq \E[H_\omega(\vupsilon_k) - H^*_\omega]$ and $b_k \coloneqq \E[H_\omega(\vupsilon_k) - \cL_\omega(\vupsilon_k,\vlambda_k)]$, $\chi \in (0,1)$ will be specified later, and the expectation is taken \wrt the stochastic process generating samples. Since $a_k,b_k \ge 0$, intuitively, if $\mathcal{P}_k(\chi) \approx 0$ we have that both $a_k,b_k \approx 0$ and, consequently, convergence is achieved. 
Let us start relating $\mathcal{P}_k(\chi)$, with the solution of the COP in Equation~\eqref{eq:gen_opt_prob}.

\begin{restatable}[Objective Function Gap and Constraint Violation]{thr}{conversion}\label{thr:conversion}
    Let $\epsilon \in \mR_{> 0}$. Under Assumption~\ref{asm:assunzione}, if $\mathcal{P}_k(\chi) \le \epsilon$, it holds that:
    \begin{align}
        \E[J_0(\vupsilon_k) - J_0(\vupsilon^*_0)] \le \epsilon + \frac{\omega}{2} \| \vlambda^*_0\|_2^2 , \qquad \E[(J_i(\vupsilon_k) - b_i)^+] \le 4\epsilon +  \omega \|\vlambda^*_0\|_2 , \quad \forall i \in \dsb{U}.
    \end{align}
\end{restatable}
\begin{proof}
Since $\mathcal{P}_k(\chi) \le \epsilon$, it follows that $a_k \le \epsilon$ and, consequently, $0 \le \E[H_\omega(\vupsilon_k) - H^*_\omega] \le \epsilon$. We start by bounding the norm of the dual variables:
\begin{align*}
    \| \vlambda^*(\vupsilon_k) \|_2 \le \| \vlambda^*_\omega \|_2 + \| \vlambda^*(\vupsilon_k) - \vlambda^*_\omega \|_2 \le \| \vlambda^*_\omega \|_2 + \frac{4}{\omega} (H_\omega(\vupsilon_k) - H^*_\omega),
\end{align*}
where we applied the triangular inequality and Lemma~\ref{lemma:techtech}, which proves that, for any $\vupsilon \in \mathcal{V}$, $H_{\omega}(\vupsilon)  - H_{\omega}^*  \ge \frac{\omega}{4} \|\vlambda^*(\vupsilon) -  \vlambda^*_\omega\|_2$.
The projection $\Pi_\Lambda$ is such that $\vlambda^*(\vupsilon) = \Pi_{\Lambda} \left(\frac{1}{\omega}  (\mathbf{J}(\vupsilon) - \mathbf{b}) \right) = \frac{1}{\omega}  (\mathbf{J}(\vupsilon) - \mathbf{b})^+$ and, consequently, we have:
\begin{align*}
    \|   (\mathbf{J}(\vupsilon_k) - \mathbf{b})^+\|_2 - \| (\mathbf{J}(\vupsilon^*_\omega) - \mathbf{b})^+\|_2 \le 4 (H_\omega(\vupsilon_k) - H^*_\omega).
\end{align*}
By the last inequality, together with Lemma~\ref{lemma:33}, which states that: 
\begin{align*}
        0 \le J_0(\vupsilon^*_0)  - J_0(\vupsilon^*_\omega)  \le \omega \| \vlambda^*_0\|_2^2 \quad \text{and} \quad \| (\mathbf{J}(\vupsilon^*_\omega)- \mathbf{b})^+\|_2 \le \omega  \| \vlambda^*_0\|_2,
\end{align*}
and applying the expectation on both sides, we have the following:
\begin{align*}
    \E[ \|  (\mathbf{J}(\vupsilon_k) - \mathbf{b})^+ \|_2 ] \le \| (\mathbf{J}(\vupsilon^*_\omega) - \mathbf{b})^+\|_2  + 4 \E[H_\omega(\vupsilon_k) - H^*_\omega] \le  \omega\|\vlambda^*_0\|_2 + {4\epsilon}.
\end{align*}
We obtain the constraint violation bound recalling that:
\begin{align*}
    \E[ \|  (\mathbf{J}(\vupsilon_k) - \mathbf{b}) )^+ \|_2] \ge  \left\|\E[    (\mathbf{J}(\vupsilon_k) - \mathbf{b})^+ ]\right\|_2 \ge \|\E[ (\mathbf{J}(\vupsilon_k) - \mathbf{b})^+]\|_\infty.
\end{align*}

For the objective function bound, let us consider the following derivation. By definition of $H_\omega(\vupsilon)$ and $\vlambda^*(\vupsilon)$ we have:
\begin{align*}
    J_0(\vupsilon_k) - J_0(\vupsilon^*_\omega) & = H_\omega(\vupsilon_k) - H^*_\omega - \frac{\omega}{2} \left( \| \vlambda^*(\vupsilon_k)\|_2^2 - \|\vlambda^*_\omega\|_2^2 \right).
\end{align*} 
Taking the expectation on both sides and upper bounding $\|\vlambda^*_\omega\|$ with $\|\vlambda^*_0\|$ from Lemma~\ref{lemma:boundLagrangeMult}, which states that $0 \le \cL_0(\vupsilon^*_0,\vlambda^*_0)- \cL_0(\vupsilon^*_\omega,\vlambda^*_\omega) \le \frac{\omega}{2} \left( \|\vlambda^*_0\|_2^2 - \|\vlambda^*_\omega\|_2^2\right)$, the following holds:
\begin{align*}
    \E[J_0(\vupsilon_k) - J_0(\vupsilon^*_\omega)] & = \E[H_\omega(\vupsilon_k) - H^*_\omega] - \frac{\omega}{2} \E[ \| \vlambda^*(\vupsilon_k)\|_2^2 - \|\vlambda^*_\omega\|_2^2 ]\\
    & \le \E[H_\omega(\vupsilon_k) - H^*_\omega] + \frac{\omega}{2} \|\vlambda^*_\omega\|_2^2 \\
    & \le \epsilon + \frac{\omega}{2} \|\vlambda^*_0\|_2^2. 
\end{align*}
The result is obtained by applying Lemma~\ref{lemma:33} (already stated in this proof) as follows:
\begin{align*}
    \E[J_0(\vupsilon_k) - J_0(\vupsilon^*_0)] = \E[J_0(\vupsilon_k) - J_0(\vupsilon^*_\omega)] + \underbrace{J_0(\vupsilon^*_\omega) - J_0(\vupsilon^*_0)}_{\le 0}.
\end{align*}
\end{proof}

Theorem~\ref{thr:conversion} justifies the study of the potential $\mathcal{P}_k(\chi)$ as a technical tool to ensure convergence. Indeed, whenever $\mathcal{P}_k(\chi) \le \epsilon$ both ($i$) 
the objective function gap and ($ii$) the constraint violation scale linearly with $\epsilon$ and with the regularization parameter $\omega$ of the regularized Lagrangian $\cL_\omega$ multiplied by the norm of the Lagrange multipliers of the non-regularized problem $\|\vlambda^*_0\|_2$, which are finite under Assumption~\ref{asm:assunzione}. This expression also suggests a choice of $\omega = \cO(\epsilon)$ to enforce an overall $\epsilon$ error on both quantities. Note that, from Theorem~\ref{thr:conversion}, it is immediate to employ a \emph{conservative constraint} ($b_i' \approx b_i - 4\epsilon - \omega \| \vlambda^*_0 \|_2$) to achieve zero constraint violation with no modification of the algorithm.

We are now ready to state the convergence guarantees for the potential function. 

\begin{restatable}[Convergence of $\mathcal{P}_K$]{thr}{convergencePot}\label{thr:convergencePot}
Under Assumptions~\ref{asm:wgd},~\ref{asm:L_grad_lip},~\ref{ass:boundedVariance}, for $\chi < 1/5$, sufficiently small $\epsilon$ and $\omega$, and a choice of \emph{constant} learning rates $\zeta_{\vupsilon},\zeta_{\vlambda}$, we have $\mathcal{P}_K(\chi) \le \epsilon + \beta_1/\alpha_1$ whenever:\footnote{In the context of this statement, the $\cO(\cdot)$ notation preserves dependences on $\epsilon$ and $\omega$ only.}
\begin{itemize}
	\item $K = \cO(\omega^{-1} \log (\epsilon^{-1}))$ if $\psi=2$ and the gradients are exact (i.e., $V_{\vupsilon}=V_{\vlambda}=0$);
	\item $K = \cO(\omega^{-1} \epsilon^{-\frac{2}{\psi}-1})$ if $\psi\in[1,2)$ and the gradients are exact (i.e., $V_{\vupsilon}=V_{\vlambda}=0$);
	\item \hll{$K = \cO(\omega^{-3} \epsilon^{-\frac{4}{\psi}+1})$ if $\psi\in[1,2]$ and the gradients are estimated (i.e., $V_{\vupsilon} = \cO(\omega^{-2})$ and $V_{\vlambda}= \cO(1) $).}
\end{itemize}
\end{restatable}
\begin{proofsketch}
    The proof of Theorem~\ref{thr:convergencePot} is quite technical, thus we report here just its sketch, which we divide into five parts.

    \textbf{Part I: bounding $a_{k}$.}~~The \emph{first part} of the proof consists of bounding $\E[a_{k+1} \mid \mathcal{F}_{k-1}] = \E[H_{\omega}(\vupsilon_{k+1}) - H^{*} \mid \mathcal{F}_{k-1}]$, considering to be at a generic $k$\textsuperscript{th} iterate of \cpg with $\mathcal{F}_{k-1}$ a filtration up to iteration $k-1$. In particular, by exploiting the update rule of \cpg, via Lemma~\ref{lemma:smoothH} stating that $H_{\omega}$ is $L_{H}$-LS, and by selecting $\zeta_{\vupsilon,k} \le L_H$, we can conclude that:
    \begin{align*}
       &\E \left[H_\omega(\vupsilon_{k+1})  | \mathcal{F}_{k-1} \right] - H^*\\ 
       &\le H_\omega(\vupsilon_{k}) - H^* - \frac{\zeta_{\vupsilon, k}}{2} \left\| \nabla_{\vupsilon} H_\omega(\vupsilon_{k}) \right\|_2^2 + \frac{\zeta_{\vupsilon, k}}{2} \left\| \nabla_{\vupsilon} \cL_{\omega} (\vupsilon_{k}, \vlambda_{k}) - \nabla_{\vupsilon} H_\omega(\vupsilon_{k}) \right\|_2^2 + \frac{L_H}{2} \zeta_{\vupsilon, k}^2 V_{\vupsilon},
    \end{align*}
    where the constant $V_{\vupsilon}$, coming from Assumption~\ref{ass:boundedVariance}, is such that $\Var[\nabla_{\vupsilon} \cL_{\omega}(\vupsilon_{k}, \vlambda_{k})] \le V_{\vupsilon}$.

    \textbf{Part II: bounding $b_{k}$.}~~Similarly to what shown in the first part, the \emph{second part} consists of bounding $\E[b_{k+1} \mid \mathcal{F}_{k-1}] = \E[H_{\omega}(\vupsilon_{k-1}) - \cL_{\omega}(\vupsilon_{k+1},\vlambda_{k+1}) \mid \mathcal{F}_{k-1}]$. To do so, we exploit Assumption~\ref{asm:L_grad_lip} stating that $\cL_{\omega}$ is $L_{2}$-LS. In particular, in Lemma~\ref{lemma:propLLambda}, we show that $\cL_{\omega}$ is $\omega$-LS and that it fulfills the PL condition with constant $\omega$. From these observations, together with the update rule of \cpg and the selection $\zeta_{\vlambda,k} \le 1/\omega$, we conclude that:
    \begin{align*}
        &\E \left[ H_\omega(\vupsilon_{k+1}) - \cL_{\omega}(\vupsilon_{k+1}, \vlambda_{k+1})  | \mathcal{F}_{k-1} \right] \\
        &\le \left( 1 - \frac{\zeta_{\vlambda, k}}{2} \omega \right) \left( H_\omega(\vupsilon_{k}) - \cL_{\omega}(\vupsilon_{k}, \vlambda_{k}) \right) + \left( 1 - \frac{\zeta_{\vlambda, k}}{2} \omega  \right) \left(\zeta_{\vupsilon, k} \left(1 + \frac{L_2}{2} \zeta_{\vupsilon, k}\right) \left\| \nabla_{\vupsilon} \cL_{\omega}(\vupsilon_{k}, \vlambda_{k}) \right\|_2^2 + \frac{L_2}{2} \zeta_{\vupsilon, k}^2 V_{\vupsilon}\right) \\
        &\quad + \left( 1 - \frac{\zeta_{\vlambda, k}}{2} \omega  \right) \left( - \frac{\zeta_{\vupsilon, k}}{2} \left\| \nabla_{\vupsilon} H_\omega(\vupsilon_{k}) \right\|_2^2 + \frac{\zeta_{\vupsilon, k}}{2} \left\| \nabla_{\vupsilon} \cL_{\omega} (\vupsilon_{k}, \vlambda_{k}) - \nabla_{\vupsilon} H_\omega(\vupsilon_{k}) \right\|_2^2 + \frac{L_H}{2} \zeta_{\vupsilon, k}^2 V_{\vupsilon} \right) + \frac{\omega}{2} \zeta_{\vlambda, k}^2 V_{\vlambda}.
    \end{align*}

    \textbf{Part III: bounding $\mathcal{P}_{k}(\chi)$.}~~Having bounded separately $a_{k}$ and $b_{k}$, and being $\mathcal{P}_{k}(\chi) = a_{k} + \chi b_{k}$, we can just put together the previously obtained results to have a bound on $\mathcal{P}_{k}(\chi)$. Moreover, exploiting Assumption~\ref{asm:L_grad_lip} and by noticing that $\cL_{\omega}$ satisfies the quadratic growth condition (since Lemma~\ref{lemma:propLLambda} states that $\cL_{\omega}$ satisfies the PL condition with $\omega$ as constant), we obtain the following inequality:
    \begin{align*}
        &a_{k+1} + \chi b_{k+1} \\
        &\le a_{k} + \chi \left( 1 - \frac{\zeta_{\vlambda, k}}{2} \omega \right) b_{k} \\
        &\quad + \left(2 \zeta_{\vupsilon, k} \left(1 + \frac{L_2}{2} \zeta_{\vupsilon, k}\right) \chi \left( 1 - \frac{\zeta_{\vlambda, k}}{2} \omega \right) \right.
        \end{align*}
    \begin{align*}
        &\qquad \left. - \frac{\zeta_{\vupsilon, k}}{2} \left( 1 + \chi\left( 1 - \frac{\zeta_{\vlambda, k}}{2}\omega\right) \right) \right) \E\left[\left\| \nabla_{\vupsilon} H_\omega(\vupsilon_{k}) \right\|_2^2\right] \\
        &\quad + \left( 2 \zeta_{\vupsilon, k} \left(1 + \frac{L_2}{2} \zeta_{\vupsilon, k}\right) \chi \left( 1 - \frac{\zeta_{\vlambda, k}}{2} \omega \right) + \frac{\zeta_{\vupsilon, k}}{2} \left( 1 + \chi\left( 1 - \frac{\zeta_{\vlambda, k}}{2}\omega\right) \right) \right)\frac{4 L_3^2}{\omega} b_{k}  \\
        &\quad + \frac{\zeta_{\vupsilon, k}^2}{2} \left( L_H  + \chi \left( 1 - \frac{\zeta_{\vlambda, k}}{2} \omega \right) (L_H + L_2) \right) V_{\vupsilon} + \chi \frac{\omega}{2} \zeta_{\vlambda, k}^2 V_{\vlambda}.
    \end{align*}

    \textbf{Part IV: applying the $\psi$-gradient domination.}~~Form the previously highlighted inequality, we aim at recovering a recursive equation in $\mathcal{P}_{k}(\chi)$. To this end, we apply Assumption~\ref{asm:wgd}, from which it follows that
    \begin{align*}
        \E \left[\left\| \nabla_{\vupsilon} H_\omega(\vupsilon) \right\|_2^2 \right] \ge \alpha_1^{\frac{2}{\psi}} &\max\left\{0, \; \E \left[H_\omega(\vupsilon) - \widetilde{H}^* \right] \right\}^{\frac{2}{\psi}},
    \end{align*}
    where $\widetilde{H}^* \coloneqq H^{*} + \beta_{1}/\alpha_{1}$. Now, exploiting this last result and enforcing $\chi \le \min\{ 1/5, \; 1/(\max_{k\in \dsb{K}} b_{k})\}$, after many algebraic steps, we obtain the following inequality:
    \begin{align*}
        \widetilde{P}_{k+1}(\chi) \le \widetilde{P}_k(\chi)  - \widetilde{C} \max\left\{ 0, \; \widetilde{P}_{k}(\chi) \right\}^{\frac{2}{\psi}} + \widetilde{V},
    \end{align*}
    where $\widetilde{P}_{k}(\chi) \coloneqq a_{k} + \chi b_{k} - \beta_{1}/\alpha_{1}$, $\widetilde{C} \coloneqq  2^{1- \frac{1}{\psi}} \frac{\zeta_{\vupsilon, k}\alpha_1^{\frac{2}{\psi}}}{2}$, and $\widetilde{V} \coloneqq \frac{\zeta_{\vupsilon, k}^2}{2}\left((1+2\chi)L_2 + \hll{(1+\chi)}\frac{L_1^2}{\omega}\right) V_{\vupsilon} + \chi  \frac{\omega}{2}\zeta_{\vlambda, k}^2 V_{\vlambda}$. We highlight that to get to this result, the learning rates have been selected as:
    \begin{align*}
        \zeta_{\vupsilon, k} \le \min\left\{ \frac{1}{L_H} , \frac{1}{L_2}, \frac{\omega^2 \chi \zeta_{\vlambda, k}}{(1+\chi)\omega \alpha_1^{\frac{2}{\psi}}  + 4L^2_3(1+7\chi)} \right\} \quad \text{and} \quad \zeta_{\vlambda, k} \le \frac{1}{\omega}.
    \end{align*}

    \textbf{Part V: rates computation.}~~Equipped with the recursive inequality reported in Part IV, we just have to compute the rates guaranteeing $\mathcal{P}_{K}(\chi) \le \epsilon + \beta_{1}/\alpha_{1}$. In particular, we first analyze the \emph{exact gradient} case, \ie $\widetilde{V} = 0$, for when $\psi = 2$, and $\psi \in [1,2)$. Then, we do the same in the case of \emph{estimated gradients}, \ie $\widetilde{V} > 0$. All the results are reported in Table~\ref{tab:summary}.
\end{proofsketch}

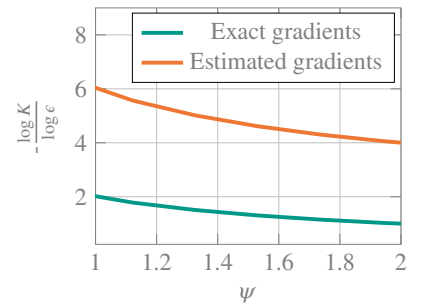
\begin{wrapfigure}{r}{0.33\textwidth}
\centering
\resizebox{\linewidth}{!}{
        \begin{tikzpicture}
        \begin{axis}[width=6cm, height=5cm, xmin=1, xmax=2, ymax=9, samples=50, xlabel=$\psi$, ylabel=-$\frac{\log K}{\log \epsilon} $, grid=both]
          \addplot[vibrantTeal, ultra thick] (x,2/x);\addlegendentry{Exact gradients}
          \addplot[vibrantOrange,  ultra thick] (x, 4/x+2);\addlegendentry{Estimated gradients}
        \end{axis}
    \end{tikzpicture}}

    \vspace{-0.2cm}

    \caption{Plot of the exponents of $\epsilon^{-1}$ in the cases of Table~\ref{tab:summary}.}\label{fig:summary}
\end{wrapfigure}

Some comments are in order. First, Theorem~\ref{thr:convergencePot} holds for a specific choice of the constant $\chi \in (0,1/5)$ defining the potential function $\mathcal{P}_K(\chi)$. Second, the presented rates hold for sufficiently small values of $\epsilon$ and $\omega$. This is just for presentation purposes, as the sample complexity\footnote{Theorem~\ref{thr:convergencePot} provides an \emph{iteration-complexity} guarantee. Concerning the \emph{estimated gradient} case, this translates into a \emph{sample complexity} guarantee since we are allowed to estimate gradients with a single sample.} can only improve if we increase the values of $\epsilon$ and $\omega$. Third, in the proof, an explicit expression of the learning rates is provided. Concerning their orders, for the case of exact gradients, we choose $\zeta_{\vlambda} = \omega^{-1}$ and $\zeta_{\vupsilon} = \cO(\omega)$, whereas for the estimated gradient case, we choose $\zeta_{\vlambda} = \cO(\omega \epsilon^{2/\psi})$ and $\zeta_{\vupsilon} = \cO(\omega^3 \epsilon^{2/\psi})$.

Assuming $\omega$ to be a constant, we observe that both learning rates display the same dependence on $\epsilon$ and, consequently, they are in \emph{single-time scale}. However, as we have seen in Theorem~\ref{thr:conversion}, in order to obtain guarantees on the original non-regularized problem, we have to set $\omega = \cO(\epsilon)$, leading to a \emph{two-time scales} algorithm. Fourth, we observe that, for both exact and estimated gradients, the sample complexity degrades as the constant $\psi$ of the gradient domination moves from $2$ to $1$, delivering the smallest sample complexity when the PL condition holds.
Finally, we highlight that \cpg jointly: ($i$) converges to the global optimum of the COP problem of Equation~\eqref{eq:gen_opt_prob}; ($ii$) delivers a \emph{last-iterate} guarantee; ($iii$) has no dependence on the cardinality of the state or action spaces, making it completely \emph{dimension-free}.
Table~\ref{tab:summary} and Figure~\ref{fig:summary} summarize the results of Theorem~\ref{thr:convergencePot}.

\subsection{\textcolor{vibrantBlue}{Action-based} and \textcolor{vibrantRed}{Parameter-based} Variants of \cpg} \label{subsec:alg-specific}

So far, we have focused on the exploration-agnostic formulation of the proposed method. In this section, we introduce its action-based and parameter-based variants, namely \cpgae and \cpgpe.
These variants differ in the considered cost functions, $\Ja$ or $\Jp$ for every $i \in \dsb{0,U}$, and the estimators employed to update the optimization variables.
We recall that the general Lagrangian function for the problem in Equation~(\ref{eq:gen_opt_prob}) is the following:
\begin{align*}
    \cL_{\dagger, \omega} (\vupsilon, \vlambda) \coloneqq J_{\dagger, 0}(\vupsilon ) + \sum_{u=1}^{U} \lambda_{u} \left( J_{\dagger, u}(\vupsilon ) - b_u \right) - \frac{\omega}{2} \left\| \vlambda \right\|_2^2,
\end{align*}
being $\vupsilon \in \mathcal{V}$ a generic parameter vector to be optimized. As highlighted in Section~\ref{sec:setting}, in the case of action-based exploration ($\dagger = \text{\textcolor{vibrantBlue}{A}}$) $\vupsilon$ corresponds to the policy parameterization $\vupsilon = \vtheta \in \Theta$, while in the case of parameter-based exploration ($\dagger = \text{\textcolor{vibrantRed}{P}}$) it coincides with the hyperpolicy parameterization $\vupsilon = \vrho \in \mathcal{R}$. In the following, we are going to consider the gradients \wrt both the parameterization $\vupsilon$ and the Lagrange multipliers $\vlambda$, having the following explicit forms in the exploration-agnostic setting:
\begin{align*}
    \nabla_{\vupsilon} \cL_{\dagger,\omega}(\vupsilon,\vlambda) = \nabla_{\vupsilon} J_{\dagger, 0}(\vupsilon ) + \sum_{u=1}^{U} \lambda_{u} \nabla_{\vupsilon} J_{\dagger, u}(\vupsilon ) \quad \text{and} \quad \nabla_{\vlambda} \cL_{\dagger,\omega}(\vupsilon,\vlambda) = \mathbf{J}_{\dagger}(\vupsilon) + \mathbf{b} + \omega \vlambda,
\end{align*}
where $\mathbf{J}_{\dagger}(\vupsilon) \coloneqq (J_{\dagger,1}(\vupsilon), \ldots, J_{\dagger,U}(\vupsilon))^{\top}$ and $\mathbf{b} \coloneqq (b_{1}, \ldots, b_{U})^{\top}$.

\begin{table}[t]
{\renewcommand{\arraystretch}{1.5}
\medmuskip=0mu
\thinmuskip=0mu
\thickmuskip=0mu
\begin{tabular}{|c||c|c|c||c|c|c|}
    \hline
    & \multicolumn{3}{|c||}{\cellcolor{vibrantTeal!40} \bfseries Exact Gradients} & \multicolumn{3}{|c|}{\cellcolor{vibrantOrange!40} \bfseries Estimated Gradients} \\\cline{2-7}
    & \cellcolor{vibrantTeal!20} $\psi=1$ (GD) & \cellcolor{vibrantTeal!20} $\psi \in (1,2)$ & \cellcolor{vibrantTeal!20} $\psi=2$ (PL) & \cellcolor{vibrantOrange!20} $\psi=1$ (GD) & \cellcolor{vibrantOrange!20} $\psi \in (1,2)$ & \cellcolor{vibrantOrange!20} $\psi=2$ (PL)   \\ \hline\hline
    \cellcolor{vibrantGrey!20} \bfseries Fixed $\omega$ & $\omega^{-1} \epsilon^{-1}$ & $\omega^{-1} \epsilon^{-\frac{2}{\psi}+1}$ & $\omega^{-1} \log (\epsilon^{-1}) $ &$\omega^{-3} \epsilon^{-3}\log(\epsilon^{-1})$ & $\omega^{-3} \epsilon^{-\frac{4}{\psi}+1}\log(\epsilon^{-1})$ & $\omega^{-3} \epsilon^{-1}\log(\epsilon^{-1})$ \\ \hline
    \cellcolor{vibrantGrey!20} \bfseries $\omega= \cO(\epsilon)$ & $ \epsilon^{-2}$ & $ \epsilon^{-\frac{2}{\psi}}$ & $\epsilon^{-1} \log (\epsilon^{-1})$ & $\epsilon^{-6}\log(\epsilon^{-1})$ & $\epsilon^{-\frac{4}{\psi}-2}\log(\epsilon^{-1})$ & $\epsilon^{-4}\log(\epsilon^{-1})$ \\ \hline
\end{tabular}
}
\caption{Summary of the sample complexity results of \cpg when either keeping $\omega$ fixed or setting it as $\omega = \cO(\epsilon)$
.}
\label{tab:summary}
\end{table}

\paragraph{\textcolor{vibrantBlue}{Action-based} Exploration for \emph{\cpg}.}
The action-based variant of \cpg, referred to as \cpgae, aims at optimizing the parameters $\vtheta$ of a parametric stochastic policy $\pi_{\vtheta}$. In particular, for every $i \in \dsb{0,U}$, we recall the definition of action-based cost functions: $\Ja(\vtheta) = \E_{\tau \sim \pA(\cdot | \vtheta)} \left[ C_i(\tau) \right]$, where $\pA(\tau, \vtheta)$ is the probability density of trajectory $\tau$ induced by the policy $\pi_{\vtheta}$.

Considering the gradient \wrt the parameters $\vtheta$, the following holds:
\begin{align*}
    \nabla_{\vtheta} \Ja(\vtheta) &= \nabla_{\vtheta} \E_{\tau \sim \pA(\cdot | \vtheta)} \left[ C_i(\tau) \right] = \E_{\tau \sim \pA(\cdot \mid \vtheta)} \left[ \nabla_{\vtheta} \log \pA(\tau , \vtheta) C_i(\tau) \right].
\end{align*}

As for standard PGs~\citep{williams1992simple,baxter2001infinite}, we can switch to its sample-based version to obtain an unbiased estimator of the gradient. In particular, we resort to a GPOMDP-like version~\citep{baxter2001infinite} for the proposed estimator:
\begin{align*}
    \widehat{\nabla}_{\vtheta} \Ja(\vtheta) \coloneqq \frac{1}{N} \sum_{j=1}^{N} \sum_{t=0}^{T-1} \left( \sum_{l=0}^{t} \nabla_{\vtheta} \log \pi_{\vtheta}(\bm{a}_{\tau_j,l}, \bm{s}_{\tau_j,l}) \right) \gamma^{t} c_{i}(\bm{s}_{\tau_{j},t}, \bm{a}_{\tau_{j},t}),
\end{align*}
where $N$, called batch size, is the number of independent trajectories $\{\tau_{j}\}_{j=1}^{N}$ such that $\tau_{j} \sim \pA(\cdot \mid \vtheta)$. We just consider a GPOMDP-like version of the estimator, since the REINFORCE-like one would suffer from a higher variance as for standard PGs~\citep{papini2022smoothing}. Thus, considering that:
\begin{align*}
    \rLa(\vtheta,\vlambda) = \textcolor{vibrantBlue}{J_{\text{A},0}}(\vtheta) + \sum_{u=1}^{U} \lambda_{u} \left(\Ja(\vtheta) - b_{u} \right) - \frac{\omega}{2} \left\| \vlambda \right\|_{2}^{2},
\end{align*}
\cpgae employs the following estimator to update the primal variable $\vtheta$:
\begin{align*}
    \widehat{\nabla}_{\vtheta} \rLa(\vtheta,\vlambda) = \widehat{\nabla}_{\vtheta} \textcolor{vibrantBlue}{J_{\text{A},0}}(\vtheta) + \sum_{u=1}^{U} \lambda_{u} \widehat{\nabla}_{\vtheta} \Ja(\vtheta).
\end{align*}

If we now focus on the action-based gradient \wrt Lagrange multipliers, we have the following:
\begin{align*}
    \nabla_{\vlambda} \rLa(\vtheta,\vlambda) = \textcolor{vibrantBlue}{\mathbf{J}_{\text{A}}}(\vtheta) - \mathbf{b} - \omega \vlambda,
\end{align*}
where $\textcolor{vibrantBlue}{\mathbf{J}_{\text{A}}}(\vtheta) \coloneqq (\textcolor{vibrantBlue}{J_{\text{A},1}}(\vtheta), \ldots, \textcolor{vibrantBlue}{J_{\text{A},U}}(\vtheta))^{\top}$. Thus, its sample-based version used by \cpgae to update the dual variable is the following:
\begin{align*}
    \frac{\widehat{\partial}}{\partial \lambda_i} \rLa (\vtheta, \vlambda) = \frac{1}{N} \sum_{j=1}^{N} C_i(\tau_j)  - b_i - \omega \lambda_i,
\end{align*}
where the $N$ independent trajectories $\{\tau_j\}_{j=1}^{N}$ are such that $\tau_j \sim \pA(\cdot \mid \vtheta)$.

\paragraph{\textcolor{vibrantRed}{Parameter-based} Exploration for \emph{\cpg}.}
The parameter-based variant of \cpg, referred to as \cpgpe, aims at optimizing the parameters $\vrho$ of a parametric stochastic hyperpolicy $\nu_{\vrho}$, used to sample the parameters $\vtheta$ for an underlying parametric policy $\pi_{\vtheta}$ (which can also be deterministic, as we shall see in the next section). In particular, for every $i \in \dsb{0,U}$, we recall the definition of parameter-based cost functions: $\Jp(\vrho) = \E_{\vtheta \sim \nu_{\vrho}} \left[\E_{\tau \sim \pA(\cdot | \vtheta)} \left[ C_i(\tau) \right]\right]$, where $\pA(\tau, \vtheta)$ is the probability density of trajectory $\tau$ induced by the policy $\pi_{\vtheta}$.

Considering the gradient \wrt the parameters $\vtheta$, the following holds:
\begin{align*}
    \nabla_{\vrho} \Jp(\vrho) &= \nabla_{\vrho} \E_{\vtheta \sim \nu_{\vrho}} \left[ \E_{\tau \sim \pA(\cdot | \vtheta)} \left[ C_i(\tau) \right] \right] = \E_{\vtheta \sim \nu_{\vrho}} \left[ \nabla_{\vrho} \log \nu_{\vrho}(\vtheta) \E_{\tau \sim \pA(\cdot, \vtheta)} \left[ \nabla_{\vtheta} \log \pA(\tau | \vtheta) C_i(\tau) \right] \right].
\end{align*}

As for the prototypical parameter-based method PGPE~\citep{SEHNKE2010551}, we switch to its sample-based version to obtain an unbiased estimator of the gradient:
\begin{align*}
    \widehat{\nabla}_{\vrho} \Jp(\vrho) \coloneqq \frac{1}{N} \sum_{j=1}^{N} \nabla_{\vrho} \log \nu_{\vrho}(\vtheta_{j}) C_{i}(\tau_{j}),
\end{align*}
where $N$, called batch size, is the number of independent parameter-trajectory pairs $\{(\vtheta_{j},\tau_{j})\}_{j=1}^{N}$ such that $\tau_{j} \sim \pA(\cdot \mid \vtheta_{j})$ and $\vtheta_{j} \sim \nu_{\vrho}$. We highlight that, for parameter-based exploration, we sample $N$ policy parameterizations $\vtheta_{j}$ from the hyperpolicy $\nu_{\vrho}$, for each of which we sample a single trajectory $\tau_j \sim \pA(\cdot \mid \vtheta_{j})$. That being said, considering that:
\begin{align*}
    \rLp(\vrho,\vlambda) = \textcolor{vibrantRed}{J_{\text{P},0}}(\vrho) + \sum_{u=1}^{U} \lambda_{u} \left(\Jp(\vrho) - b_{u} \right) - \frac{\omega}{2} \left\| \vlambda \right\|_{2}^{2},
\end{align*}
\cpgpe employs the following estimator to update the primal variable $\vrho$:
\begin{align*}
    \widehat{\nabla}_{\vrho} \rLp(\vrho,\vlambda) = \widehat{\nabla}_{\vrho} \textcolor{vibrantRed}{J_{\text{P},0}}(\vrho) + \sum_{u=1}^{U} \lambda_{u} \widehat{\nabla}_{\vrho} \Jp(\vrho).
\end{align*}

If we now focus on the parameter-based gradient \wrt Lagrange multipliers, we have the following:
\begin{align*}
    \nabla_{\vlambda} \rLp(\vrho,\vlambda) = \textcolor{vibrantRed}{\mathbf{J}_{\text{P}}}(\vrho) - \mathbf{b} - \omega \vlambda,
\end{align*}
where $\textcolor{vibrantRed}{\mathbf{J}_{\text{P}}}(\vrho) \coloneqq (\textcolor{vibrantRed}{J_{\text{P},1}}(\vrho), \ldots, \textcolor{vibrantRed}{J_{\text{P},U}}(\vrho))^{\top}$. Thus, its sample-based version used by \cpgpe to update the dual variable is the following:
\begin{align*}
    \frac{\widehat{\partial}}{\partial \lambda_i} \rLp (\vrho, \vlambda) = \frac{1}{N} \sum_{j=1}^{N} C_i(\tau_j)  - b_i - \omega \lambda_i,
\end{align*}
where the $N$ independent trajectories $\{\tau_j\}_{j=1}^{N}$ are such that $\tau_j \sim \pA(\cdot \mid \vtheta_j)$, where $\vtheta_j \sim \nu_{\vrho}$ for every $j \in \dsb{N}$.

%% file: contents/04_deterministic.tex
\section{Deterministic Policy Deployment of \cpg}
\label{sec:deterministic}
In this section, we analyze the convergence guarantees of \cpg towards an optimal deterministic policy. To this end, we focus on the setting of \emph{white noise (hyper)policies}~\citep{montenegro2024learning}. Specifically, we restrict both action-based and parameter-based exploration strategies as stochastic perturbations of an underlying parametric deterministic policy~$\mu_{\vtheta}$. In this framework, stochastic policies~$\pi_{\vtheta}$ are modeled as perturbations of the actions prescribed by~$\mu_{\vtheta}$, while stochastic hyperpolicies~$\nu_{\vrho}$ are interpreted as perturbations of the parameters~$\vtheta$ of the deterministic policy.
Leveraging this structure, we study the behavior of the \cpg algorithm when it learns using stochastic (hyper)policies and subsequently \emph{deploys} their deterministic counterpart by \emph{switching off} the stochasticity at the end of training.
Our analysis establishes the sample complexity required by \cpg to guarantee that the deployed deterministic policy, regardless of the exploration paradigm used during training, is an optimal feasible one.

We begin by presenting the noise model used in our analysis (Section~\ref{subsec:white_noise}), followed by a description of the deterministic policy deployment process in CMDPs (Section~\ref{subsec:deploying}). We then state the assumptions required for convergence (Section~\ref{subsec:conditions_convergence}), and finally provide the sample complexity required by \cpg to converge in the last iterate to an optimal deterministic policy when the noise is \emph{switched off} at the end of training (Section~\ref{subsec:convergence_deterministic}).

\subsection{White Noise Exploration}
\label{subsec:white_noise}
While deterministic policies are desirable in real-world applications (see Section~\ref{sec:intro}), learning them directly typically requires off-policy actor-critic architectures~\citep{silver2014deterministic,lillicrap2016deterministic,xiong2022deterministic}, which pose significant challenges for convergence analysis even in unconstrained settings. In this part, we introduce a specific noise model that enables us to restrict both action-based and parameter-based exploration strategies as stochastic perturbations of deterministic policies. This formulation allows us to quantify the performance gap induced by a given parameterization of a stochastic (hyper)policy \wrt its associated deterministic policy.

We begin by defining deterministic policies and the related performance and cost functions. A parametric deterministic policy is a function $\mu_{\vtheta}: \mathcal{S} \rightarrow \mathcal{A},$
where $\vtheta \in \Theta \subseteq \Reals^{\dt}$ is the parameter vector.
For every $i \in \dsb{0,U}$, the performance and cost functions $\Jd:\Theta \to \Reals$ related to a deterministic policy are:
\begin{align*}
\label{eq:deterministic_performance_cost}
    \Jd(\vtheta) \coloneqq \E_{\tau \sim \pD(\cdot | \vtheta)} \left[C_i(\tau)\right],
\end{align*}
where $\pD(\tau; \vtheta) \coloneqq \phi_{0}(\bm{s}_{0}) \prod_{t=0}^{T-1} p(\bm{s}_{t+1} | \bm{s}_{t}, \mu_{\vtheta}(\bm{s}_{t}))$ is the probability density of trajectory $\tau$ induced by $\mu_{\vtheta}$. Using these definitions, we introduce the deterministic regularized Lagrangian function employed by \cpg as:
\begin{align*}
    \rLd(\vtheta,\vlambda) = J_{\text{D}, 0}(\vtheta) + \sum_{i=1}^{U}\left(\lambda_{i}(\Jd(\vtheta) - b_{i})\right) - \frac{\omega}{2}\|\bm{\vlambda}\|_{2}^{2}.    
\end{align*}

We can now redefine both \textcolor{vibrantBlue}{AB} and \textcolor{vibrantRed}{PB} exploration on top of deterministic policies~\citep{montenegro2024learning}. Considering \textcolor{vibrantBlue}{AB} exploration, stochasticity is injected at the action level, perturbing the deterministic policy’s output at each environmental interaction step. In \textcolor{vibrantRed}{PB} exploration, noise is applied directly to the policy parameters before execution, resulting in a fixed perturbed version of the underlying deterministic policy for an entire trajectory. Next, we formally present how we intend a perturbation in the action or parameter spaces.

\begin{defi}[White Noise] \label{def:white_noise}
    Let $d\in \mathbb{N}$ and $\sigma \in \mR_{>0}$. A probability distribution $\Phi_d \in \Delta\left(\mR^d\right)$ is defined as \emph{white noise} if it satisfies the following conditions:
    \begin{align}
        \mathbb{E}_{\vepsilon \sim \Phi_d}[\vepsilon] = \bm{0}_d, \quad
        \mathbb{E}_{\vepsilon \sim \Phi_d}[\left\| \vepsilon\right\|_{2}^{2}] \leq d\sigma^2,
    \end{align}
    where $\bm{0}_{d}\in \mathbb{R}^{d}$ is a $d$-dimensional vector of all zero components.
\end{defi}
Definition~\ref{def:white_noise} includes zero-mean Gaussian distributions $\vepsilon \sim \mathcal{N}(\bm{0}_{d}, \sigma \bm{\Lambda})$ with $\lambda_{\max} (\bm{\Lambda}) = 1$, ensuring that $\mathbb{E}[\|\vepsilon \|_{2}^{2}] = \sigma^{2} \text{tr}(\bm{\Lambda}) \leq d \sigma^{2}$. We stress that this noise has to be considered \emph{white} across exploration steps.  We can now redefine action-based and parameter-based explorations as white noise perturbations of the actions or the parameters of an underlying parametric deterministic policy $\mu_{\vtheta}$.

\paragraph{\textcolor{vibrantBlue}{Action-Based} (\textcolor{vibrantBlue}{AB}) Exploration.}
Considering \textcolor{vibrantBlue}{AB} PG methods (see Section \ref{sec:setting}), we consider a \emph{parametric stochastic policy} $\pi_{\vtheta}$ as built upon an underlying deterministic policy $\mu_{\vtheta}$ by perturbing each action suggested by $\mu_{\vtheta}$ with a white noise random vector. Formally, we consider the following definition of white noise policies.

\begin{defi}[\textcolor{vibrantBlue}{White Noise Policies}] \label{def:white_noise_policies}
    Let $\vtheta \in \Theta$ and $\mu_{\vtheta}: \mathcal{S} \to \mathcal{A}$ be a parametric deterministic policy. Given a white noise distribution $\Phi_{d_{\mathcal{A}}}$ (Definition~\ref{def:white_noise}), a white-noise-based policy $\pi_{\vtheta}: \mathcal{S} \rightarrow \Delta(\mathcal{A})$ is defined such that, for every state $\bm{s} \in \mathcal{S}$, the action $\bm{a} \sim \pi_{\vtheta}(\cdot|\bm{s})$ satisfies $\bm{a} = \mu_{\vtheta}(\bm{s}) + \vepsilon$, where $\vepsilon \sim \Phi_{d_{\mathcal{A}}}$ which is sampled independently at every step (\ie whenever an action is sampled).
\end{defi}

We highlight that Definition \ref{def:white_noise_policies} further justifies the name for \textcolor{vibrantBlue}{AB} exploration since the exploration is carried out at the action level.

Next, for every $i \in \dsb{0,U}$ we redefine the cost functions $\Ja$ leveraging the introduced characterization of white noise policies.
To this end, we need to introduce the concept of \emph{non stationary} deterministic cost functions. Let $\uvepsilon = (\vepsilon_{t})_{t=0}^{T-1}$ be a sequence of independently sampled white noise vectors satisfying Definition~\ref{def:white_noise}. Let $\uvmu = (\vmu_{t})_{t=0}^{T-1}$ be a non stationary deterministic policy where, at time step $t$, the deterministic policy $\vmu_{t}: \cS \to \cA$ is played, with $\vmu_{t} = \mu_{\vtheta} + \vepsilon_{t}$. For every $i \in \dsb{0,U}$, we introduce the cost functions for this kind of policy: $\Jd(\uvmu) \coloneqq \E_{\tau \sim \pD(\cdot \mid \uvmu)}[C_i(\tau)]$, where $\pD(\tau, \uvmu)$ is the density of a trajectory $\tau$ induced by the non stationary deterministic policy $\uvmu$.
Equipped with this new definition, we can reintroduce the \textcolor{vibrantBlue}{AB} cost functions $\Ja$ which admit the following definition, together with the one already provided in Section~\ref{sec:setting}, when the considered stochastic policy $\pi_{\vtheta}$ complies with Definition~\ref{def:white_noise_policies}:
\begin{align*}
    \Ja(\vtheta) \coloneqq \E_{\vepsilon \sim \Phi_{\da}^{T}} \left[ \Jd(\uvmu_{\vtheta} + \uvepsilon) \right],
\end{align*}
where $\uvmu_{\vtheta} + \uvepsilon = (\mu_{\vtheta} + \vepsilon_{t})_{t=0}^{T-1}$ and $\Phi_{\da}$ is a white noise distribution compliant with Definition~\ref{def:white_noise}..

\paragraph{\textcolor{vibrantRed}{Parameter-Based} (\textcolor{vibrantRed}{PB}) Exploration.}
Considering \textcolor{vibrantRed}{PB} PG methods (see Section \ref{sec:setting}), we consider a \emph{parametric stochastic hyperpolicy} $\nu_{\vtheta}$ as built upon an underlying deterministic policy $\mu_{\vtheta}$ by perturbing the parameter vector $\vtheta$ with a white noise random vector. Formally, we consider the following definition of white-noise hyperpolicies.

\begin{defi}[\textcolor{vibrantRed}{White Noise Hyperpolicies}] \label{def:white_noise_hyperpolicies}
    Let $\vtheta \in \Theta$ and $\mu_{\vtheta}: \mathcal{S} \to \mathcal{A}$ be a parametric deterministic policy. Given a white noise distribution $\Phi_{\dt}$ (Definition~\ref{def:white_noise}), a white-noise-based hyperpolicy $\nu_{\vtheta} \in \Delta(\Theta)$ is defined such that, for every parameter $\vtheta \in \Theta$, the perturbed parameter $\vtheta' \sim \nu_{\vtheta}$ satisfies $ \vtheta' = \vtheta + \vepsilon$, where $\vepsilon \sim \Phi_{\dt}$, independently for every trajectory.
\end{defi}

As previously done for action-based exploration, we stress that this definition further justifies the name of \textcolor{vibrantRed}{PB} exploration, since the exploration is carried out at parameter level.
Moreover, we let the reader note that the noise $\vepsilon$ is sampled once at the \emph{beginning} of each trajectory, meaning that the resulting policy $\mu_{\vtheta+\vepsilon}$ remains deterministic throughout the entire trajectory collection phase.

Next, for every $i \in \dsb{0,U}$ we redefine the cost functions $\Jp$ leveraging the introduced characterization of white noise hyperpolicies.
We reintroduce the \textcolor{vibrantRed}{PB} cost functions, which admit the following definition, together with the one already provided in Section~\ref{sec:setting}, when the considered stochastic hyperpolicy $\nu_{\vtheta}$ complies with Definition~\ref{def:white_noise_hyperpolicies}:
\begin{align*}
    \Jp(\vtheta) = \E_{\vepsilon \sim \Phi_{\dt}} \left[ \Jd(\vtheta+\vepsilon) \right],
\end{align*}
where $\Phi_{\dt}$ is a white noise distribution compliant with Definition~\ref{def:white_noise}.

In the remaining part of this section, we will consider an \emph{exploration-agnostic} setting, in which we denote with $\dagger \in \{\textcolor{vibrantBlue}{\text{A}}, \textcolor{vibrantRed}{\text{P}}\}$ the two different exploration approaches. We highlight that the problem formulation and all the theoretical results of Section~\ref{sec:alg_general} still hold.

\subsection{Deploying Deterministic Policies in CMDPs}
\label{subsec:deploying}

In this part, we analyze the effect of switching off the stochasticity on the regularized Lagrangian employed by \cpg to solve the COP in Equation~\eqref{eq:gen_opt_prob} when dealing with stochastic policies and hyperpolicies that satisfy Definitions~\ref{def:white_noise_policies} and~\ref{def:white_noise_hyperpolicies}, respectively.

Before presenting such results, we introduce two assumptions enforcing the regularity of the deterministic objectives $\Jd$ \wrt the parameters $\vtheta$ and the non stationary deterministic policies $\uvmu$ associated with \textcolor{vibrantBlue}{AB} exploration.

\begin{ass}[$\Jd$ Regularity \wrt $\vtheta$] \label{asm:Jd-reg-theta}
    For every $i \in \dsb{0,U}$, there exist $\LCJd,\LSJd \in \mR_{> 0}$ such that, for every $\vtheta,\vtheta' \in \Theta$, the following conditions hold:
    \begin{align*}
        \left| \Jd(\vtheta) - \Jd(\vtheta') \right| \le \LCJd \left\| \vtheta - \vtheta' \right\|_{2} \quad \text{and} \quad \left\| \nabla_{\vtheta} \Jd(\vtheta) - \nabla_{\vtheta} \Jd(\vtheta') \right\|_{2} \le \LSJd \left\| \vtheta - \vtheta' \right\|_{2}.
    \end{align*}
    Moreover, we denote $\LCJdmax \coloneqq \max_{i \in \dsb{0,U}} \LCJd$ and $\LSJdmax \coloneqq \max_{i \in \dsb{0,U}} \LSJd$.
\end{ass}

\begin{ass}[$\Jd$ Regularity \wrt $\uvmu$] \label{asm:Jd-reg-ns}
    For every $i \in \dsb{0,U}$, there exist $\LCns,\LSns \in \mR_{> 0}$ such that, for every pair of non stationary deterministic policies $\uvmu,\uvmu'$, the following conditions hold:
    \begin{align*}
        &\left| \Jd(\uvmu) - \Jd(\uvmu') \right| \le \LCns \sum_{t=0}^{T-1} \sup_{\bm{s} \in \cS} \left\| \vmu_{t}(\bm{s}) - \vmu_{t}'(\bm{s}) \right\|_{2} \quad \text{and} \\ 
        &\left\| \nabla_{\uvmu}\Jd(\uvmu) - \nabla_{\uvmu} \Jd(\uvmu') \right\|_{2} \le \LSns \sum_{t=0}^{T-1} \sup_{\bm{s} \in \cS} \left\| \vmu_{t}(\bm{s}) - \vmu_{t}'(\bm{s}) \right\|_{2}.
    \end{align*}
    Moreover, we denote $\LCnsmax \coloneqq \max_{i \in \dsb{0,U}} \LCns$ and $\LSnsmax \coloneqq \max_{i \in \dsb{0,U}} \LSns$.
\end{ass}

We stress that these assumptions will be crucial for presenting the core result of this section regarding the effects on $\rLg$ when switching off the stochasticity in the context of white noise exploration (see Section~\ref{subsec:white_noise}). Additionally, we let the reader note that Assumption~\ref{asm:Jd-reg-theta} induces both $\Ja$ and $\Jp$ to enjoy the same regularity condition stated in such an assumption when considering (hyper)policies complying with Definitions~\ref{def:white_noise_policies} and~\ref{def:white_noise_hyperpolicies}~\citep{montenegro2024learning}. Moreover, the Lipschitz constants are fully characterized in~\citep{montenegro2024learning}.

We are now ready to analyze the effect of switching off the stochasticity in \textcolor{vibrantRed}{PB} and \textcolor{vibrantBlue}{AB} exploration on the regularized Lagrangian $\rLg$ employed by the \cpg method.
\begin{thr} \label{thr:L-deploy}
    Considering (hyper)policies complying with Definitions~\ref{def:white_noise_policies} (\textcolor{vibrantBlue}{AB}) or~\ref{def:white_noise_hyperpolicies} (\textcolor{vibrantRed}{PB}), under Assumptions~\ref{asm:Jd-reg-theta} (\textcolor{vibrantRed}{PB}) or~\ref{asm:Jd-reg-ns} (\textcolor{vibrantBlue}{AB}), the following results hold:
    \begin{enumerate}[noitemsep,topsep=0pt,label=\textit{\roman{*}.}, ref=(\roman{*})]
        \item (Uniform Bound) for every $\vtheta \in \Theta$ and $\vlambda \in \mR^{U}_{\ge 0}$: $$\left| \rLd(\vtheta,\vlambda) - \rLg(\vtheta,\vlambda) \right| \le \left(1 + \left\| \vlambda \right\|_{1}\right) L_{1\dagger} \sigma \sqrt{d_{\dagger}}.$$
        \item ($\rLd$ Upper Bound) let $(\vtheta_{\text{D},\omega}^{*}, \vlambda_{\text{D},\omega}^{*})$ be a saddle point of $\rLd$ and let $(\vtheta^{*}_{\dagger,\omega}, \vlambda^{*}_{\dagger,\omega})$ be a saddle point of $\rLg$. Then:
        $$\rLd(\vtheta^{*}_{\dagger,\omega}, \vlambda^{*}_{\dagger,\omega}) - \rLd(\vtheta_{\text{D},\omega}^{*}, \vlambda_{\text{D},\omega}^{*}) \le 2 \left(1 + \| \vlambda^{*}_{\dagger,\omega} \|_{1}\right) L_{1\dagger} \sigma \sqrt{d_{\dagger}}.$$
    \end{enumerate}
    Where $L_{1\text{\textcolor{vibrantRed}{P}}} \coloneqq \LCJdmax$, $L_{1\text{\textcolor{vibrantBlue}{A}}} \coloneqq \LCnsmax$, $d_{\text{\textcolor{vibrantRed}{P}}}\coloneqq \dt$, and $d_{\text{\textcolor{vibrantBlue}{A}}}\coloneqq \da$.
\end{thr}
\begin{proof}
    We start the derivation by recalling the explicit form of $\left| \rLd(\vtheta,\vlambda) - \rLg(\vtheta,\vlambda) \right|$:
    \begin{align*}
        &\left| \rLd(\vtheta,\vlambda) - \rLg(\vtheta,\vlambda) \right| \\
        &= \left| \Jdzero(\vtheta) + \sum_{i=1}^{U} \lambda_{i} \left( \Jd(\vtheta) - b_{i} \right) - \frac{\omega}{2} \left\| \vlambda \right\|_{2}^{2} -\Jgenzero(\vtheta) - \sum_{i=1}^{U} \lambda_{i} \left( \Jgen(\vtheta) - b_{i} \right) + \frac{\omega}{2} \left\| \vlambda \right\|_{2}^{2} \right| \\
        &\le \left| \Jdzero(\vtheta) - \Jgenzero(\vtheta) \right| + \sum_{i=1}^{U} \lambda_{i} \left| \Jd(\vtheta) - \Jgen(\vtheta) \right|,
    \end{align*}
    where the last line follows by simply having applied the triangular inequality. 
    
    To continue the proof, we need to resort to Theorems~5.1 (\textcolor{vibrantRed}{PB}) and~5.2 (\textcolor{vibrantBlue}{AB}) by~\citep{montenegro2024learning}. These state that under Assumptions~\ref{asm:Jd-reg-theta} (\textcolor{vibrantRed}{PB}) or~\ref{asm:Jd-reg-ns} (\textcolor{vibrantBlue}{AB}), when dealing with an (hyper)policy complying with Definitions~\ref{def:white_noise_policies} (\textcolor{vibrantBlue}{AB}) and~\ref{def:white_noise_hyperpolicies} (\textcolor{vibrantRed}{PB}), the following holds:
    \begin{align*}
        \left| \Jd(\vtheta) - \Jgen(\vtheta) \right| \le L_{1\dagger,i} \sigma \sqrt{d_{\dagger}},
    \end{align*}
    where $L_{1\text{\textcolor{vibrantRed}{P}},i} \coloneqq \LCJd$, $L_{1\text{\textcolor{vibrantBlue}{A}},i} \coloneqq \LCns$, $d_{\text{\textcolor{vibrantRed}{P}}} = \dt$, and $d_{\text{\textcolor{vibrantBlue}{A}}} = \da$.

    By leveraging this result, the following holds:
    \begin{align*}
        \left| \rLd(\vtheta,\vlambda) - \rLg(\vtheta,\vlambda) \right| &\le \left| \Jdzero(\vtheta) - \Jgenzero(\vtheta) \right| + \sum_{i=1}^{U} \lambda_{i} \left| \Jd(\vtheta) - \Jgen(\vtheta) \right| \\
        &\le \left(1 + \sum_{i=1}^{U} \lambda_{i}\right) L_{1\dagger} \sigma \sqrt{d_{\dagger}} \\
        &= \left(1 + \left\| \vlambda \right\|_{1}\right) L_{1\dagger} \sigma \sqrt{d_{\dagger}},
    \end{align*}
    being $L_{1\text{\textcolor{vibrantRed}{P}}} \coloneqq \LCJdmax$ and $L_{1\text{\textcolor{vibrantBlue}{A}}} \coloneqq \LCnsmax$, which concludes the first part of the proof.

    We can now face the second part of the proof. In particular, let $(\vtheta^{*}_{\dagger,\omega}, \vlambda^{*}_{\dagger,\omega})$ be a saddle point of $\rLg$ and let $(\vtheta^{*}_{\text{D},\omega}, \vlambda^{*}_{\text{D},\omega})$ be a saddle point of $\rLd$. Before going on with the derivation, we recall that a saddle point by definition satisfies the following property:
    \begin{align*}
        \rLd(\vtheta^{*}_{\text{D},\omega}, \vlambda) \le \rLd(\vtheta^{*}_{\text{D},\omega}, \vlambda^{*}_{\text{D},\omega}) \le \rLd(\vtheta, \vlambda^{*}_{\text{D},\omega}),
    \end{align*}
    for every $\vtheta \in \Theta$ and $\vlambda \in \mR^{U}_{\ge 0}$. That being said, the following holds:
    \begin{align*}
        &\rLd(\vtheta^{*}_{\dagger,\omega}, \vlambda^{*}_{\dagger,\omega}) - \rLd(\vtheta_{\text{D},\omega}^{*}, \vlambda_{\text{D},\omega}^{*}) \\
        &\le \rLd(\vtheta^{*}_{\dagger,\omega}, \vlambda^{*}_{\dagger,\omega}) - \rLd(\vtheta_{\text{D},\omega}^{*}, \vlambda^{*}_{\dagger,\omega}) \\
        &= \rLd(\vtheta^{*}_{\dagger,\omega}, \vlambda^{*}_{\dagger,\omega}) - \rLd(\vtheta_{\text{D},\omega}^{*}, \vlambda^{*}_{\dagger,\omega}) \pm \rLg(\vtheta^{*}_{\dagger,\omega}, \vlambda^{*}_{\dagger,\omega}) \\
        &\le \rLd(\vtheta^{*}_{\dagger,\omega}, \vlambda^{*}_{\dagger,\omega}) - \rLg(\vtheta^{*}_{\dagger,\omega}, \vlambda^{*}_{\dagger,\omega}) + \rLg(\vtheta_{\text{D},\omega}^{*}, \vlambda^{*}_{\dagger,\omega}) - \rLd(\vtheta_{\text{D},\omega}^{*}, \vlambda^{*}_{\dagger,\omega}) \\
        &\le \left| \rLd(\vtheta^{*}_{\dagger,\omega}, \vlambda^{*}_{\dagger,\omega}) - \rLg(\vtheta^{*}_{\dagger,\omega}, \vlambda^{*}_{\dagger,\omega}) \right| + \left| \rLg(\vtheta_{\text{D},\omega}^{*}, \vlambda^{*}_{\dagger,\omega}) - \rLd(\vtheta_{\text{D},\omega}^{*}, \vlambda^{*}_{\dagger,\omega}) \right| \\
        &\le 2 \left( 1 + \left\| \vlambda^{*}_{\dagger,\omega} \right\|_{1} \right) L_{1\dagger} \sigma \sqrt{d_{\dagger}},
    \end{align*}
    where we have just exploited the previously recalled property of saddle points and, in the last line, the result proved in the first part of this proof.
\end{proof}

Some comments are in order. Theorem~\ref{thr:L-deploy} quantifies two sources of error: $(i)$ is the gap $|\rLd(\vtheta, \vlambda) - \rLg(\vtheta, \vlambda)|$ incurred when \emph{switching off} the stochasticity of the (hyper)policy; $(ii)$ is the error $\rLd(\vtheta^{*}_{\dagger,\omega}, \vlambda^{*}_{\dagger,\omega}) - \rLd(\vtheta^{*}_{\text{D},\omega}, \vlambda^{*}_{\text{D},\omega})$ arising when \emph{deploying} the parameters of the learned stochastic (hyper)policy.
We highlight that both error terms scale linearly with the stochasticity level $\sigma$ in the (hyper)policy, and with the regularity constants introduced in Assumptions~\ref{asm:Jd-reg-theta} (\textcolor{vibrantRed}{PB}) and~\ref{asm:Jd-reg-ns} (\textcolor{vibrantBlue}{AB}).
In addition, the losses depend on the $\ell_1$-norm of the Lagrange multipliers. In particular, the second bound depends on the $\ell_1$-norm of the Lagrange multiplier at the saddle point of the regularized stochastic Lagrangian $\cL_{\dagger,\omega}$. They also depend on the problem dimensionality, denoted $d_{\dagger}$, which corresponds to the parameter space dimensionality $\dt$ in the \textcolor{vibrantRed}{PB} case and to the action space dimensionality $\da$ in the \textcolor{vibrantBlue}{AB} case.
We further note that \textcolor{vibrantBlue}{AB} exploration embeds an additional dependence on the interaction horizon $T$ within the constant $\LCnsmax$.
Finally, this result allows us to recover the well-known trade-off between \textcolor{vibrantRed}{PB} and \textcolor{vibrantBlue}{AB} exploration strategies~\citep{metelli2018policy, montenegro2024learning}: the former may suffer in high-dimensional parameter spaces (large $\dt$), whereas the latter may struggle with high-dimensional action spaces or long interaction horizons (large $\da$ or large $T$).

The analysis presented here will play a central role for establishing the sample complexity of \cpg when learning an optimal feasible deterministic policy via stochastic (hyper)policies and subsequently deploying their deterministic counterpart by switching off the stochasticity at the end of training.

\subsection{Conditions for Convergence}
\label{subsec:conditions_convergence}

In order to establish the convergence guarantees of \cpg to the optimal feasible deterministic policy, achieved by switching off the stochasticity after learning an optimal stochastic (hyper)policy, we proceed as follows. First, we leverage Theorem~\ref{thr:convergencePot} to quantify the sample complexity required for learning an optimal stochastic (hyper)policy in the last iterate of \cpg. Then, we leverage Theorem~\ref{thr:L-deploy} to characterize the loss incurred in terms of potential function $\cP_{K}(\chi)$ when transitioning from a stochastic (hyper)policy to its deterministic counterpart by setting $\sigma = 0$.

To apply Theorem~\ref{thr:convergencePot}, it is necessary to verify the set of assumptions introduced in Section~\ref{sec:ass}. In this part, we revisit those assumptions and, when possible, we aim to minimize their number by showing that, under the noise model introduced in Section~\ref{subsec:white_noise} to represent both the \textcolor{vibrantBlue}{AB} and \textcolor{vibrantRed}{PB} exploration paradigms, most of the required conditions can be inherited from analogous regularity properties imposed on the underlying deterministic-policy-dependent quantities.

\paragraph{Saddle Point Existence.}
Assumption~\ref{asm:assunzione} enforces the existence of a saddle point $(\vtheta^{*}_{\dagger, \omega}, \vlambda^{*}_{\dagger, \omega})$ for the \emph{stochastic} Lagrangian $\cL_{\dagger,\omega}$. This assumption is only needed in Theorem~\ref{thr:conversion} to map the \emph{stochastic} (\ie associated to $\cL_{\dagger,\omega}$) potential function $\cP_{\dagger,K}(\chi)$ to the $\Jgen$ terms, demonstrating it is a useful tool in quantifying the sample complexity to ensure last-iterate global convergence of \cpg. Similarly, it is possible to show the same mapping between $\cP_{\text{D},K}(\chi)$ to the $\Jd$ terms as done in Theorem~\ref{thr:conversion}, by applying the same theorem straightforwardly under the existence of a saddle point $(\vtheta^{*}_{\text{D}, \omega}, \vlambda^{*}_{\text{D}, \omega})$ for the \emph{deterministic} Lagrangian $\cL_{\text{D},\omega}$.
Since the \emph{stochastic} saddle point $(\vtheta^{*}_{\dagger, \omega}, \vlambda^{*}_{\dagger, \omega})$ is not related to the \emph{deterministic} one $(\vtheta^{*}_{\text{D}, \omega}, \vlambda^{*}_{\text{D}, \omega})$, we need to assume the existence of both.

\paragraph{Weak $\psi$-Gradient Domination.}
Assumption~\ref{asm:wgd} is a core component of the theoretical analysis of \cpg, as it enables regularization solely \wrt the Lagrange multipliers~$\vlambda$, as further discussed in Section~\ref{sec:ass}. Rather than directly assuming weak $\psi$-gradient domination on the stochastic Lagrangian~$\cL_{\dagger,0}$, we demonstrate that this property can be inherited from the corresponding assumption on the deterministic Lagrangian~$\cL_{\text{D},0}$, for both exploration paradigms.

\begin{ass}[Weak $\psi$-Gradient Domination on $\cL_{\text{D},0}$] \label{asm:wgd-det}
    Let $\psi \in [1,2]$. There exist $\dalpha > 0$ and $\dbeta \ge 0$ such that, for every $\vtheta \in \Theta$ and $\vlambda \in \mR^{U}_{\ge 0}$, it holds that:
    \begin{align*}
        \left\| \nabla_{\vtheta} \cL_{\text{D},0}(\vtheta, \vlambda) \right\|_2^{\psi} \ge \dalpha \left( \cL_{\text{D},0}(\vtheta, \vlambda) - \min_{\vtheta' \in \Theta} \cL_{\text{D},0}(\vtheta', \vlambda) \right) - \dbeta.
    \end{align*}
\end{ass}
This leads to a characterization analogous to that of Assumption~\ref{asm:wgd}, but imposed on the deterministic Lagrangian~$\cL_{\text{D},0}$ rather than directly on the stochastic one~$\cL_{\dagger,0}$, that just inherits this property, as we later show. Before establishing this inheritance result, we introduce an additional assumption that is required only in the case of \textcolor{vibrantBlue}{AB} exploration.
\begin{ass}[Regularity of $\mu_{\vtheta}$] \label{asm:mu-reg}
    There exists $\LCmu \in \mR_{\ge 0}$ such that, for every $\vs \in \cS$ and $\vtheta,\vtheta' \in \Theta$, the following holds:
    \begin{align*}
        \left\| \mu_{\vtheta}(\vs) - \mu_{\vtheta'}(\vs) \right\|_{2} \le \LCmu \left\| \vtheta - \vtheta' \right\|_{2}.
    \end{align*}
\end{ass}
This assumption only requires that the deterministic policy $\mu_{\vtheta}$ is $\LCmu$-LC \wrt its parameters, which is needed to inherit regularity properties from the deterministic objectives $\Jd$ in the \textcolor{vibrantBlue}{AB} exploration paradigm~\citep{montenegro2024learning}. We are now ready to state the inheritance of the weak $\psi$-gradient domination.
\begin{thr}[Inherited Weak $\psi$-Gradient Domination on $\cL_{\dagger,0}$] \label{thr:wgd-inherited}
    Consider an (hyper)policy complying with Definitions~\ref{def:white_noise_policies} (\textcolor{vibrantBlue}{AB}) or~\ref{def:white_noise_hyperpolicies} (\textcolor{vibrantRed}{PB}). Under Assumptions~\ref{asm:Jd-reg-theta} (\textcolor{vibrantRed}{PB}) or~\ref{asm:Jd-reg-ns} (\textcolor{vibrantBlue}{AB}), \ref{asm:wgd-det}, and~\ref{asm:mu-reg} (\textcolor{vibrantBlue}{AB}), for any $\psi \in [1,2]$, $\vtheta \in \Theta$, and $\vlambda \in \mR^{U}_{\ge 0}$, the following holds:
    \begin{align*}
        \left\| \nabla_{\vtheta} \cL_{\dagger,0}(\vtheta,\vlambda) \right\|_{2}^{\psi} \ge \dalpha \left( \cL_{\dagger,0}(\vtheta,\vlambda) - \min_{\vtheta' \in \Theta}\cL_{\dagger,0}(\vtheta',\vlambda) \right) - \beta_{\dagger}(\sigma,\psi),
    \end{align*}
    where: 
    \begin{align*}
        \pbeta(\sigma,\psi) &\coloneqq \dbeta + \left( 2 \dalpha \LCJdmax + (1 + \| \vlambda \|_{1})^{\psi-1} \LSJdmax^{\psi} \sigma^{\psi-1} \dt^{\psi/2 - 1} \right) (1 + \| \vlambda \|_{1}) \sigma \sqrt{\dt},
    \end{align*}
    and: 
    \begin{align*}
        \abeta(\sigma,\psi) &\coloneqq \dbeta + \left( 2 \dalpha \LCnsmax + (1 + \| \vlambda \|_{1})^{\psi-1} \LCmu^{\psi} \LSJdmax^{\psi} \sigma^{\psi-1} T^{\psi/2} \dt^{\psi/2 - 1} \right) (1 + \| \vlambda \|_{1}) \sigma \sqrt{\da}.
    \end{align*}
\end{thr}
\begin{proof}
    We begin the proof by considering the term $\cL_{\text{D},0}(\vtheta,\vlambda) - \min_{\vtheta' \in \Theta} \cL_{\text{D},0}(\vtheta',\vlambda)$. Given that we consider an (hyper)policy complying with Definitions~\ref{def:white_noise_policies} (\textcolor{vibrantBlue}{AB}) or~\ref{def:white_noise_hyperpolicies} (\textcolor{vibrantRed}{PB}), and being under Assumptions~\ref{asm:Jd-reg-theta} (\textcolor{vibrantRed}{PB}) and~\ref{asm:Jd-reg-ns} (\textcolor{vibrantBlue}{AB}), we can apply Theorem~\ref{thr:L-deploy}, stating that:
    \begin{align*}
        \cL_{\text{D},0}(\vtheta,\vlambda) - \cL_{\dagger,0}(\vtheta,\vlambda) \ge - (1 +\| \vlambda \|_{1}) L_{1\dagger} \sigma \sqrt{d_{\dagger}},
    \end{align*}
    where for $\dagger = \text{\textcolor{vibrantBlue}{A}}$ we have $L_{1\text{\textcolor{vibrantBlue}{A}}} = \LCnsmax$ and $d_{\text{\textcolor{vibrantBlue}{A}}} = \da$, and for $\dagger = \text{\textcolor{vibrantRed}{P}}$ we have $L_{1\text{\textcolor{vibrantRed}{P}}} = \LCJdmax$ and $d_{\text{\textcolor{vibrantRed}{P}}} = \dt$.
    That being said, the following derivation holds:
    \begin{align*}
        &\cL_{\text{D},0}(\vtheta,\vlambda) - \min_{\vtheta' \in \Theta} \cL_{\text{D},0}(\vtheta',\vlambda) \\
        &= \cL_{\text{D},0}(\vtheta,\vlambda) - \min_{\vtheta' \in \Theta} \cL_{\text{D},0}(\vtheta',\vlambda) \pm \cL_{\dagger, 0}(\vtheta,\vlambda) \\
        &\ge - (1 + \| \vlambda \|_{1}) L_{1\dagger} \sigma \sqrt{d_{\dagger}} + \cL_{\dagger, 0}(\vtheta,\vlambda) -  \min_{\vtheta' \in \Theta} \cL_{\text{D},0}(\vtheta',\vlambda) \pm  \min_{\vtheta' \in \Theta} \cL_{\dagger,0}(\vtheta',\vlambda) \\
        &\ge \cL_{\dagger, 0}(\vtheta,\vlambda) - \min_{\vtheta' \in \Theta} \cL_{\dagger,0}(\vtheta',\vlambda) - (1 +\| \vlambda \|_{1}) L_{1\dagger} \sigma \sqrt{d_{\dagger}} + \min_{\vtheta \in \Theta}\left( \cL_{\dagger,0}(\vtheta',\vlambda) -  \cL_{\text{D},0}(\vtheta',\vlambda) \right) \\
        &\ge \cL_{\dagger, 0}(\vtheta,\vlambda) - \min_{\vtheta' \in \Theta} \cL_{\dagger,0}(\vtheta',\vlambda) - 2 (1 + \| \vlambda \|_{1}) L_{1\dagger} \sigma \sqrt{d_{\dagger}},
    \end{align*}
    where we applied Theorem~\ref{thr:L-deploy} twice.

    Thus, starting from the statement of Assumption~\ref{asm:wgd-det}, the following holds:
    \begin{align}
        \left\| \nabla_{\vtheta} \cL_{\text{D},0}(\vtheta, \vlambda) \right\|_2^{\psi} &\ge \dalpha \left( \cL_{\text{D},0}(\vtheta, \vlambda) - \min_{\vtheta' \in \Theta} \cL_{\text{D},0}(\vtheta', \vlambda) \right) - \dbeta \nonumber \\
        &\ge \dalpha \left( \cL_{\dagger, 0}(\vtheta,\vlambda) - \min_{\vtheta' \in \Theta} \cL_{\dagger,0}(\vtheta',\vlambda)  \right) - \dbeta - 2 \dalpha (1 + \| \vlambda \|_{1}) L_{1\dagger} \sigma \sqrt{d_{\dagger}}. \label{eq:wgd-inherited-1}
    \end{align}

    To continue the proof, we need to consider separately the two exploration paradigms, in order to properly exploit the corresponding noise model.

    \textbf{\textcolor{vibrantRed}{PB} Exploration.}~~According to Definition~\ref{def:white_noise_hyperpolicies}, we can rewrite the \textcolor{vibrantRed}{PB} Lagrangian as:
    \begin{align*}
        \textcolor{vibrantRed}{\cL_{\text{P},0}}(\vtheta,\vlambda) = \E_{\vepsilon \sim \Phi_{\dt}}\left[ \cL_{\text{D},0}(\vtheta+\vepsilon,\vlambda) \right].
    \end{align*}
    Thus, given $\alpha \in [0,1]$ and defining $\widetilde{\vtheta}_{\vepsilon} \coloneqq \alpha \vtheta + (1-\alpha)(\vtheta+\vepsilon)$, the following holds:
    \begin{align*}
        \nabla_{\vtheta} \textcolor{vibrantRed}{\cL_{\text{P},0}}(\vtheta,\vlambda) = \E_{\vepsilon \sim \Phi_{\dt}}\left[ \nabla_{\vtheta} \cL_{\text{D},0}(\vtheta+\vepsilon,\vlambda) \right] = \nabla_{\vtheta} \cL_{\text{D},0}(\vtheta,\vlambda) + \E_{\vepsilon \sim \Phi_{\dt}}\left[ \vepsilon^{\top} \nabla_{\vtheta}^{2} \cL_{\text{D},0}(\widetilde{\vtheta}_{\vepsilon},\vlambda) \right],
    \end{align*}
    where we have simply applied the Taylor expansion centered in $\vepsilon=\bm{0}_{\dt}$.
    Now, by applying the Euclidean norm, we have the following:
    \begin{align*}
       \left\|  \nabla_{\vtheta} \textcolor{vibrantRed}{\cL_{\text{P},0}}(\vtheta,\vlambda) \right\|_{2} &= \left\| \nabla_{\vtheta} \cL_{\text{D},0}(\vtheta,\vlambda) + \E_{\vepsilon \sim \Phi_{\dt}}\left[ \vepsilon^{\top} \nabla_{\vtheta}^{2} \cL_{\text{D},0}(\widetilde{\vtheta}_{\vepsilon},\vlambda) \right] \right\|_{2} \\
       &\ge \left\| \nabla_{\vtheta} \cL_{\text{D},0}(\vtheta,\vlambda) \right\|_{2} - \left\| \E_{\vepsilon \sim \Phi_{\dt}}\left[ \vepsilon^{\top} \nabla_{\vtheta}^{2} \cL_{\text{D},0}(\widetilde{\vtheta}_{\vepsilon},\vlambda) \right] \right\|_{2} \\
       &\ge \left\| \nabla_{\vtheta} \cL_{\text{D},0}(\vtheta,\vlambda) \right\|_{2} -  \E_{\vepsilon \sim \Phi_{\dt}}\left[ \left\| \vepsilon^{\top} \nabla_{\vtheta}^{2} \cL_{\text{D},0}(\widetilde{\vtheta}_{\vepsilon},\vlambda) \right\|_{2} \right],
    \end{align*}
    which follows by applying the triangular and Jensen's inequalities.
    Now, by applying the Cauchy-Schwartz inequality as $\E_{\vepsilon \sim \Phi_{\dt}}[\| \vepsilon \|_{2}] \le \sqrt{\E_{\vepsilon \sim \Phi_{\dt}}[\| \vepsilon \|_{2}^{2}]} \le \sigma \sqrt{\dt}$ and by exploiting the fact that, under Assumption~\ref{asm:Jd-reg-theta}:
    \begin{align*}
        \left\| \nabla_{\vtheta}^{2} \cL_{\text{D},0}(\vtheta,\vlambda) \right\|_{2} \le \left\| \nabla_{\vtheta}^{2} \Jdzero(\vtheta) \right\|_{2} + \sum_{i=1}^{U} \lambda_{i} \left\| \nabla_{\vtheta}^{2} \Jd(\vtheta) \right\|_{2} \le \LSJdmax \left( 1 + \sum_{i=1}^{U}\lambda_{i} \right) = \left( 1 + \left\| \vlambda \right\|_{1} \right) \LSJdmax,
    \end{align*}
    we can conclude the following:
    \begin{align*}
         \left\| \nabla_{\vtheta} \cL_{\text{D},0}(\vtheta,\vlambda) \right\|_{2} &\le  \left\|  \nabla_{\vtheta} \textcolor{vibrantRed}{\cL_{\text{P},0}}(\vtheta,\vlambda) \right\|_{2} + \E_{\vepsilon \sim \Phi_{\dt}}\left[ \left\| \vepsilon^{\top} \nabla_{\vtheta}^{2} \cL_{\text{D},0}(\widetilde{\vtheta}_{\vepsilon},\vlambda) \right\|_{2} \right] \\
         &\le \left\|  \nabla_{\vtheta} \textcolor{vibrantRed}{\cL_{\text{P},0}}(\vtheta,\vlambda) \right\|_{2} + \left( 1 + \left\| \vlambda \right\|_{1} \right) \LSJdmax \sigma \sqrt{\dt}.
    \end{align*}

    Now, considering that $\psi \in [1,2]$, by exploiting the superadditivity of $( \cdot )^{\psi}$, we have:
    \begin{align*}
         \left\| \nabla_{\vtheta} \cL_{\text{D},0}(\vtheta,\vlambda) \right\|_{2}^{\psi} \le \left\|  \nabla_{\vtheta} \textcolor{vibrantRed}{\cL_{\text{P},0}}(\vtheta,\vlambda) \right\|_{2}^{\psi} + \left(\left( 1 + \left\| \vlambda \right\|_{1} \right) \LSJdmax \sigma \sqrt{\dt}\right)^{\psi}.
    \end{align*}

    Combining this last result with Equation~\eqref{eq:wgd-inherited-1}, we conclude the inheritance of the weak $\psi$-GD from $\cL_{\text{D},\omega}$ in the \textcolor{vibrantRed}{PB} case:
    \begin{align*}
        \left\|  \nabla_{\vtheta} \textcolor{vibrantRed}{\cL_{\text{P},0}}(\vtheta,\vlambda) \right\|_{2}^{\psi} \ge \dalpha \left( \textcolor{vibrantRed}{\cL_{\text{P},0}}(\vtheta,\vlambda) - \min_{\vtheta' \in \Theta} \textcolor{vibrantRed}{\cL_{\text{P},0}}(\vtheta',\vlambda)  \right) - \pbeta(\sigma,\psi),
    \end{align*}
    where: 
    \begin{align*}
        \pbeta(\sigma,\psi) &\coloneqq \dbeta + \left( 2 \dalpha \LCJdmax + (1 + \| \vlambda \|_{1})^{\psi-1} \LSJdmax^{\psi} \sigma^{\psi-1} \dt^{\psi/2 - 1} \right) (1 + \| \vlambda \|_{1}) \sigma \sqrt{\dt}.
    \end{align*}

    \textbf{\textcolor{vibrantBlue}{AB} Exploration.}~~According to Definition~\ref{def:white_noise_policies}, we can rewrite the \textcolor{vibrantBlue}{AB} Lagrangian as:
    \begin{align*}
        \textcolor{vibrantBlue}{\cL_{\text{A},0}}(\vtheta,\vlambda) = \E_{\uvepsilon \sim \Phi_{\da}^{T}}\left[ \cL_{\text{D},0}(\uvmu_{\vtheta} + \uvepsilon,\vlambda) \right].
    \end{align*}
    Thus, given $\alpha \in [0,1]$ and defining $\widetilde{\uvmu}_{\vtheta} \coloneqq \alpha \uvmu_{\vtheta} + (1-\alpha)(\uvmu_{\vtheta} + \uvepsilon)$, the following holds:
    \begin{align*}
        \nabla_{\vtheta} \textcolor{vibrantBlue}{\cL_{\text{A},0}}(\vtheta,\vlambda) &= \E_{\uvepsilon \sim \Phi_{\da}^{T}}\left[ \nabla_{\vtheta} \cL_{\text{D},0}(\uvmu_{\vtheta} + \uvepsilon,\vlambda) \right] \\
        &= \E_{\uvepsilon \sim \Phi_{\da}^{T}}\left[ \nabla_{\uvmu} \cL_{\text{D},0}(\uvmu,\vlambda) \rvert_{\uvmu = \uvmu_{\vtheta} + \uvepsilon} \nabla_{\vtheta} (\uvmu_{\vtheta} + \uvepsilon) \right] \\
        &= \E_{\uvepsilon \sim \Phi_{\da}^{T}}\left[ \nabla_{\uvmu} \cL_{\text{D},0}(\uvmu,\vlambda) \rvert_{\uvmu = \uvmu_{\vtheta}} \nabla_{\vtheta} \uvmu_{\vtheta} + \uvepsilon^{\top} \nabla_{\uvmu}^{2} \cL_{\text{D},0}(\uvmu,\vlambda) \rvert_{\uvmu = \widetilde{\uvmu}_{\vtheta}} \nabla_{\vtheta} \uvmu_{\vtheta} \right] \\
        &= \nabla_{\vtheta} \cL_{\text{D},0}(\vtheta,\vlambda) + \E_{\uvepsilon \sim \Phi_{\da}^{T}}\left[ \uvepsilon^{\top} \nabla_{\uvmu}^{2} \cL_{\text{D},0}(\uvmu,\vlambda) \rvert_{\uvmu = \widetilde{\uvmu}_{\vtheta}} \nabla_{\vtheta} \uvmu_{\vtheta} \right].
    \end{align*}
    where we simply applied the chain rule and the Taylor expansion centered in $\uvepsilon=\bm{0}_{\da T}$. Now, by applying the Euclidean norm, we have:
    \begin{align*}
        \left\| \nabla_{\vtheta} \textcolor{vibrantBlue}{\cL_{\text{A},0}}(\vtheta,\vlambda) \right\|_{2} &= \left\| \nabla_{\vtheta} \cL_{\text{D},0}(\vtheta,\vlambda) + \E_{\uvepsilon \sim \Phi_{\da}^{T}}\left[ \uvepsilon^{\top} \nabla_{\uvmu}^{2} \cL_{\text{D},0}(\uvmu,\vlambda) \rvert_{\uvmu = \widetilde{\uvmu}_{\vtheta}} \nabla_{\vtheta} \uvmu_{\vtheta} \right] \right\|_{2} \\
        &\ge \left\| \nabla_{\vtheta} \cL_{\text{D},0}(\vtheta,\vlambda) \right\|_{2} -  \left\| \E_{\uvepsilon \sim \Phi_{\da}^{T}}\left[ \uvepsilon^{\top} \nabla_{\uvmu}^{2} \cL_{\text{D},0}(\uvmu,\vlambda) \rvert_{\uvmu = \widetilde{\uvmu}_{\vtheta}} \nabla_{\vtheta} \uvmu_{\vtheta} \right] \right\|_{2} \\
        &\ge \left\| \nabla_{\vtheta} \cL_{\text{D},0}(\vtheta,\vlambda) \right\|_{2} -  \E_{\uvepsilon \sim \Phi_{\da}^{T}}\left[ \left\| \uvepsilon^{\top} \nabla_{\uvmu}^{2} \cL_{\text{D},0}(\uvmu,\vlambda) \rvert_{\uvmu = \widetilde{\uvmu}_{\vtheta}} \nabla_{\vtheta} \uvmu_{\vtheta}  \right\|_{2} \right],
    \end{align*}
    which follows from just applying the triangular and Jensen's inequalities. Now, by applying the Cauchy-Schwartz inequality as $\E_{\uvepsilon \sim \Phi_{\da}^{T}} [\| \uvepsilon \|_{2}] \le \sqrt{\E_{\uvepsilon \sim \Phi_{\da}^{T}} [\| \uvepsilon \|_{2}^{2}]} \le \sigma \sqrt{T\da}$ and exploiting Assumptions~\ref{asm:Jd-reg-ns} and~\ref{asm:mu-reg}, by following the same procedure of the \textcolor{vibrantRed}{PB} case, we obtain:
    \begin{align*}
        \left\| \nabla_{\vtheta} \textcolor{vibrantBlue}{\cL_{\text{A},0}}(\vtheta,\vlambda) \right\|_{2} \ge \left\| \nabla_{\vtheta} \cL_{\text{D},0}(\vtheta,\vlambda) \right\|_{2} - (1 + \| \vlambda \|_{1}) \LCmu \LSnsmax \sigma \sqrt{T \da},
    \end{align*}
    and thus:
    \begin{align*}
        \left\|  \nabla_{\vtheta} \textcolor{vibrantBlue}{\cL_{\text{A},0}}(\vtheta,\vlambda) \right\|_{2}^{\psi} \ge \dalpha \left( \textcolor{vibrantBlue}{\cL_{\text{A},0}}(\vtheta,\vlambda) - \min_{\vtheta' \in \Theta} \textcolor{vibrantBlue}{\cL_{\text{A},0}}(\vtheta',\vlambda)  \right) - \abeta(\sigma,\psi),
    \end{align*}
    where 
    \begin{align*}
        \abeta(\sigma,\psi) &\coloneqq \dbeta + \left( 2 \dalpha \LCnsmax + (1 + \| \vlambda \|_{1})^{\psi-1} \LCmu^{\psi} \LSJdmax^{\psi} \sigma^{\psi-1} T^{\psi/2} \dt^{\psi/2 - 1} \right) (1 + \| \vlambda \|_{1}) \sigma \sqrt{\da},
    \end{align*}
    showing the inheritance of the weak $\psi$-GD from $\cL_{\text{D},\omega}$ in the \textcolor{vibrantBlue}{AB} case too. 
\end{proof}
We highlight that the weak $\psi$-gradient domination property is inherited from the deterministic Lagrangian~$\cL_{\text{D},0}$ with the \emph{same} multiplicative constant~$\dalpha$. Moreover, by setting $\sigma = 0$, one exactly recovers the result stated in Assumption~\ref{asm:wgd-det}. Finally, under the adopted noise model, this property follows directly from regularity assumptions on the deterministic cost functions~$\Jd$ (Assumptions~\ref{asm:Jd-reg-theta} and~\ref{asm:Jd-reg-ns}) and on the deterministic policy~$\mu_{\vtheta}$ (Assumption~\ref{asm:mu-reg}), together with the bounds on the loss incurred when switching off the noise (see Theorem~\ref{thr:L-deploy}).

\paragraph{Regularity of $\cL_{\dagger,0}$.}
Assumption~\ref{asm:L_grad_lip} imposes regularity conditions on the stochastic Lagrangian~$\cL_{\dagger,0}$, which are standard in the literature on the convergence of primal-dual methods~\citep{yang2020minimax}, as previously discussed in Section~\ref{sec:alg_general}. As with the other conditions required to apply Theorem~\ref{thr:convergencePot}, these regularity properties can also be inherited from Assumption~\ref{asm:Jd-reg-theta}, as formalized in the following theorem.
\begin{thr}[Inherited Regularity of $\cL_{\dagger,0}$] \label{thr:inherited-L-reg}
    Consider a (hyper)policy complying with Definitions~\ref{def:white_noise_policies} (\AB) or~\ref{def:white_noise_hyperpolicies} (\PB). Under Assumption~\ref{asm:Jd-reg-theta}, for every $\vlambda,\vlambda' \in \mR^{U}_{\ge 0}$ and $\vtheta,\vtheta' \in \Theta$ the following conditions hold:
    \begin{align*}
        &\left\| \nabla_{\vlambda} \cL_{\dagger,0} (\vtheta,\vlambda) - \nabla_{\vlambda} \cL_{\dagger,0} (\vtheta',\vlambda) \right\|_{2} \le \sqrt{U} \LCJdmax \left\| \vtheta - \vtheta' \right\|_{2}, \\
        &\left\| \nabla_{\vtheta} \cL_{\dagger,0}(\vtheta,\vlambda) - \nabla_{\vtheta} \cL_{\dagger,0}(\vtheta',\vlambda) \right\|_{2} \le \left( 1 + \| \vlambda \|_{1} \right) \LSJdmax \left\| \vtheta - \vtheta' \right\|_{2}, \\
        &\left\|  \nabla_{\vtheta}\cL_{\dagger,0}(\vtheta,\vlambda) - \nabla_{\vtheta}\cL_{\dagger,0}(\vtheta,\vlambda')\right\|_{2} \le \LCJdmax \left\| \vlambda-\vlambda' \right\|_{2}.
    \end{align*}
\end{thr}
\begin{proof}
    We start by proving that $\cL_{\dagger,0}(\cdot,\vlambda)$ is LS \wrt the parameters $\vtheta$, thus, for every $\vlambda \in \mR^{U}_{\ge 0}$ and $\vtheta,\vtheta' \in \Theta$, we aim to find $L_{2} \in \mR_{\ge0}$ such that:
    \begin{align*}
        \left\| \nabla_{\vtheta} \cL_{\dagger,0}(\vtheta,\vlambda) - \nabla_{\vtheta} \cL_{\dagger,0}(\vtheta',\vlambda) \right\|_{2} \le L_{2} \left\| \vtheta - \vtheta' \right\|_{2}.
    \end{align*}
    This can be easily done by expanding the $\nabla_{\vtheta} \cL_{\dagger,0}$ terms:
    \begin{align*}
        \left\| \nabla_{\vtheta} \cL_{\dagger,0}(\vtheta,\vlambda) - \nabla_{\vtheta} \cL_{\dagger,0}(\vtheta',\vlambda) \right\|_{2} &= \left\| \nabla_{\vtheta} \Jgenzero(\vtheta) - \nabla_{\vtheta} \Jgenzero(\vtheta') + \sum_{i=1}^{U} \lambda_{i} \left( \nabla_{\vtheta}\Jgen(\vtheta) - \nabla_{\vtheta}\Jgen(\vtheta') \right) \right\|_{2} \\
        &\le \left\| \nabla_{\vtheta} \Jgenzero(\vtheta) - \nabla_{\vtheta} \Jgenzero(\vtheta') \right\|_{2} + \sum_{i=1}^{U} \lambda_{i} \left\| \nabla_{\vtheta}\Jgen(\vtheta) - \nabla_{\vtheta} \Jgen(\vtheta') \right\|_{2},
    \end{align*}
    by the triangular inequality. Now, we recover Lemmas~D.3 and~D.7 by \citet{montenegro2024learning}, stating that under Assumption~\ref{asm:Jd-reg-theta} both $\Jp$ and $\Ja$ are LS with the same constant of $\Jd$. In our setting, it means that $\Jgen$ is $\LSJdmax$-LS, for every $i \in \dsb{0,U}$. That being said, the following holds:
    \begin{align*}
        \left\| \nabla_{\vtheta} \cL_{\dagger,0}(\vtheta,\vlambda) - \nabla_{\vtheta} \cL_{\dagger,0}(\vtheta',\vlambda) \right\|_{2} &\le \left\| \nabla_{\vtheta} \Jgenzero(\vtheta) - \nabla_{\vtheta} \Jgenzero(\vtheta') \right\|_{2} + \sum_{i=1}^{U} \lambda_{i} \left\| \nabla_{\vtheta}\Jgen(\vtheta) - \nabla_{\vtheta} \Jgen(\vtheta') \right\|_{2} \\
        &\le \left( 1 + \sum_{i=1}^{U} \lambda_{i} \right) \LSJdmax \left\| \vtheta - \vtheta' \right\|_{2} \\
        &= \left( 1 + \| \vlambda \|_{1} \right) \LSJdmax \left\| \vtheta - \vtheta' \right\|_{2},
    \end{align*}
    thus having quantified the smoothness constant.

    We can now proceed by proving that $\nabla_{\vlambda} \cL_{\dagger,0} (\cdot,\vlambda)$ is LC, so we have to find a constant $L_{1} \in \mR_{\ge 0}$ such that, for every $\vlambda \in \mR^{U}_{\ge 0}$ and $\vtheta,\vtheta' \in \Theta$:
    \begin{align*}
        \left\| \nabla_{\vlambda} \cL_{\dagger,0} (\vtheta,\vlambda) - \nabla_{\vlambda} \cL_{\dagger,0} (\vtheta',\vlambda) \right\|_{2} \le L_{1} \left\| \vtheta - \vtheta' \right\|_{2}.
    \end{align*}
    As done in the previous case, we expand the $\nabla_{\vlambda} \cL_{\dagger,0}$ terms:
    \begin{align*}
        \left\| \nabla_{\vlambda} \cL_{\dagger,0} (\vtheta,\vlambda) - \nabla_{\vlambda} \cL_{\dagger,0} (\vtheta',\vlambda) \right\|_{2} &= \left\| \mathbf{J}_{\dagger}(\vtheta) - \mathbf{J}_{\dagger}(\vtheta') \right\|_{2} = \sqrt{\sum_{i=0}^{U} \left( \Jgen(\vtheta) - \Jgen(\vtheta') \right)^{2}} \le \sqrt{U} \LCJdmax \left\| \vtheta - \vtheta' \right\|_{2},
    \end{align*}
    where we just applied the same reasoning of Lemmas~D.3 and~D.7 by \citet{montenegro2024learning} to state that $\Jgen$ is $\LCJdmax$-LS, for every $i \in \dsb{0,U}$.

    Finally, we prove that $\nabla_{\vtheta}\cL_{\dagger,0}(\vtheta,\cdot)$ is LC, finding a constant $L_{3} \in \mR_{\ge 0}$ such that, for every $\vlambda,\vlambda' \in \mR^{U}_{\ge 0}$:
    \begin{align*}
        \left\|  \nabla_{\vtheta}\cL_{\dagger,0}(\vtheta,\vlambda) - \nabla_{\vtheta}\cL_{\dagger,0}(\vtheta,\vlambda')\right\|_{2} \le L_{3} \left\| \vlambda - \vlambda' \right\|_{2}.
    \end{align*}
    As done before, we expand the $\nabla_{\vtheta}\cL_{\dagger,0}$ terms:
    \begin{align*}
         \left\|  \nabla_{\vtheta}\cL_{\dagger,0}(\vtheta,\vlambda) - \nabla_{\vtheta}\cL_{\dagger,0}(\vtheta,\vlambda')\right\|_{2} = \left\| \nabla_{\vtheta} \mathbf{J}_{\dagger}(\vtheta) (\vlambda-\vlambda')\right\|_{2} \le \LCJdmax \left\| \vlambda-\vlambda' \right\|_{2},
    \end{align*}
    where we just exploited the fact that under Assumption~\ref{asm:Jd-reg-theta} $\Jd$ is $\LCJdmax$-LC for every $i \in \dsb{0,U}$.
\end{proof}
We highlight that, as previously discussed in Section~\ref{sec:alg_general}, the only constant that depends on $\cO(\omega^{-1})$ when considering the learning process of \cpg is the smoothness constant of $\cL_{\dagger,0}(\cdot,\vlambda)$. This dependence arises from the term $\| \vlambda \|_{1}$, which is upper bounded by $\cO(\omega^{-1})$ due to the regularization employed in \cpg.

\paragraph{Bounded Estimators' Variances.}
The last condition for convergence we have to discuss is the one of Assumption~\ref{ass:boundedVariance}, requiring that the unbiased estimators $\widehat{\nabla}_{\vtheta} \cL_{\dagger,\omega}$ and $\widehat{\nabla}_{\vlambda} \cL_{\dagger,\omega}$ have bounded variance. Under the specific noise model at hand, this property holds under the following assumption.
\begin{ass}[Bounded Scores of $\Phi$] \label{asm:bounded-noise-scores}
    Let $\Phi \in \Delta(\mR^{d})$ be a white noise complying with Definition~\ref{def:white_noise} with variance bound $\sigma \in \mR_{> 0}$ and density $\phi$. $\phi$ is differentiable in its argument and there exists universal constant $c \in \mR_{> 0}$ such that:
    \begin{align*}
        \E_{\vepsilon \sim \Phi} \left[\| \nabla_{\vepsilon}\log \phi(\vepsilon)\|_{2}^{2} \right] \leq c d \sigma^{-2} \quad \text{and} \quad \E_{\vepsilon \sim \Phi} \left[\| \nabla^{2}_{\vepsilon}\log \phi(\vepsilon)\|_{2} \right] \leq c \sigma^{-2}.
    \end{align*}
\end{ass}
Intuitively, this assumption is equivalent to the more common ones requiring the boundedness of the expected norms of the score function and its gradient~\citep{papini2022smoothing,yuan2022general}.
Note that a zero-mean Gaussian noise  $\Phi = \mathcal{N}(\mathbf{0}_d, \bm{\Sigma})$ fulfills Assumption~\ref{asm:bounded-noise-scores}. Indeed, one has $\nabla_{\vepsilon} \log \phi(\vepsilon) = \bm{\Sigma}^{-1}\vepsilon$ and $\nabla_{\vepsilon}^2 \log \phi(\vepsilon) = \bm{\Sigma}^{-1}$. Thus, $\E[\| \nabla_{\vepsilon} \log \phi(\vepsilon)\|_2^2]  = \text{tr}(\bm{\Sigma}^{-1}) \le d \lambda_{\min}(\bm\Sigma)^{-1}$ and $\E[\| \nabla_{\vepsilon}^2 \log \phi(\vepsilon)\|_2] = \lambda_{\min}(\bm{\Sigma})^{-1}$. In particular, for an isotropic Gaussian $\bm{\Sigma}=\sigma^2 \mathbf{I}$, we have $\lambda_{\min}(\bm{\Sigma}) = \sigma^2$, fulfilling Assumption~\ref{asm:bounded-noise-scores} with $c=1$.

Under Assumption~\ref{asm:bounded-noise-scores}, the variances of the estimators employed in \cpg, which are described in Section~\ref{subsec:alg-specific}, are bounded, as shown by the following lemma.
\begin{lemma}[Bounded Estimators' Variances] \label{lem:bounded-estim-var}
    Consider a (hyper)policy complying with Definitions~\ref{def:white_noise_policies} (\AB) and~\ref{def:white_noise_hyperpolicies} (\PB). Under Assumptions~\ref{asm:mu-reg} (just for \AB) and~\ref{asm:bounded-noise-scores}, the following conditions hold:
    \begin{align*}
        \Var\left[ \widehat{\nabla}_{\vlambda} \cL_{\dagger,\omega}(\vtheta,\vlambda) \right] \le \frac{U (1-\gamma^{T})^{2}}{N (1-\gamma)^{2}} \eqqcolon \vargenlambda \quad \text{and} \quad \Var\left[ \widehat{\nabla}_{\vtheta} \cL_{\dagger,\omega}(\vtheta,\vlambda) \right] \le \frac{\zgentheta (1+\| \vlambda \|_{1})^2}{N \sigma^{2}} \eqqcolon \vargentheta,
    \end{align*}
    where $\zptheta \coloneqq c \dt \left(\frac{1-\gamma^{T}}{1-\gamma}\right)^{2}$ and $\zatheta \coloneqq c \da \LCmu^{2} \left(\frac{1-\gamma^{T}}{1-\gamma}\right)^{3}$.
\end{lemma}
\begin{proof}
    We start by bounding the variance of $\widehat{\nabla}_{\vlambda} \cL_{\dagger,\omega}(\vtheta,\vlambda) = \widehat{\mathbf{J}}_{\dagger}(\vtheta) - \mathbf{b} - \omega \vlambda$, where $\widehat{\mathbf{J}}_{\dagger}(\vtheta) = (\widehat{J}_{\dagger.0}(\vtheta), \ldots, \widehat{J}_{\dagger.U}(\vtheta))^{\top}$ and $\widehat{J}_{\dagger.i}(\vtheta) = \frac{1}{N} \sum_{j=0}^{N-1} C_{i}(\tau_{j})$ with $\tau_{j} \sim \pA(\cdot, \vtheta)$ (in \AB exploration) or $\tau_{j} \sim \pD(\cdot, \vtheta_{j})$ and $\vtheta_{j} \sim \nu_{\vtheta}$ (in \PB exploration).
    Thus, we can notice that the variance arises just from $\mathbf{J}_{\dagger}(\vtheta)$. Defining $\mathbf{C}(\tau) \coloneqq (C_{1}(\tau), \ldots, C_{U}(\tau))^{\top}$, where $C_{i}(\tau) = \sum_{t=0}^{T-1}\gamma^{t} c_{i}(\vs_{\tau,t},\va_{\tau,t}) \le \frac{1-\gamma^{T}}{1-\gamma}$, the following holds:
    \begin{align*}
        \Var\left[ \widehat{\nabla}_{\vlambda} \cL_{\dagger,\omega}(\vtheta,\vlambda) \right] = \frac{1}{N} \Var\left[ \mathbf{C}(\tau_{1})\right] = \frac{1}{N} \E \left[ \left\|\mathbf{C}(\tau_{1}) \right\|_{2}^{2} \right] \le \frac{U (1-\gamma^{T})^{2}}{N (1-\gamma)^{2}}.
    \end{align*}

    We now bound the variance of $\widehat{\nabla}\cL_{\dagger,\omega}(\vtheta,\vlambda)$, for which we distinguish the \PB and the \AB cases. Starting from \PB exploration, we can express the estimator at hand as:
    \begin{align*}
        \left\| \widehat{\nabla}_{\vtheta} \rLp(\vtheta,\vlambda) \right\|_{2} &\le \frac{1}{N} \sum_{j=0}^{N-1} \left\| \nabla_{\vtheta} \log \nu_{\vtheta}(\vtheta_{j}) \right\|_{2} \left\| C_{0}(\tau_{j}) + \sum_{i=1}^{U} \lambda_{i} C_{i}(\tau_{j}) \right\|_{2} \\
        &\le \frac{(1+\| \vlambda \|_{1}) (1-\gamma^{T})}{N(1-\gamma)} \sum_{j=0}^{N-1} \left\| \nabla_{\vtheta} \log \nu_{\vtheta}(\vtheta_{j}) \right\|_{2}.
    \end{align*}
    Thus, we have the following:
    \begin{align*}
        \Var\left[ \widehat{\nabla}_{\vtheta} \rLp(\vtheta,\vlambda) \right] &= \frac{1}{N} \Var \left[ \nabla_{\vtheta} \log \nu_{\vtheta}(\vtheta_{1}) \left( C_{0}(\tau_{1}) + \sum_{i=1}^{U} \lambda_{i} C_{i}(\tau_{1}) \right) \right] \\
        &= \frac{1}{N} \E \left[ \left\| \nabla_{\vtheta} \log \nu_{\vtheta}(\vtheta_{1}) \left( C_{0}(\tau_{1}) + \sum_{i=1}^{U} \lambda_{i} C_{i}(\tau_{1}) \right) \right\|_{2}^{2} \right] \\
        &\le \frac{(1+\| \vlambda \|_{1})^{2} (1-\gamma^{T})^{2}}{N(1-\gamma)^{2}} \E \left[ \left\| \nabla_{\vtheta} \log \nu_{\vtheta}(\vtheta_{1}) \right\|_{2}^{2} \right].
    \end{align*}
    We now recover Lemma~E.4 by~\citet{montenegro2024learning}, stating that under Assumption~\ref{asm:bounded-noise-scores} $\E[\| \nabla_{\vtheta} \log \nu_{\vtheta}(\vtheta') \|_{2}^{2}] \le c \dt \sigma^{-2}$ for every $\vtheta,\vtheta' \in \Theta$. Thus, we can conclude the following:
    \begin{align*}
        \Var\left[ \widehat{\nabla}_{\vtheta} \rLp(\vtheta,\vlambda) \right] \le
        \frac{c \dt (1+\| \vlambda \|_{1})^{2} (1-\gamma^{T})^{2}}{N \sigma^{2} (1-\gamma)^{2}}.
    \end{align*}

    Switching to the \AB case, we can express the estimator as:
    \begin{align*}
        \widehat{\nabla}_{\vtheta} \rLa (\vtheta,\vlambda) = \frac{1}{N} \sum_{j=0}^{N-1} \sum_{t=0}^{T-1} \left( \sum_{l=0}^{t} \nabla_{\vtheta} \log \pi_{\vtheta}(\va_{\tau_{j},l} \mid \vs_{\tau_{j},l}) \right) \gamma^{t} \left( c_{0}(\vs_{\tau_{j},t}, \va_{\tau_{j},t}) + \sum_{i=1}^{U} \lambda_{i} c_{i}(\vs_{\tau_{j},t}, \va_{\tau_{j},t})\right).
    \end{align*}
    Thus, we have the following:
    \begin{align*}
        &\Var \left[ \widehat{\nabla}_{\vtheta} \rLa (\vtheta,\vlambda) \right] \\
        &= \frac{1}{N} \Var \left[ \sum_{t=0}^{T-1} \left( \sum_{l=0}^{t} \nabla_{\vtheta} \log \pi_{\vtheta}(\va_{\tau_{1},l} \mid \vs_{\tau_{1},l}) \right) \gamma^{t} \left( c_{0}(\vs_{\tau_{1},t}, \va_{\tau_{1},t}) + \sum_{i=1}^{U} \lambda_{i} c_{i}(\vs_{\tau_{1},t}, \va_{\tau_{1},t})\right) \right] \\
        &= \frac{1}{N} \E \left[ \left\| \sum_{t=0}^{T-1} \left( \sum_{l=0}^{t} \nabla_{\vtheta} \log \pi_{\vtheta}(\va_{\tau_{1},l} \mid \vs_{\tau_{1},l}) \right) \gamma^{t} \left( c_{0}(\vs_{\tau_{1},t}, \va_{\tau_{1},t}) + \sum_{i=1}^{U} \lambda_{i} c_{i}(\vs_{\tau_{1},t}, \va_{\tau_{1},t})\right) \right\|_{2}^{2} \right] \\
        &\le \frac{1}{N} \E \left[ \sum_{t=0}^{T-1} \gamma^{t} \left( c_{0}(\vs_{\tau_{1},t}, \va_{\tau_{1},t}) + \sum_{i=1}^{U} \lambda_{i} c_{i}(\vs_{\tau_{1},t}, \va_{\tau_{1},t}) \right)^{2} \left( \sum_{t=0}^{T-1} \gamma^{t} \sum_{l=0}^{t} \left\| \nabla_{\vtheta} \log \pi_{\vtheta}(\va_{\tau_{1},l} \mid \vs_{\tau_{1},l}) \right\|_{2}^{2} \right)\right] \\
        &\le \frac{(1+\| \vlambda \|_{1})^{2} (1-\gamma^{T})}{N (1-\gamma)} \E \left[ \sum_{t=0}^{T-1} \gamma^{t} \sum_{l=0}^{t} \left\| \nabla_{\vtheta} \log \pi_{\vtheta}(\va_{\tau_{1},l} \mid \vs_{\tau_{1},l}) \right\|_{2}^{2} \right].
    \end{align*}
    We now recover the result of Lemma~E.3 by~\citet{montenegro2024constraints} stating that, under Assumptions~\ref{asm:mu-reg} and~\ref{asm:bounded-noise-scores}, it holds that $\E[\| \nabla_{\vtheta} \log \pi_{\vtheta}(\tau) \|_{2}^{2}] \le c \da \LCmu^{2} \sigma^{-2}$ for every $\vtheta \in \Theta$ and trajectory $\tau$.
    Thus, we conclude the following:
    \begin{align*}
        \Var \left[ \widehat{\nabla}_{\vtheta} \rLa (\vtheta,\vlambda) \right] \le \frac{c \da \LCmu^{2}(1+\| \vlambda \|_{1})^{2} (1-\gamma^{T})^{3}}{N \sigma^{2} (1-\gamma)^{3}}.
    \end{align*}
\end{proof}
As previously discussed in Section~\ref{sec:alg_general}, we highlight that while $\vargenlambda = \cO(1)$, $\vargentheta = \cO(\sigma^{-2} \omega^{-2})$ when considering the learning process of \cpg. The dependence of the latter from $\omega^{-2}$ arises from the regularization and projection $\Pi_{\Lambda}$, which ensures that the Lagrange multipliers are bounded by $\cO(\omega^{-1})$.

\subsection{Convergence Analysis} \label{subsec:convergence_deterministic}
We are now almost ready to present the convergence of \cpg to the optimal feasible \emph{deterministic} policy obtained by switching off the stochasticity (\ie setting $\sigma = 0$) at the end of the learning process. Recall that in Theorem~\ref{thr:convergencePot} we analyzed the convergence of the \emph{stochastic} (\ie associated with $\cL_{\dagger,\omega}$) potential function $\cP_{\dagger,k}(\chi)$, after establishing in Theorem~\ref{thr:conversion} that it serves as a tool to assess the last-iterate global convergence of \cpg.
Before proceeding, we formally introduce the \emph{deterministic} (\ie associated with $\cL_{\text{D},\omega}$) potential function. Considering $\omega > 0$, $\chi \in [0,1]$, and to be at a generic $k \in \mathbb{N}$ iterate of \cpg, $\cP_{\text{D},k}(\chi)$ i sdefined as:
\begin{align*}
    \cP_{\text{D},k}(\chi) \coloneqq \E\left[ \max_{\vlambda \in \mR^{U}_{\ge 0}} \cL_{\text{D},\omega}(\vtheta_{k}, \vlambda) - \cL_{\text{D},\omega}(\vtheta^*_{\text{D},\omega},\vlambda^*_{\text{D},\omega}) \right] + \chi \E\left[ \max_{\vlambda \in \mR^{U}_{\ge 0}} \cL_{\text{D},\omega}(\vtheta_{k}, \vlambda) - \cL_{\text{D},\omega}(\vtheta_{k},\vlambda_{k}) \right],
\end{align*}
where $(\vtheta^*_{\text{D},\omega},\vlambda^*_{\text{D},\omega})$ is the saddle point of $\cL_{\text{D},\omega}$.

In this part, we first quantify the loss $|\cP_{\dagger,k}(\chi) - \cP_{\text{D},k}(\chi)|$ between the stochastic and deterministic potential functions at a generic iterate $k \in \dsb{K}$ of \cpg. We then leverage this result to derive the sample complexity required to ensure that the deterministic policy deployed in the last iterate of \cpg is an optimal feasible solution for the problem at hand.

\begin{thr} \label{thr:pot-loss}
    Suppose to run \emph{\cpg} for $k$ iterations. Considering (hyper)policies complying with Definitions~\ref{def:white_noise_policies} (\textcolor{vibrantBlue}{AB}) or~\ref{def:white_noise_hyperpolicies} (\textcolor{vibrantRed}{PB}), under Assumptions~\ref{asm:Jd-reg-theta} (\textcolor{vibrantRed}{PB}) or~\ref{asm:Jd-reg-ns} (\textcolor{vibrantBlue}{AB}), for any $\chi \in [0,1]$, the following result holds:
    \begin{align*}
        \left| \cP_{\dagger,k}(\chi) - \cP_{\text{D},k}(\chi) \right| \le 4 (1+\Lambda_{\max}) L_{1\dagger} \sigma \sqrt{d_{\dagger}},
    \end{align*}
    where $\Lambda_{\max} \coloneqq \omega^{-1} U J_{\max}$, $L_{1\text{\textcolor{vibrantRed}{P}}} \coloneqq \LCJdmax$, $L_{1\text{\textcolor{vibrantBlue}{A}}} \coloneqq \LCnsmax$, $d_{\text{\textcolor{vibrantRed}{P}}}\coloneqq \dt$, and $d_{\text{\textcolor{vibrantBlue}{A}}}\coloneqq \da$.
\end{thr}
\begin{proof}
    Considering to be at iteration $k$ of \cpg, we aim to find an upper bound to the quantity $|\cP_{\dagger,k}(\chi) - \cP_{\text{D},k}(\chi)|$. Let $(\vtheta^{*}_{\text{D},\omega},\vlambda^{*}_{\text{D},\omega})$ be a saddle point of $\cL_{\text{D},\omega}$ and let $(\vtheta^{*}_{\dagger,\omega},\vlambda^{*}_{\dagger,\omega})$ be a saddle point of $\cL_{\dagger,\omega}$. Moreover, since we are considering the learning process of \cpg, the norm of the Lagrange multipliers is bounded by $\| \vlambda \|_{1} \le \sqrt{U} \| \vlambda \|_{2} \le \omega^{-1} U J_{\max} \eqqcolon \Lambda_{\max}$ due to the regularization. 
    The first thing we do is to quantify the quantity $|\cP_{\dagger,k}(\chi) - \cP_{\text{D},k}(\chi)|$:
    \begin{align*}
        &\left|\cP_{\dagger,k}(\chi) - \cP_{\text{D},k}(\chi)\right| \\
        &= \left| \E\left[ \max_{\vlambda \in \mR^{U}_{\ge 0}} \cL_{\dagger,\omega}(\vtheta_{k}, \vlambda) - \cL_{\dagger,\omega}(\vtheta^*_{\dagger,\omega},\vlambda^*_{\dagger,\omega}) \right] + \chi \E\left[ \max_{\vlambda \in \mR^{U}_{\ge 0}} \cL_{\dagger,\omega}(\vtheta_{k}, \vlambda) - \cL_{\text{D},\omega}(\vtheta_{k},\vlambda_{k}) \right] \right.  \\
        &\quad \left. - \E\left[ \max_{\vlambda \in \mR^{U}_{\ge 0}} \cL_{\text{D},\omega}(\vtheta_{k}, \vlambda) - \cL_{\text{D},\omega}(\vtheta^*_{\text{D},\omega},\vlambda^*_{\text{D},\omega}) \right] - \chi \E\left[ \max_{\vlambda \in \mR^{U}_{\ge 0}} \cL_{\text{D},\omega}(\vtheta_{k}, \vlambda) - \cL_{\text{D},\omega}(\vtheta_{k},\vlambda_{k}) \right] \right|.
    \end{align*}
    
    For the sake of readability, we introduce the following quantities:
    \begin{align*}
        &\textsf{A} \coloneqq \max_{\vlambda \in \mR^{U}_{\ge 0}}\cL_{\dagger,\omega}(\vtheta_{k},\vlambda) - \max_{\vlambda \in \mR^{U}_{\ge 0}}\cL_{\text{D},\omega}(\vtheta_{k},\vlambda), \\
        &\textsf{B} \coloneqq \cL_{\text{D},\omega}(\vtheta^{*}_{\text{D},\omega},\vlambda^{*}_{\text{D},\omega}) - \cL_{\dagger,\omega}(\vtheta^{*}_{\dagger,\omega},\vlambda^{*}_{\dagger,\omega}), \\
        &\textsf{C} \coloneqq \cL_{\text{D},\omega}(\vtheta_{k},\vlambda_{k}) - \cL_{\dagger,\omega}(\vtheta_{k},\vlambda_{k}).
    \end{align*}

    The quantity $|\cP_{\dagger,k}(\chi) - \cP_{\text{D},k}(\chi)|$ we aim to bound, can be expressed as:
    \begin{align*}
        \left| \cP_{\dagger,k}(\chi) - \cP_{\text{D},k}(\chi) \right| = \left| (1+\chi) \E\left[ \textsf{A} \right]  + \E\left[ \textsf{B} \right] + \chi \E\left[ \textsf{C} \right] \right|.
    \end{align*}

    By simply applying the triangular and Jensen's inequalities, it holds that:
    \begin{align*}
        \left| \cP_{\dagger,k}(\chi) - \cP_{\text{D},k}(\chi) \right| \le  (1+\chi) \E\left[ \left| \textsf{A} \right| \right]  + \E\left[ \left| \textsf{B} \right| \right] + \chi \E\left[ \left| \textsf{C} \right| \right].
    \end{align*}

    Thus, we can now focus on bounding the terms \textsf{A}, \textsf{B}, and \textsf{C}.
    Starting from \textsf{A}:
    \begin{align*}
        \left| \textsf{A} \right| &= \left| \max_{\vlambda \in \mR^{U}_{\ge 0}}\cL_{\dagger,\omega}(\vtheta_{k},\vlambda) - \max_{\vlambda \in \mR^{U}_{\ge 0}}\cL_{\text{D},\omega}(\vtheta_{k},\vlambda) \right| \\
        &\le \left| \max_{\vlambda \in \mR^{U}_{\ge 0}} \left(\cL_{\dagger,\omega}(\vtheta_{k},\vlambda) - \cL_{\text{D},\omega}(\vtheta_{k},\vlambda) \right) \right| \\
        &\le (1 + \Lambda_{\max}) L_{1\dagger} \sigma \sqrt{d_{\dagger}},
    \end{align*}
    where in the last step we applied the result from Theorem~\ref{thr:L-deploy}. 
    We can now focus on the term \textsf{B}:
    \begin{align*}
        \left| \textsf{B} \right| &= \left| \cL_{\text{D},\omega}(\vtheta^{*}_{\text{D},\omega},\vlambda^{*}_{\text{D},\omega}) - \cL_{\dagger,\omega}(\vtheta^{*}_{\dagger,\omega},\vlambda^{*}_{\dagger,\omega}) \right| \\
        &\le \left| \cL_{\text{D},\omega}(\vtheta^{*}_{\dagger,\omega},\vlambda^{*}_{\text{D},\omega}) - \cL_{\dagger,\omega}(\vtheta^{*}_{\dagger,\omega},\vlambda^{*}_{\text{D},\omega}) \right| \\
        &\le (1 + \Lambda_{\max}) L_{1\dagger} \sigma \sqrt{d_{\dagger}},
    \end{align*}
    where we first exploited the properties of saddle points, then we leveraged Theorem~\ref{thr:L-deploy}.
    Finally, we can focus on the term \textsf{C}:
    \begin{align*}
        \left| \textsf{C} \right| = \left| \cL_{\text{D},\omega}(\vtheta_{k},\vlambda_{k}) - \cL_{\dagger,\omega}(\vtheta_{k},\vlambda_{k}) \right| \le (1+\Lambda_{\max}) L_{1\dagger} \sigma \sqrt{d_{\dagger}},
    \end{align*}
    again by applying Theorem~\ref{thr:L-deploy}. 

    Putting all together, and bounding $\chi \le 1$, we conclude the following:
    \begin{align*}
        \left| \cP_{\dagger,k}(\chi) - \cP_{\text{D},k}(\chi) \right| \le 4 (1+\Lambda_{\max}) L_{1\dagger} \sigma \sqrt{d_{\dagger}}.
    \end{align*}
\end{proof}
Notice that the result from Theorem~\ref{thr:pot-loss} has exactly the same form as the one from Theorem~\ref{thr:L-deploy}, except for the fact that here, since we are dealing exclusively with regularized Lagrangian functions, we directly exploited the fact that the Lagrange multipliers have a norm bounded by $\Lambda_{\max}$.

\begin{restatable}[Sample Complexity for Deterministic Deployment]{thr}{DetSC} \label{thr:det-sc}
    Suppose to run \emph{\cpg} for $K$ iterations employing a (hyper)policy complying with Definitions~\ref{def:white_noise_policies} (\AB) or~\ref{def:white_noise_hyperpolicies} (\PB).
    Suppose to be under Assumptions~\ref{asm:Jd-reg-theta} (\PB) or~\ref{asm:Jd-reg-ns} (\AB), \ref{asm:wgd-det}, \ref{asm:mu-reg} (\AB), and~\ref{asm:bounded-noise-scores}. For $\psi \in [1,2]$, $\chi < 1/5$, sufficiently small $\epsilon$ and $\omega$, and a choice of \emph{constant} learning rates $\zeta_{\vlambda} = \cO(\omega \sigma^{2} \epsilon^{2/\psi})$ and $\zeta_{\vtheta} = \omega \zeta_{\vlambda}$,  whenever $K = \cO(\omega^{-3} \sigma^{-2} \epsilon^{-\frac{4}{\psi}+1})$ and the gradients are estimated, we have that:
    \begin{align*}
        \cP_{\text{D},K}(\chi) \le \epsilon + \frac{\beta_{\dagger}(\sigma,\psi)}{\dalpha} + 4(1+\Lambda_{\max})L_{1\dagger} \sigma \sqrt{d_{\dagger}}, 
    \end{align*}
   where $\beta_{\dagger}(\sigma,\psi)$ is quantified in Theorem~\ref{thr:wgd-inherited}, $\Lambda_{\max} \coloneqq \omega^{-1} U J_{\max}$, $L_{1\text{\textcolor{vibrantRed}{P}}} \coloneqq \LCJdmax$, $L_{1\text{\textcolor{vibrantBlue}{A}}} \coloneqq \LCnsmax$, $d_{\text{\textcolor{vibrantRed}{P}}}\coloneqq \dt$, and $d_{\text{\textcolor{vibrantBlue}{A}}}\coloneqq \da$.
\end{restatable}
\begin{proofsketch}
    Under the considered set of assumptions, we recover the results of Theorems~\ref{thr:wgd-inherited}, \ref{thr:inherited-L-reg}, and~\ref{thr:pot-loss} and of Lemma~\ref{lem:bounded-estim-var}, matching the conditions needed to establish the sample complexity exhibited in Theorem~\ref{thr:convergencePot} for ensuring that $\cP_{\dagger,K}(\chi) \le \epsilon + \frac{\beta_{\dagger}(\sigma.\psi)}{\dalpha}$, where we employed the coefficients of the inherited weak $\psi$-GD (Theorem~\ref{thr:wgd-inherited}).

    In particular, recovering \quotes{Part V: Rates Computation} of the proof of Theorem~\ref{thr:convergencePot} with $\psi \in [1,2]$ and inexact gradients, considering $L_1=\cO(1)$ and $L_2=\cO(\omega^{-1})$ (Theorem~\ref{thr:inherited-L-reg}), while $V_{\dagger,\vtheta} = \cO(\omega^{-2} \sigma^{-2})$ and $V_{\vlambda} = \cO(1)$ (Lemma~\ref{lem:bounded-estim-var}), and selecting $\zeta_{\vlambda} = \cO(\omega \sigma^{2} \epsilon^{2/\psi})$ and $\zeta_{\vtheta} = \cO(\omega^{3} \sigma^{2} \epsilon^{2/\psi})$, we have that, with a sample complexity of order:
    \begin{align*}
        K \le \cO\left( \frac{\log \frac{1}{\epsilon}}{\omega^{3} \sigma^{2} \epsilon^{-1 + 4/\psi}} \right),
    \end{align*}
    it is guaranteed that $\cP_{\dagger,K}(\chi) \le \epsilon + \frac{\beta_{\dagger}(\sigma,\psi)}{\dalpha}$. Now, exploiting the result of Theorem~\ref{thr:pot-loss}, we have that the same sample complexity ensures that:
    \begin{align*}
        \cP_{\text{D},K}(\chi) \le \epsilon + \frac{\beta_{\dagger}(\sigma,\psi)}{\dalpha} + 4(1+\Lambda_{\max})L_{1\dagger} \sigma \sqrt{d_{\dagger}}.
    \end{align*}
\end{proofsketch}

We highlight that all the remarks made for Theorem~\ref{thr:convergencePot} also apply in this case. Moreover, the sample complexity achieved here matches that of Theorem~\ref{thr:convergencePot}, with the exception of an additional $\sigma^{-2}$ factor. This term arises from the characterization of the constant $V_{\dagger,\vtheta} = \cO(\omega^{-2} \sigma^{-2})$ in Lemma~\ref{lem:bounded-estim-var}, due to the specific noise model employed to define the \AB and \PB exploration paradigms.

We also note that Theorem~\ref{thr:conversion} remains applicable in this setting, as it holds for general potential functions $\cP_{\dagger,k}(\chi)$, including the deterministic one $\cP_{\text{D},k}(\chi)$. As discussed in Section~\ref{sec:alg_general}, and consistently with the analysis without deterministic deployment, Theorem~\ref{thr:conversion} suggests choosing $\omega = \cO(\epsilon)$. In the current context, Theorem~\ref{thr:det-sc} further suggests setting $\sigma = \cO(\epsilon)$ in order to ensure that $\cP_{\text{D},K}(\chi) \le \cO(\epsilon) + c$, where $c$ is a constant.

Therefore, by letting both $\omega$ and $\sigma$ scale as $\cO(\epsilon)$, the resulting sample complexity becomes $K = \cO(\epsilon^{4 - 4/\psi} \log \epsilon^{-1})$, as summarized in Table~\ref{tab:det-sc}. We emphasize that this result aligns with the known sample complexity bounds for policy gradient methods under the same noise model in the unconstrained setting~\citep{montenegro2024learning}.

Nonetheless, we remark that the theoretical recommendation of setting $\sigma = \cO(\epsilon)$ may be impractical in real applications. In practice, using such a low level of stochasticity may lead to slower convergence, even though it results in a more accurate deterministic deployment~\citep{montenegro2025convergence}.

Finally, by retrieving the specific quantities of the \AB and \PB exploration paradigms, we recover the known trade-off between them~\citep{metelli2018policy}: \AB may suffer from long interaction horizons or high-dimensional action spaces (large $T$ or $\da$), whereas \PB may suffer high-dimensional parameter spaces (large $\dt$).

\begin{table}[t!]
\begin{center}
{\renewcommand{\arraystretch}{1.5}
\medmuskip=0mu
\thinmuskip=0mu
\thickmuskip=0mu
\begin{tabular}{|c||c|c|c|}
    \hline
    & \multicolumn{3}{c|}{\cellcolor{vibrantOrange!40} \bfseries Deterministic Deployment with Estimated Gradients} \\\cline{2-4}
    & \cellcolor{vibrantOrange!20} $\psi=1$ (GD) & \cellcolor{vibrantOrange!20} $\psi \in (1,2)$ & \cellcolor{vibrantOrange!20} $\psi=2$ (PL)   \\ \hline\hline
    \cellcolor{vibrantGrey!20} \bfseries Fixed $\omega$ and $\sigma$ & $\omega^{-3} \sigma^{-2} \epsilon^{-3}\log(\epsilon^{-1})$ & $\omega^{-3} \sigma^{-2} \epsilon^{1-4/\psi} \log({\epsilon^{-1}})$ & $\omega^{-3} \sigma^{-2} \epsilon^{-1}\log(\epsilon^{-1})$ \\ \hline
    \cellcolor{vibrantGrey!20} \bfseries $\omega= \cO(\epsilon)$ and $\epsilon= \cO(\epsilon)$ & $\epsilon^{-8}\log(\epsilon^{-1})$ & $\epsilon^{-4-4/\psi}\log(\epsilon^{-1})$ & $\epsilon^{-6}\log(\epsilon^{-1})$ \\ \hline
\end{tabular}}
\caption{Summary of the sample complexity results for the deterministic deployment \cpg when considering estimated gradients.}
\label{tab:det-sc}
\end{center}
\end{table}

%% file: contents/07_related.tex
\section{Related Work} \label{apx:related}

In this section, we review related work primarily focusing on policy optimization methods for CRL, convergence guarantees of primal-dual approaches, and the learning of deterministic policies in constrained environments.

\paragraph{Policy Optimization Approaches for Constrained Reinforcement Learning.}
Policy Optimization based algorithms for Constrained Reinforcement Learning mostly follow \emph{primal}-only or \emph{primal-dual} approaches.
\emph{Primal}-only algorithms~\citep{dalal2018safe,chow2018lyapunov,yu2019convergent,liu2020ipo,xu2021crpo} avoid considering dual variables by focusing on the design of the objective function and by designing the update rules for the policy at hand incorporating the constraint satisfaction part.

The main benefit of employing primal-only algorithms lies in the fact that there is no need to consider another variable to learn, and therefore, no need to tune its learning rate. However, few of the existing methods establish global convergence to an optimal feasible solution. For instance, \cite{xu2021crpo} propose CRPO, an algorithm employing an \emph{unconstrained} policy maximization update taking into account the reward when all the constraints are satisfied, while leveraging on-policy minimization updates in the direction of violated constraint functions. Moreover, it exhibits average global convergence guarantees for the tabular setting.
On the other hand, \emph{primal-dual} algorithms~\citep{chow2017risk,achiam2017constrained,tessler2018reward,stooke2020responsive,ding2020natural,ding2021provably,bai2022achieving,ying2022dual,bai2023achieving,gladin2023algorithm,ding2024last} are the most commonly used and investigated. 
Indeed, the effectiveness of using the primal-dual approach is justified by \cite{paternain2019constrained}, which states that this kind of approach has zero duality gap under Slater's condition when optimizing over the space of all the possible stochastic policies.
Among the reported works, \cite{stooke2020responsive} propose PID Lagrangian, a method to update the dual variable, smoothing the oscillations around the threshold value of the costs during the learning. The practical strength of such a method is that it can be paired with any of the existing policy optimization methods. The other cited works are treated in detail in the next paragraph.

\paragraph{Lagrangian-based Policy Search Convergence Guarantees.} A lot of research effort has been spent on studying the convergence guarantees of primal-dual policy optimization methods. In this field, the goal is to ensure last-iterate convergence guarantees with rates that are dimension-free, \ie independent of the state and action spaces' dimensions, and work with multiple constraints. In the rest of this paragraph, we talk about single time-scale algorithms when the methods at hand prescribe the usage of the same step sizes for both the primal and dual variables' updates.
\cite{felisia2002self} and \cite{bhatnagar2012online} propose primal-dual policy gradient-based methods built upon distinct time scales and relying on nested loops. Such methods only show \emph{asymptotic} convergence guarantees.
\cite{chow2017risk} propose two primal-dual methods ensuring \emph{asymptotic} convergence guarantees. The peculiarity of those methods lies in the fact that their notion of CMDP encapsulates risk-based constraints, introducing an additional learning variable.
Their algorithms have guarantees of \emph{asymptotic} convergence to stationary points.
The recent works by \cite{zheng2022constrained} and \cite{moskovitz2023reload} also propose methods ensuring \emph{asymptotic} global convergence guarantees. These methods exploit occupancy-measure iterates rather than policy iterates.
\cite{ding2020natural} propose NPG-PD, which relies on a natural policy gradient approach and, under Slater's assumption, ensures dimension-dependent \emph{average}-iterate global convergence guarantees in the single-constrained setting with a single time-scale and with exact gradients. This work has been extended by \cite{ding2022convergence}, striking dimension-free rates, but still just guaranteeing \emph{average}-iterate convergence with exact gradients. However, sample-based versions of NPG-PD achieving, under additional assumptions, the same convergence rates are provided by the authors.
Another work ensuring an \textit{average}-iterate rate is the one by~\citet{liu2021policy}.
The latter exhibits a convergence rate of order $\tilde{\cO}(\epsilon^{-1})$, considering to act in tabular CMDPs with softmax policies and having access to exact gradients and to a generative model. \citet{liu2021policy} propose also a sample-based version of their algorithm, keeping the same setting previously described, which ensures a convergence rate on \textit{average} of order $\tilde{\cO}(\epsilon^{-3})$.
Both \cite{ying2022dual} and \cite{gladin2023algorithm} propose algorithms involving regularization. The proposed methods rely on natural policy-based subroutines and show dimension-dependent \emph{last-iterate} global convergence guarantees, relying on two time-scales. These methods work also with multiple constraints.
Additionally, \cite{ding2024last} propose RPG-PD and OPG-PD, exhibiting \emph{last-iterate} global convergence guarantees under Slater's condition in a single-constraint setting. The former is a regularized version of the algorithm proposed by \cite{ding2020natural}, showing \emph{last-iterate} global convergence at a sublinear rate. The latter leverages on the optimistic gradient method~\citep{hsieh2019convergence} to unlock a faster linear convergence rate. These methods show single time-scale dimension-dependent rates and both leverage on exact gradients. However, for RPG-PD there exists an \emph{inexact} version showing, under additional assumptions on the statistical and transfer errors and the relative condition number~\citep[][Assumption 2]{ding2024last}, the same guarantees of the exact one.
Finally, \cite{mondal2024last} introduce PDR-ANPG a primal dual-based regularized accelerated natural policy gradient algorithm that utilizes entropy and quadratic regularizers in the Lagrangian function and a natural policy gradient (NPG) oracle as a subroutine. This method, which is not single time scale, operates in the setting of CMDPs with $|\cS| = +\infty$, $|\cA| < +\infty$, and $U=1$. Under the Slater condition, considering a fisher information matrix (FIM) to be positive definite, and considering a general policy parameterization with a policy-class error $\epsilon_{\text{bias}}>0$, PDR-ANPG exhibits a sample complexity of order $\cO(\epsilon^{-2} \min\{\epsilon^{-2}, (\epsilon_{\text{bias}}^{-1/3})\} \log^{2}\epsilon^{-1})$.
It is worth noticing that all the mentioned works just consider the \emph{action-based} exploration approach for policy optimization, while the \emph{parameter-based} one remains unexplored.
For the sake of clarity, Table~\ref{tab:convergence_comparison} shows a detailed comparison between our approach and the other presented methods exhibiting \emph{last-iterate} global convergence guarantees.

Furthermore, \citet{vaswani2022near} have recently proposed a dimension-dependent lower bound for the sample complexity of $\Omega\left(\epsilon^{-2}\right)$, assuming to be under the Slater condition and considering single-constrained CMDPs with finite state and action spaces.

\setlength\extrarowheight{4pt}
\begin{table}[t]{
    \renewcommand{\arraystretch}{1.5}
        \resizebox{\linewidth}{!}{
        \begin{tabular}{|c||c|c|c|c|c|c|c|c|}
            \hline
            \bfseries Algorithm & \bfseries Dimension-free & \bfseries Setting & \bfseries \makecell{Exploration\\Type} & \bfseries \makecell{Single\\time-scale} & \bfseries Gradients & \bfseries Assumptions & \bfseries \makecell{Sample\\Complexity} & \bfseries \makecell{Iteration\\Complexity} \\
            \hline
            \hline
            \makecell{Dual Descent \\ \citep{ying2022dual}} & \xmark & \makecell{$U \ge 1$ \\ $T=\infty$ \\ Softmax param.} & \textcolor{vibrantBlue}{AB} & \xmark & Inexact & \makecell{Slater\\Sufficient Exploration} & $\cO\left( \epsilon^{-2} \log^2 \epsilon^{-1} \right)$ & $\cO\left( \log^2 \epsilon^{-1} \right)$ \\
            \hline
            \makecell{Cutting-Plane \\ \citep{gladin2023algorithm}} & \xmark & \makecell{$U \ge 1$ \\ $T=\infty$ \\ Softmax param.} & \textcolor{vibrantBlue}{AB} & \xmark & Inexact & \makecell{Slater\\Uniform Ergodicity\\Oracle} & $\cO\left( \epsilon^{-4} \log^3 \epsilon^{-1} \right)$ & $\cO\left( \log^3 \epsilon^{-1} \right)$ \\
            \hline
            \makecell{Exact RPG-PD \\ \citep{ding2024last}} & \xmark & \makecell{$U=1$\\$T=\infty$\\Softmax param.} & \textcolor{vibrantBlue}{AB} & \cmark & \makecell{Exact} & \makecell{Slater} & - & $\cO\left(\epsilon^{-6} \log^2 \epsilon^{-1}\right)$ \\
            \hline
            \makecell{Inexact RPG-PD \\ \citep{ding2024last}} & \xmark & \makecell{$U=1$\\$T=\infty$\\Softmax param.} & \textcolor{vibrantBlue}{AB} & \cmark & \makecell{Inexact} & \makecell{Slater\\Stat. Err. Bounded\\Transf. Err. Bounded\\Cond. Num. $< +\infty$} & $\cO\left(\epsilon^{-6} \log^2 \epsilon^{-1}\right)$  & $\cO\left(\epsilon^{-6} \log^2 \epsilon^{-1}\right)$ \\
            \hline
            \makecell{OPG-PD \\ \citep{ding2024last}} & \xmark & \makecell{$U=1$\\$T=\infty$\\Softmax param.} & \textcolor{vibrantBlue}{AB} & \cmark & Exact & Slater & - & $\cO\left(\log^2 \epsilon^{-1}\right)$ \\
            \hline
            \makecell{PDR-ANPG \\ \citep{mondal2024last}} & \xmark & \makecell{$U=1$\\$T=\infty$\\General param.} & \textcolor{vibrantBlue}{AB} & \xmark & Exact & \makecell{Slater\\NPG Oracle\\$\epsilon_{\text{bias}}>0$\\FIM Positive Definite} & $\cO(\epsilon^{-2} \min\{\epsilon^{-2},(\epsilon_{\text{bias}})^{-1/3}\}\log^{2}\epsilon^{-1})$ & - \\
            \hline
            \rowcolor{lavendermist} 
            \Gape[0pt][2pt]{\makecell{\cpg\\(\texttt{This work})}} & \cmark & \Gape[0pt][2pt]{\makecell{$U \ge 1$\\$T \in \mathbb{N} \cup \{ \infty\}$\\General param.}} & \textcolor{vibrantBlue}{AB} and \textcolor{vibrantRed}{PB} & \xmark & \Gape[0pt][2pt]{\makecell{Exact\\Inexact}} & \Gape[0pt][2pt]{\makecell{Asm. \ref{asm:assunzione},\\ \ref{asm:wgd}, \ref{asm:L_grad_lip}, \ref{ass:boundedVariance}}} & Table~\ref{tab:summary} & Table~\ref{tab:summary} \\
            \hline
            \hline
            \makecell{Lower Bound \\ \citep{vaswani2022near}} & \xmark & \makecell{$U=1$\\$T=\infty$} & - & - & Inexact & Slater & $\Omega\left(\epsilon^{-2}\right)$ & -\\
            \hline
        \end{tabular}
    }}

    \vspace{0.3cm}
    
    \caption{Comparison among primal-dual methods ensuring last-iterate global convergence guarantees.}
    \label{tab:convergence_comparison}
\end{table}

\paragraph{Learning Deterministic Policies in CMDPs.}
While policy gradient methods have been extensively studied in the context of CRL, most existing approaches focus exclusively on stochastic policies, whereas the study of deterministic policies for CMDPs has received comparatively little attention. Deterministic policies, however, are crucial for real-world applications where reliability, safety, and predictability are essential. Despite this, very few works have tackled the problem of learning deterministic policies in CRL settings, particularly in continuous-state and continuous-action CMDPs.
A recent contribution in this direction is presented by \citet{rozada2024deterministic}, who introduce the \textit{Deterministic Policy Gradient Primal-Dual} (D-PGPD) algorithm, a novel method designed to directly learn deterministic policies in CMDPs. The proposed approach leverages an \emph{entropy-regularized Lagrangian} formulation, where the primal update performs a proximal-point-type ascent step solving a quadratic-regularized maximization sub-problem, while the dual update performs a gradient descent step solving a quadratic-regularized minimization sub-problem. The theoretical contribution of this work is the proof that the exact version of D-PGPD achieves asymptotically an $\epsilon$-optimal solution in $\cO(\epsilon^{-6})$ iterations, making it one of the first primal-dual methods that directly optimize deterministic policies in  CRL. Moreover, an approximated version of D-PGPD (namely AD-PGPD), incorporating function approximation, achieves the same convergence rate under the assumption that the approximation error satisfies $\epsilon_{\text{approx}} = \cO(\epsilon^4)$. These rates hold under the assumption that the model of the environment is known. 
A key limitation of the proposed approach is that it considers only a single constraint, restricting its applicability to more complex multi-constrained CMDPs commonly found in real-world applications. Additionally, unlike stochastic policy-based CRL approaches, where exploration is driven by the inherent randomness of the policy, D-PGPD learns a deterministic policy directly, relying entirely on the environment to provide the required exploration. While this design ensures stable and consistent policy execution, it also presents a major drawback: the lack of explicit exploration mechanisms significantly limits the applicability of D-PGPD in practice, as it may struggle in environments where intrinsic randomness is insufficient to ensure adequate state-action space coverage.
To extend D-PGPD to a model-free setting, the authors propose a sample-based version of AD-PGPD, which leverages rollouts for policy evaluation. However, this approach introduces a significant complexity increase: the model-free algorithm requires $\cO(\epsilon^{-18})$ rollouts to compute an $\epsilon$-optimal policy, making it substantially less practical for large-scale real-world tasks.

%% file: contents/06_exp.tex
\section{Numerical Validation} \label{sec:experiments}
In this section, we empirically validate the theoretical results established throughout the paper. Further experimental details are reported in Appendix~\ref{apx:exp}.\footnote{The code to reproduce the experiments is available at~\url{https://github.com/MontenegroAlessandro/MagicRL/tree/constraints}.}

Before proceeding, we clarify that \cpgpe and \cpgae refer to the \PB and \AB variants of \cpg, respectively. We denote by \dcpgpe and \dcpgae their corresponding versions with deterministic deployment, obtained by switching off the stochasticity (\ie setting $\sigma = 0$). Accordingly, the curves shown for \dcpgpe and \dcpgae correspond to the performances and costs of the deterministic policies encountered during training by \cpgpe and \cpgae, respectively.

The section is organized as follows. Section~\ref{subsec:comparison} presents comparisons between the proposed algorithms and state-of-the-art baselines, Section~\ref{subsec:exp_deterministic} investigates the effects of deterministic deployment on policies learned via \cpg methods, and Section~\ref{subsec:sensitivity} provides a sensitivity analysis on the impact of the regularization parameter~$\omega$.

\subsection{Comparison Against Baselines}
\label{subsec:comparison}
\paragraph{Comparison in DGWW.} We compare our proposal \cpgae against the sample-based versions of NPG-PD~\citep[][Appendix H]{ding2020natural} and RPG-PD~\citep[][Appendix C.9]{ding2024last}.
The environment in which the methods are tested is the Discrete Grid World with Walls (DGWW, see Appendix~\ref{apx:exp}) with a horizon of $T=100$.
In this experiment, all the methods aim to learn the parameters of a tabular softmax policy with $196$ parameters, maximizing the trajectory reward while considering a single constraint on the average trajectory cost, for which we set a threshold $b=0.2$. 
All the methods were run for $K=3000$ iterations with a batch size of $N=10$ trajectories per iteration, and with constant learning rates.
In particular, for both \cpgae and NPG-PD, we employed $\zeta_{\vtheta} = 0.01$ and $\zeta_{\vlambda} = 0.1$, while for RPG-PD we selected $\zeta_{\vtheta} = 0.01$ and $\zeta_{\vlambda} = 0.01$.
For \cpgae and RPG-PD, we used a regularization constant $\omega=10^{-4}$.
All the details about the experimental setting are summarized in Table~\ref{tab:exp-dgww}.
We would like to stress that, as prescribed by the respective convergence theorems, we chose a two-timescale learning rate approach for \cpgae and a single-timescale one for RPG-PD.
Figures~\ref{subfig:ding_comparison-dgww-perf} and~\ref{subfig:ding_comparison-dgww-cost} show the performance curves (\ie the one associated with the objective function and the one for the costs). 
As can be noticed, \cpgae manages to strike the objective of the constrained optimization problem with less trajectories. 
Indeed, the sample-based NPG-PD requires to estimate the value and the action-value functions for all the states and state-action pairs, resulting in analyzing $|\cS| + |\cS| |\cA|$ additional trajectories \wrt \cpgae for every iteration of the algorithm.
The sample-based RPG-PD also requires additional trajectories to be analyzed, which, in practice, for a correct learning behavior, result to be the same in number as the extra ones analyzed by NPG-PD.
In this environment, \dcpgae exhibits almost the same behavior of \cpgae, thus meaning that the encountered stochastic policies do not meet a significant loss in performances and costs when switching off the stochasticity.

\begin{figure}[t!]
    \centering
    \subfloat[Performance Comparison in DGWW.\label{subfig:ding_comparison-dgww-perf}]{
        \includegraphics[width=0.38\textwidth]{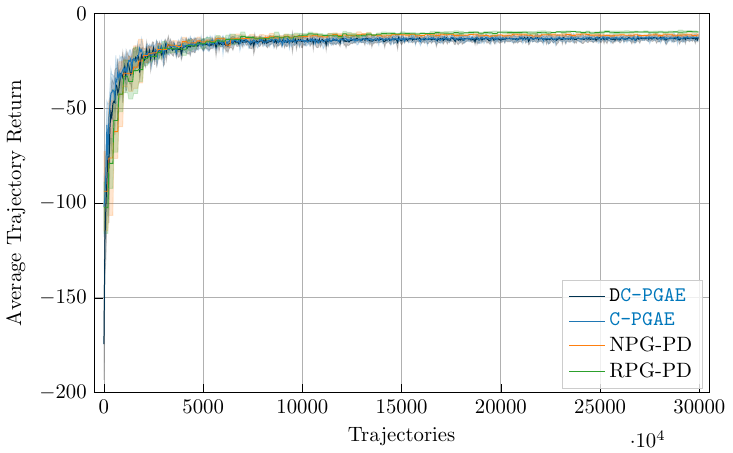}
    }
    \hspace{0.5cm}
    \subfloat[Performance Comparison in LQR.\label{subfig:ding_comparison-lqr-perf}]{
        \includegraphics[width=0.38\textwidth]{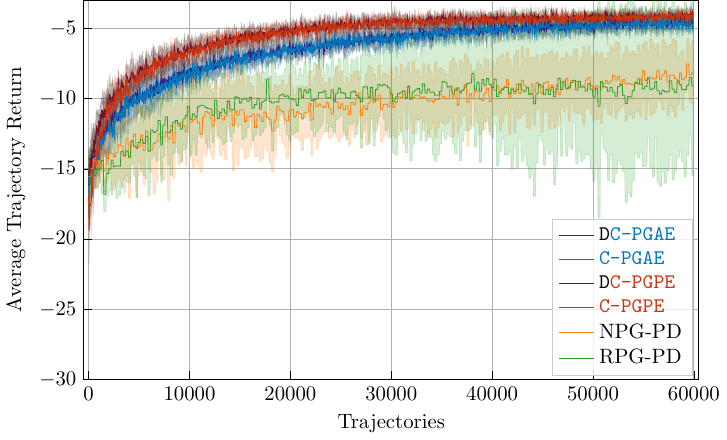}
    } 

    \vspace{0.8cm}
    
    \subfloat[Cost Comparison in DGWW.\label{subfig:ding_comparison-dgww-cost}]{
        \includegraphics[width=0.38\textwidth]{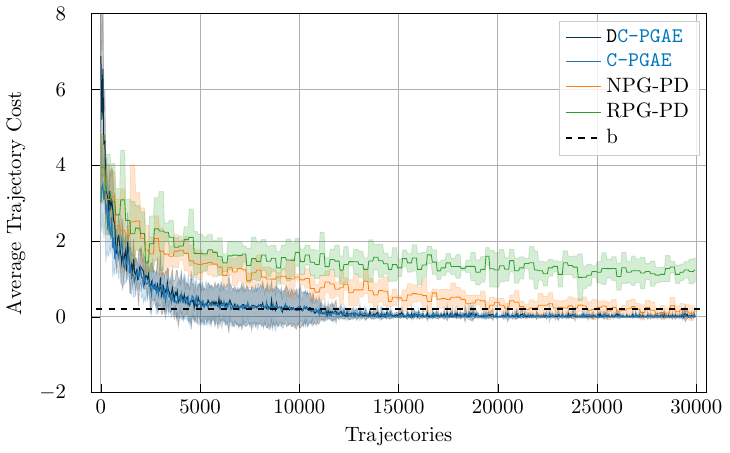}
    }
    \hspace{0.5cm}
    \subfloat[Cost Comparison in LQR.\label{subfig:ding_comparison-lqr-cost}]{
        \includegraphics[width=0.38\textwidth]{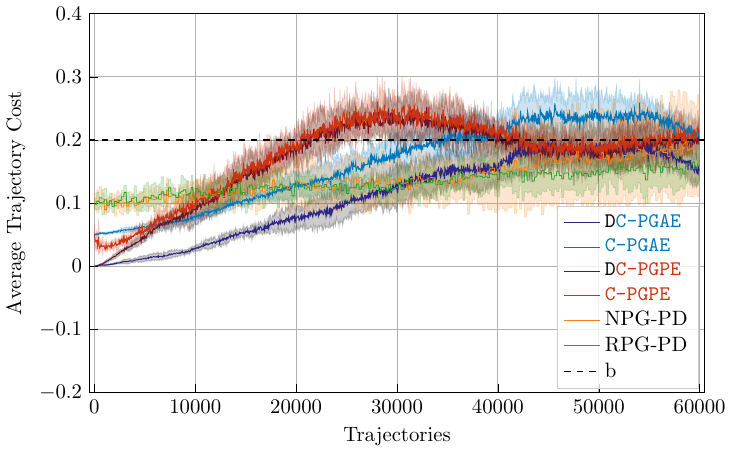}
    }

    \caption{Average return and cost curves in the \emph{CostLQR} and \emph{DGWW} environments (5 runs, mean $\pm$ $95\%$ C.I.).}
    \label{fig:ding_comparison}
\end{figure}

\paragraph{Comparison in LQR.} We compare our proposals \cpgae and \cpgpe against the continuous sample-based version of NPG-PD~\citep[][Algorithm 1]{ding2022convergence} with works with generic policy parameterizations. 
In the following, we refer to this version of NPG-PD as NPG-PD2.
Moreover, we added a ridge-regularized version of NPG-PD2, which we call RPG-PD2, to resemble the type of regularization we employed for our proposed methods.
For all the regularized methods (\ie \cpgae, \cpgpe, and RPG-PD2) we selected as regularization constant $\omega=10^{-4}$.
The setting for this experiment considers a bidimensional LQR environment with a single cost over the provided actions (see Appendix~\ref{apx:exp}) and with a fixed horizon $T=50$.
Here, the methods aim at maximizing the average reward over trajectories, while keeping the average cost over trajectories under the threshold $b=0.9$.
In particular, \cpgae learns the parameters of a linear gaussian policy with a variance $\sigma_{\text{A}}^2 = 10^{-3}$ and employing a learning rate schedule governed by the Adam scheduler~\citep{kingma2015adam} with $\zeta_{\vtheta,0}=0.001$ and $\zeta_{\vlambda,0}=0.01$.
\cpgpe learns the parameters of a Gaussian hyperpolicy, with a variance $\sigma_{\text{P}}^2 = 10^{-3}$, which samples the parameters of a deterministic linear policy. It employs a learning rate schedule also governed by Adam with $\zeta_{\vrho,0}=0.001$ and $\zeta_{\vlambda,0}=0.01$.
Both \cpgae and \cpgpe were run for $K=6000$ iterations with a batch of $N=100$ trajectories per iteration.
NPG-PD2 and RPG-PD2 are both actor-critic methods which were run for $K=1000$ iterations with a batch size of $N=600$ trajectories per iteration. In particular, among the trajectories of the reported batch size, $N_1 = 500$ was used for the inner critic-loop, while $N_2 = 100$ for performance and cost estimations. The inner loop step size was selected as a constant, as prescribed by the original algorithm, and with a value $\alpha=10^{-5}$. Furthermore, since such methods were designed for infinite-horizon discounted environments, we tested them on the same LQR as for \cpgae and \cpgpe, but leaving $T=1000$ and $\gamma=0.98$ (the effective horizon is $(1-\gamma)^{-1}=50$). The step sizes for the primal and dual variables updates were governed by Adam with $\zeta_{\vtheta,0}=0.003$ and $\zeta_{\vlambda,0}=0.01$. As for \cpgae, both NPG-PD2 and RPG-PD2 aimed at learning the parameters of a linear Gaussian policy, with variance $\sigma_{\text{A}}^2 = 10^{-3}$.
All the details about this experiment are summarized in Table~\ref{tab:exp-cost-lqr}.
Figures~\ref{subfig:ding_comparison-lqr-perf} and~\ref{subfig:ding_comparison-lqr-cost} report the learning curves for the average return and the cost over trajectories.
As can be seen, our methods manage to solve the constrained optimization problem at hand by leveraging on less trajectories.
Indeed, NPG-PD2 and RPG-PD2 suffer the inner critic loop, which adds additional trajectories to be analyzed per iteration (in this specific case $N_1=500$).
We stress that the actor-critic methods were very sensible to the hyperparameter selection, especially to the length and the step size of the inner loop.
When considering \dcpgpe and \dcpgae, \ie the deterministic policy curves associated with \cpgpe and \cpgae respectively, we observe that their overall behavior remains comparable. However, a notable difference arises in the cost curve under \AB exploration, where a significant reduction in the incurred costs leads the resulting deterministic policies to consistently satisfy the cost constraint.

\paragraph{Comparison in RobotWorld.} We evaluate \cpgae and \cpgpe against the sample-based versions of AD-PGPD~\citep{rozada2024deterministic} and PGDual~\citep{zaho2021pgdual, brunke2021pgdual} in the \emph{RobotWorld} environment (see Appendix~\ref{apx:exp}), which is a modification of the \emph{CostLQR} one with quadratic reward and cost functions and where the agents are allowed to control both velocity and acceleration. The \cpg algorithms operate with a finite horizon of $T = 100$ and a discount factor of $\gamma = 1$, while AD-PGPD and PGDual use an infinite horizon $T = 1000$ with $\gamma = 0.99$. Also in this case, the aim is to maximize the performance function while keeping the unique cost function under the threshold $b = 1000$. We highlight that this experiment is the same presented in~\citep{rozada2024deterministic} to which we added our methods. That being said, both AD-PGPD and PGDual as prescribed in~\citep{rozada2024deterministic}. \cpgpe employs a linear gaussian hyperpolicy with a variance $\sigma_{\text{P}} = 10^{-6}$, collecting a batch of $N=100$ trajectories per iteration, using a regularization parameter $\omega=10^{-4}$, and learning rate schedules governed by Adam~\citep{kingma2015adam} with initial values $\zeta_{\vtheta,0}=5\cdot10^{-6}$ and $\zeta_{\vlambda,0}=5\cdot10^{-3}$.
\cpgae employs a linear gaussian stochastic policy with a variance $\sigma_{\text{A}} = 5\cdot10^{-2}$, collecting a batch of $N=100$ trajectories per iteration, using a regularization parameter $\omega=10^{-4}$, and learning rate schedules governed by Adam~\citep{kingma2015adam} with initial values $\zeta_{\vtheta,0}=5\cdot10^{-6}$ and $\zeta_{\vlambda,0}=10^{-4}$. Both \cpgpe and \cpgae were run for $K = 10000$ iterations. All the hyperparameters are further presented in Table~\ref{tab:exp-cost-robotworld}.
Figure~\ref{fig:ding_comparison_deterministic} presents the learning curves for performance and cost across different algorithms. The results highlight that \cpg-based methods consistently achieve better constraint satisfaction while maintaining competitive performance. Notably, \cpgae and \cpgpe and their deterministic deployment counterparts \dcpgae and \dcpgpe show faster convergence compared to AD-PGPD and PGDual, which exhibit significant variance and instability in both performance and cost. Furthermore,  deterministic variants of \cpg demonstrate a lower constraint violation than the stochastic counterpart, especially when dealing with \PB exploration.

\begin{figure}[t!]
    \centering
    \subfloat[Performance Comparison.]{
        \includegraphics[width=0.38\linewidth]{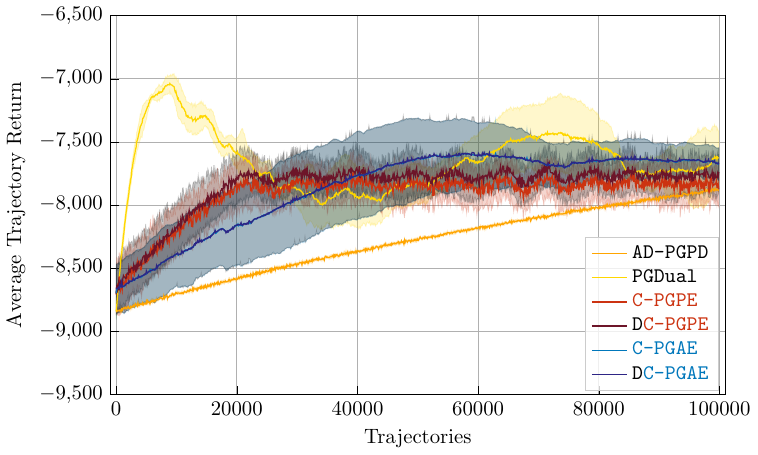}
        \label{fig:ding_performance_comparison}
    }
    \hspace{0.5cm}
    \subfloat[Cost Comparison.]{
        \includegraphics[width=0.38\linewidth]{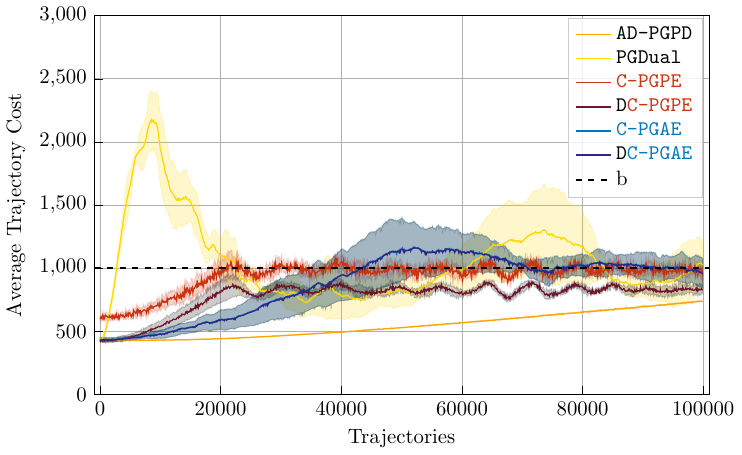}
        \label{fig:ding_cost_comparison}
    }
    \caption{Average return and cost curves in \emph{RobotWorld} (5 runs, mean $\pm$ $95\%$ C.I.).}
    \label{fig:ding_comparison_deterministic}
\end{figure}

\subsection{Deterministic Deployment Study}
\label{subsec:exp_deterministic}

In this experiment, we empirically analyze the deterministic deployment of both \cpgae and \cpgpe when learning by employing a fixed stochasticity ($\sigma > 0$) and then deploying a deterministic policy switching off the stochasticity ($\sigma=0$) of the last parameterization encountered while learning. In particular, here we consider the \emph{CostSwimmer-v4} and the \emph{CostHopper-v4} environments (see Appendix~\ref{apx:exp} for details) with $T=100$ and $\gamma=1$. The employed version for the environments resemble the original one from the MuJoCo control suite~\citep{todorov2012mujoco}, but introducing a cost function representing the energy associated with the control action.
In this set of experiments, we study the difference in performance and cost when switching from a stochastic (hyper)policy to a deterministic policy at the end of the learning. In this case, we averaged the last $100$ iterates to evaluate the actual deterministic deployment. Additionally, we conducted this deployment loss study for diverse values of stochasticity $\sigma$. Both \cpgpe and \cpgae were run for $K=3000$ iterations with a batch size of $N=100$ trajectories collected per iteration. The learning rates were governed by Adam~\citep{kingma2015adam}, the regularization parameter was set to $\omega=10^{-4}$, and both the methods employed linear Gaussian (hyper)policies with variances $\sigma^{2} \in \{0.01, 0.05, 0.1, 0.5, 1\}$. Further details on the setting employed for this set of experiments are presented in Tables~\ref{tab:exp-cost-swimmer} and~\ref{tab:exp-cost-hopper}.
In Figures~\ref{fig:variance_study_cpgae_swimmer} and~\ref{fig:variance_study_cpgae_hopper} it is possible to note that, as the stochasticity parameter $\sigma$ grows, the distance of $\textcolor{vibrantBlue}{J_{\text{A}, 1}}(\vtheta_K)$ and $\textcolor{vibrantRed}{J_{\text{P}, 1}}(\vtheta_K)$ from ${J_{\text{D}, 1}}(\vtheta_K)$ increases, while the distance between $\textcolor{vibrantBlue}{J_{\text{A}, 0}}(\vtheta_K)$ and $\textcolor{vibrantRed}{J_{\text{P}, 0}}(\vtheta_K)$ from ${J_{\text{D}, 0}}(\vtheta_K)$ shows the same straightforward behavior only in \AB exploration---in \PB exploration, this is respected only when the learned policy has meaningful performance values. Furthermore, the impact of different exploration paradigms on the learning curves can be observed. In \cpgae, where noise is injected at each time step, the variance is significantly higher compared to \cpgpe, where noise is sampled only once at the beginning of each trajectory. This distinction results in more stable learning dynamics for \cpgpe. However, \cpgpe also exhibits sensitivity to the magnitude of the injected noise, which can negatively affect its learning capabilities when the noise level is too high.
Finally, we highlight that empirically there exists an \quotes{optimal} value for the stochasticity $\sigma$ leading to a parameterization $\vtheta_{K}$ resulting in a deterministic policy maximizing the performance while staying below the cost threshold. 

\begin{figure}[t!]
    \centering
    \subfloat[Performance Comparison on \cpgae.]{
        \includegraphics[width=0.38\textwidth]{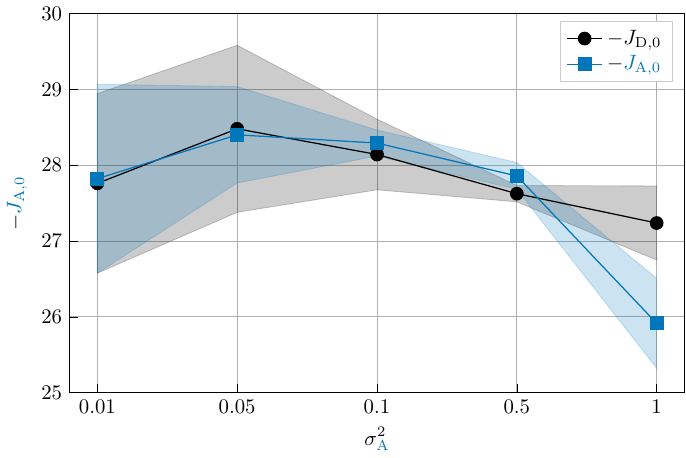}
    }
    \hspace{0.5cm}
    \subfloat[Cost Comparison on \cpgae.]{
        \includegraphics[width=0.38\textwidth]{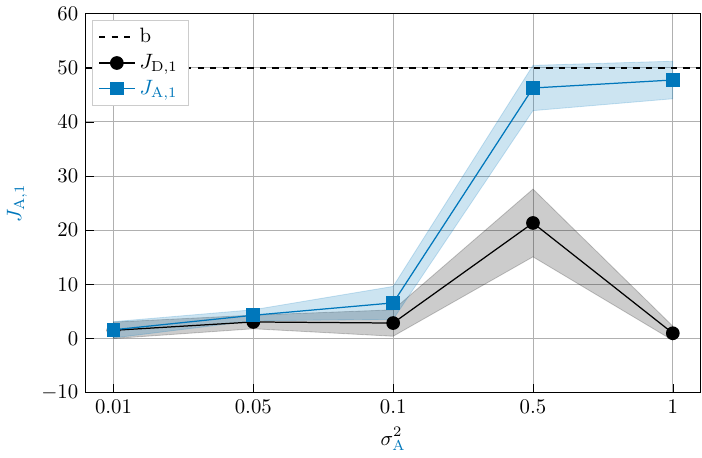}
    } 

    \vspace{0.5cm}
    
    \subfloat[Performance Comparison on \cpgpe.]{
        \includegraphics[width=0.38\textwidth]{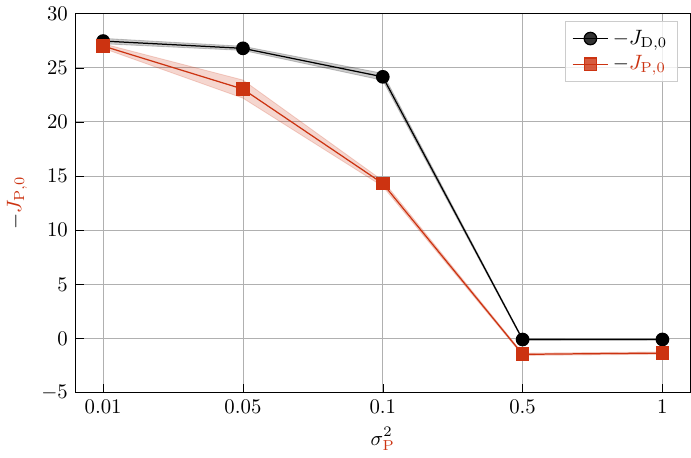}
    }
    \hspace{0.5cm}
    \subfloat[Cost Comparison on \cpgpe.]{
        \includegraphics[width=0.38\textwidth]{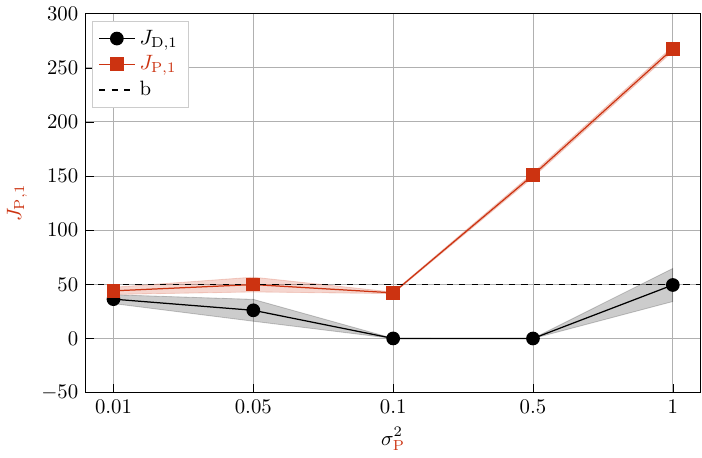}
    }

    \caption{Deterministic Deployment Study on \emph{CostSwimmer-v4} (5 runs, mean $\pm$ $95\% $ C.I.).}
    \label{fig:variance_study_cpgae_swimmer}
\end{figure}
\begin{figure}[t!]
    \centering
    \subfloat[Performance Comparison on \cpgae.]{
        \includegraphics[width=0.38\textwidth]{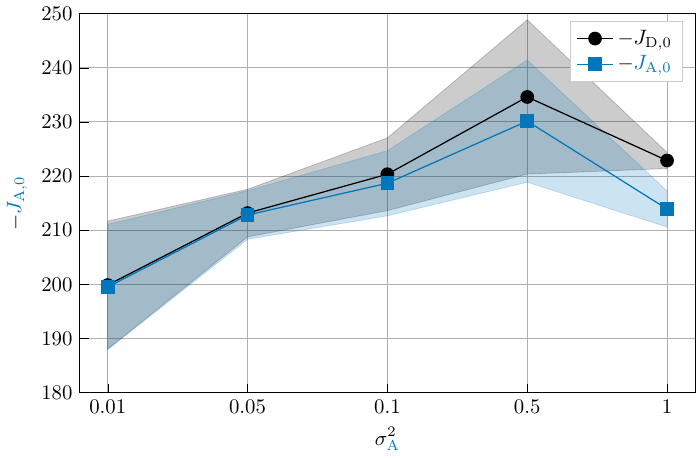}
    }
    \hspace{0.5cm}
    \subfloat[Cost Comparison on \cpgae.]{
        \includegraphics[width=0.38\textwidth]{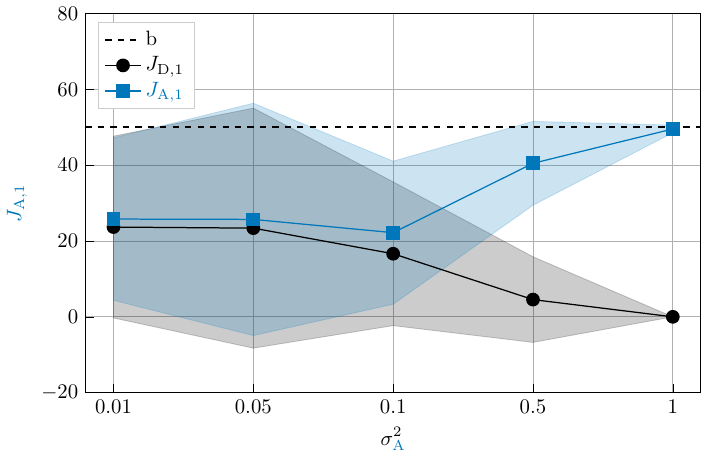}
    } 

    \vspace{0.5cm}
    
    \subfloat[Performance Comparison on \cpgpe.]{
        \includegraphics[width=0.38\textwidth]{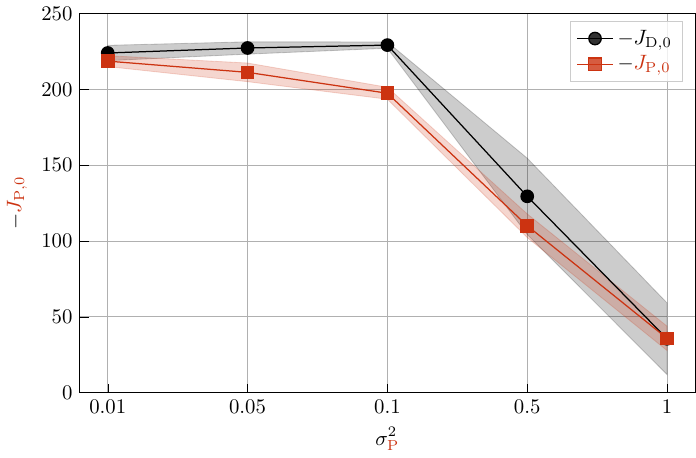}
    }
    \hspace{0.5cm}
    \subfloat[Cost Comparison on \cpgpe.]{
        \includegraphics[width=0.38\textwidth]{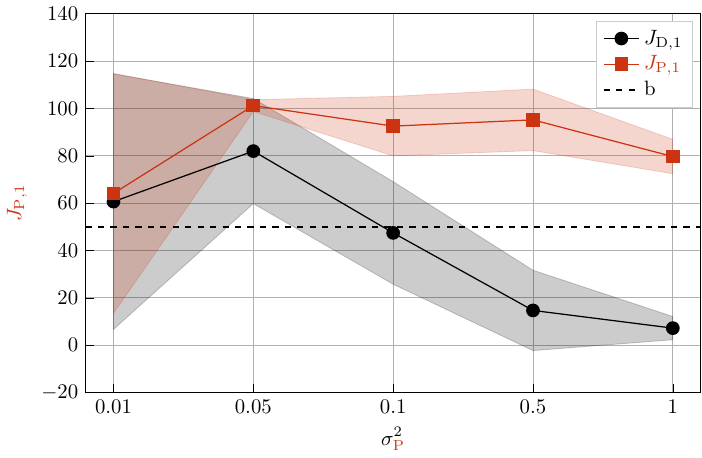}
    }

    \caption{Deterministic Deployment Study on \emph{CostHopper-v4} (5 runs, mean $\pm$ $95\% $ C.I.).}
    \label{fig:variance_study_cpgae_hopper}
\end{figure}

\subsection{Regularization Sensitivity Study}
\label{subsec:sensitivity}

In this last experiment, we study the sensitivity of \cpgae and \cpgpe \wrt the regularization term $\omega$.
We tested the algorithms on a bidimensional \emph{CostLQR} environment (see Appendix~\ref{apx:exp} for details). For the environment at hand, we considered a horizon $T=50$.
We run both algorithms for $K=10000$ iterations, with a batch size of $N=100$ trajectories per iteration, and with a varying regularization term such that $\omega \in \{0, 10^{-4}, 10^{-2}\}$.
We considered a single constraint on the average trajectory cost, for which we set a threshold $b=0.2$.
For the step size schedules, we employed Adam~\citep{kingma2015adam} with initial rates $\zeta_{\vtheta,0} = 10^{-3}$ and $\zeta_{\vlambda,0} = 10^{-2}$.
Moreover, in this specific experiment, \cpgae employed a linear gaussian policy with variance $\sigma_{\text{A}}^2 = 10^{-3}$. On the other hand, \cpgpe employed a linear Gaussian hyperpolicy with variance $\sigma_{\text{P}}^2 = 10^{-3}$ over a linear deterministic policy.
The experimental setting is summarized in Table~\ref{tab:exp-reg-lqr}.
Figures~\ref{fig:sensitiviy_cpgpe} and~\ref{fig:sensitiviy_cpgae} and show the Lagrangian curves, the performance ones (\ie the one associated with the objective function), and the cost-related ones. 
From the shown curves it is possible to notice that, for both \cpgae and \cpgpe, a higher regularization ($\omega=10^{-2}$) corresponds to a higher bias \wrt the constraint satisfaction. 
This bias is compliant with what is shown by Theorem~\ref{thr:conversion}. Indeed, the higher the regularization, the stricter the constraint threshold should be made. These considerations also hold for both \dcpgpe and \dcpgae. In particular, we highlight that, while for \PB exploration the stochastic and deterministic curves are almost the same, for \AB exploration the deterministic curve related to the cost is always under the cost threshold. 
Finally, we report in Figure~\ref{fig:reg_study_lam} the evolution of the values of the Lagrangian multipliers $\vlambda$ during the learning. As expected from the theory, for both \cpgae and \cpgpe a higher regularization leads to smaller values of $\vlambda$. Moreover, we empirically notice that \cpgae reaches higher values of $\vlambda$ \wrt the ones seen by \cpgpe.

\begin{figure}[t!]
    \centering
    
     \subfloat[Lagrangian Curves (\cpgpe).]{
        \includegraphics[width=0.38\textwidth]{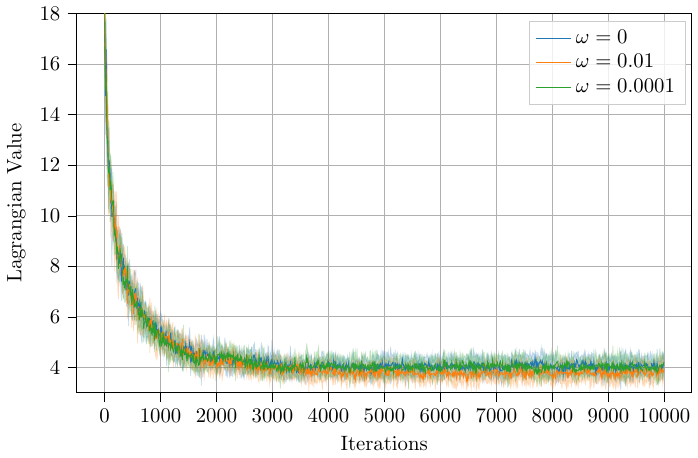}
    }
    \hspace{0.5cm}
    \subfloat[Lagrangian Curves (\dcpgpe).]{
        \includegraphics[width=0.38\textwidth]{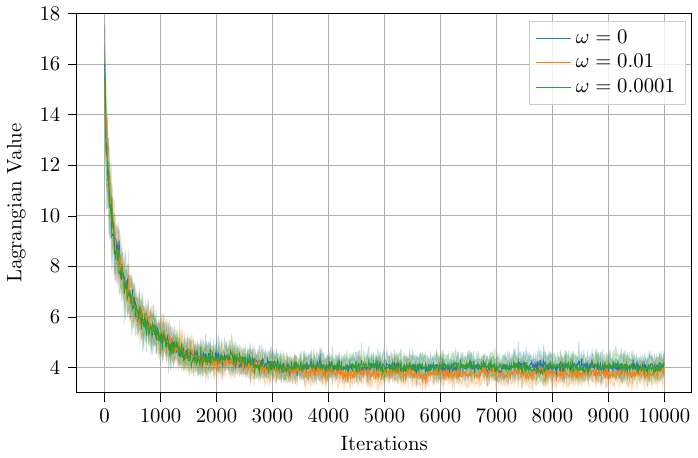}
    }

    \vspace{0.5cm}

    \subfloat[Performance Curves (\cpgpe).]{
        \includegraphics[width=0.38\textwidth]{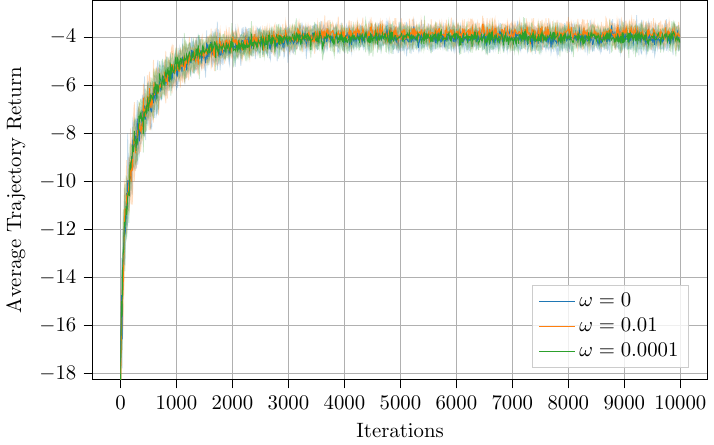}
    }
    \hspace{0.5cm}
    \subfloat[Performance Curves (\dcpgpe).]{
        \includegraphics[width=0.38\textwidth]{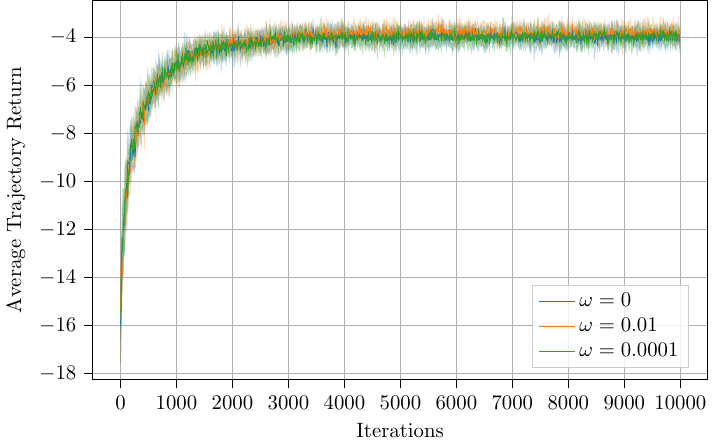}
    } 

    \vspace{0.5cm}
   
    \subfloat[Cost Curves (\cpgpe).]{
        \includegraphics[width=0.38\textwidth]{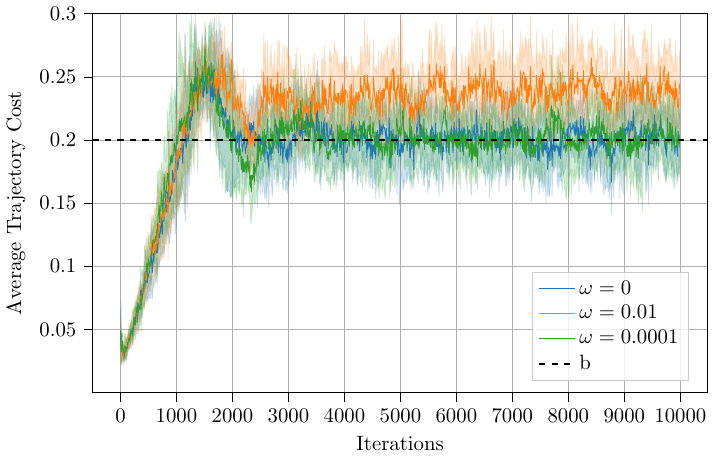}
    }
    \hspace{0.5cm}
    \subfloat[Cost Curves (\dcpgpe).]{
        \includegraphics[width=0.38\textwidth]{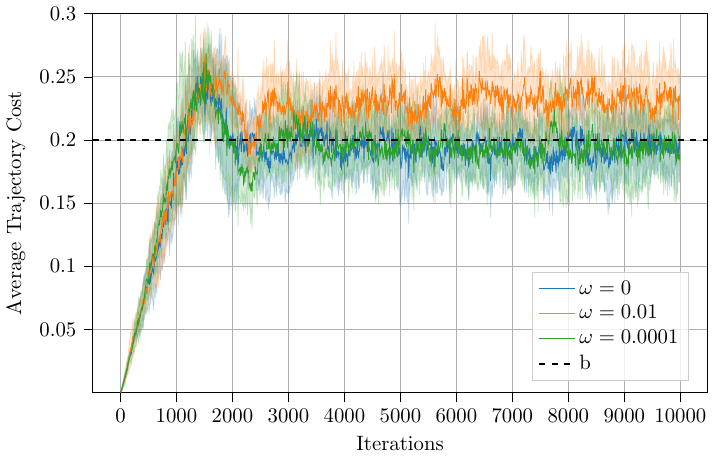}
    } 

    \caption{Sensitivity study on \emph{CostLQR} of \cpgpe with regularization values $\omega \in \{0, 10^{-2}, 10^{-4}\}$ (5 runs, mean $\pm$ $95\% $ C.I.).}
    \label{fig:sensitiviy_cpgpe}
\end{figure}

\begin{figure}[t!]
    \centering
    
     \subfloat[Lagrangian Curves (\cpgae).]{
        \includegraphics[width=0.38\textwidth]{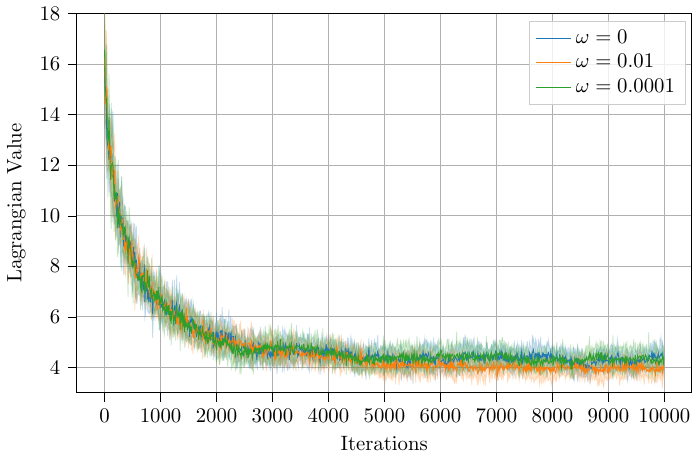}
    }
    \hspace{0.5cm}
    \subfloat[Lagrangian Curves (\dcpgae).]{
        \includegraphics[width=0.38\textwidth]{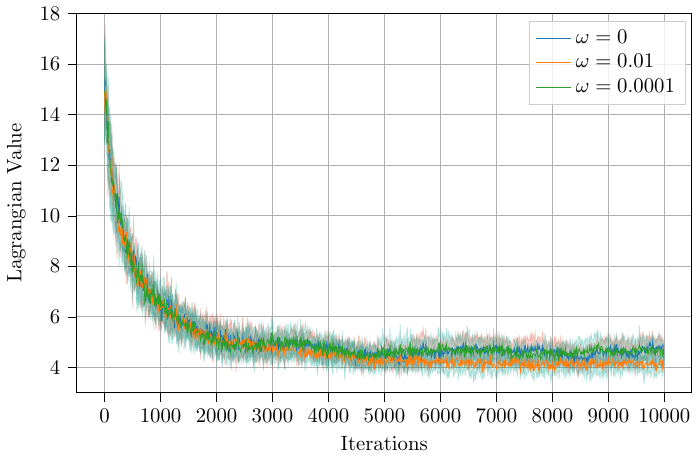}
    }

    \vspace{0.5cm}

    \subfloat[Performance Curves (\cpgae).]{
        \includegraphics[width=0.38\textwidth]{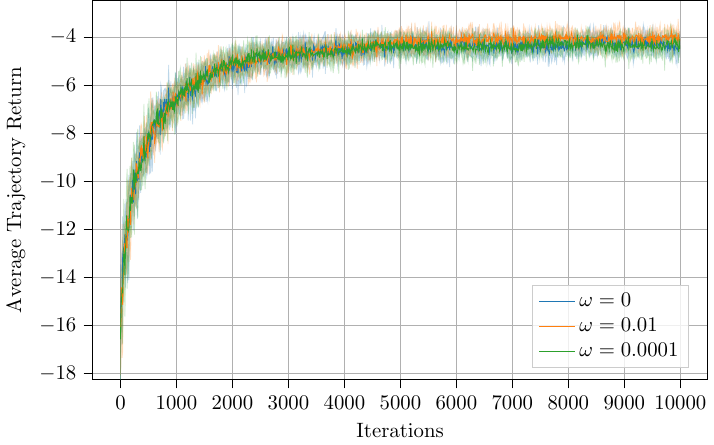}
    }
    \hspace{0.5cm}
    \subfloat[Performance Curves (\dcpgae).]{
        \includegraphics[width=0.38\textwidth]{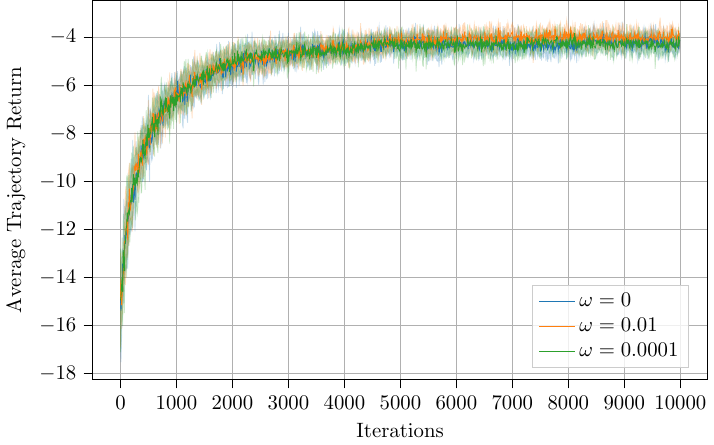}
    } 

    \vspace{0.5cm}
   
    \subfloat[Cost Curves (\cpgae).]{
        \includegraphics[width=0.38\textwidth]{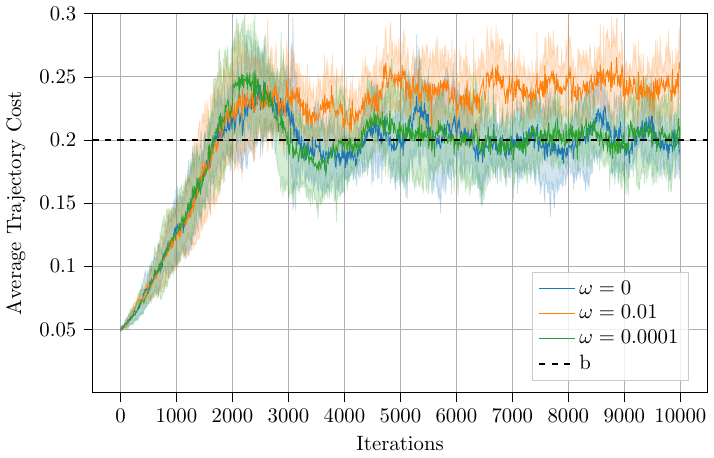}
    }
    \hspace{0.5cm}
    \subfloat[Cost Curves (\dcpgae).]{
        \includegraphics[width=0.38\textwidth]{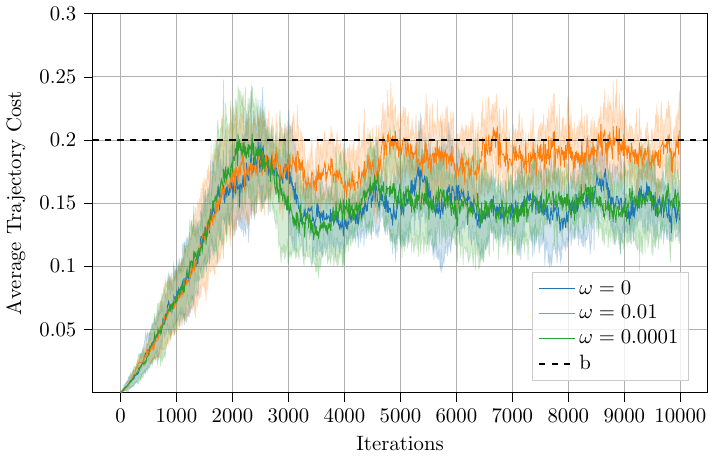}
    } 

    \caption{Sensitivity study on \emph{CostLQR} of \cpgae with regularization values $\omega \in \{0, 10^{-2}, 10^{-4}\}$ (5 runs, mean $\pm$ $95\% $ C.I.).}
    \label{fig:sensitiviy_cpgae}
\end{figure}

\begin{figure}[t!]
    \centering
    
     \subfloat[Lambda Curves (\cpgpe).]{
        \includegraphics[width=0.38\textwidth]{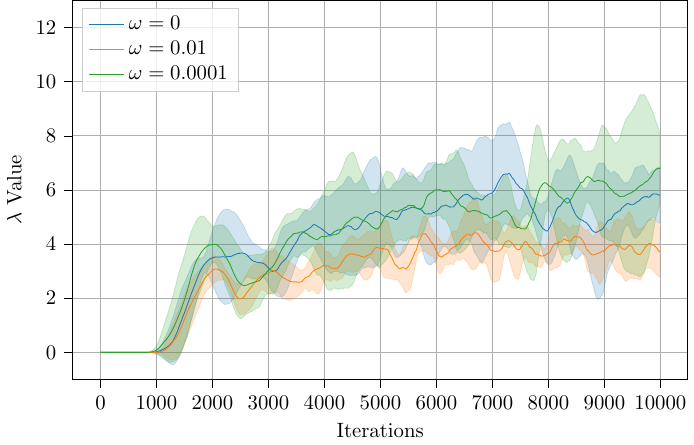}
        \label{fig:performance_swimmer_cpgae}
    }
    \hspace{0.5cm}
    \subfloat[Lambda Curves (\cpgae).]{
        \includegraphics[width=0.38\textwidth]{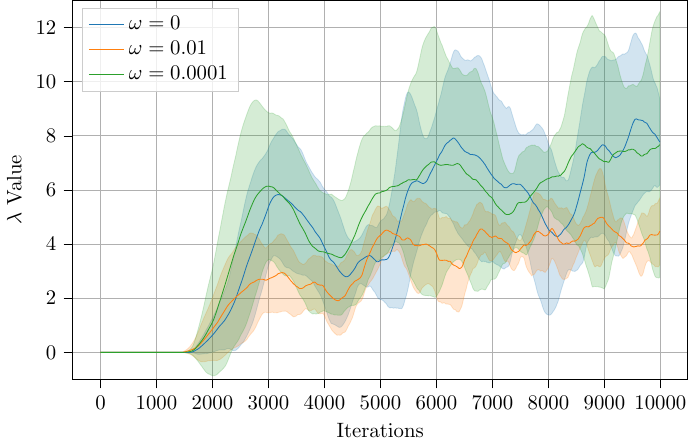}
    }

    \caption{$\lambda$ curves for \cpgpe and \cpgae on \emph{CostLQR} with \cpgae with regularization values $\omega \in \{0, 10^{-2}, 10^{-4}\}$ (5 runs, mean $\pm$ $95\%$ C.I.).}
    \label{fig:reg_study_lam}
\end{figure}

%% file: contents/08_concl.tex
\section{Conclusion} 
\label{sec:conclusions}

In this work, we proposed a general framework for addressing continuous CRL problems via \emph{primal-dual policy-based} algorithms, employing an alternating ascent–descent scheme. Our \emph{exploration-agnostic} algorithm, \cpg, provides \emph{dimension-free, global, last-iterate} convergence guarantees under the standard weak gradient domination assumption.
Furthermore, we reinterpreted both \emph{action-based} and \emph{parameter-based} exploration paradigms as white-noise perturbations applied to parametric deterministic policies, either at the action level or at the parameter one. Under this noise model, we established all the conditions required to ensure last-iterate convergence of \cpg, and we proved that \cpg converges to an optimal deterministic policy when trained via a stochastic (hyper)policy and the stochasticity is subsequently switched off at the end of the learning phase.
We validated our theoretical findings by comparing our methods against state-of-the-art baselines and demonstrating their effectiveness, particularly in the deployment of deterministic policies.
Future research should aim to improve \emph{sample complexity} of \cpg, with the goal of matching the lower bounds established by~\citet{vaswani2022near}. Another promising direction is the development of \emph{single time-scale algorithms} that retain the same convergence guarantees.
Finally, our analysis assumes a fixed level of stochasticity~$\sigma$, which must be set on the order of~$\cO(\epsilon)$ to guarantee convergence to an optimal deterministic policy under the adopted noise model. However, this assumption may be impractical in real-world scenarios, where stochasticity is often either learned or gradually annealed during training. This gap between theory and practice has recently been addressed by~\citet{montenegro2025convergence} in the unconstrained setting; extending this line of work to the constrained case remains an important direction for future research.

%% file: apx/proofs.tex
\section{Omitted Proofs} \label{apx:proofs}

\begin{lemma}[Regularization Bias on Saddle Points - 1]\label{lemma:boundLagrangeMult}
Under Assumption~\ref{asm:assunzione}, for every $\omega \ge  0$, let $(\vupsilon^*_{\omega}, \vlambda^*_{\omega})$ be a saddle point of $\cL_{\omega}$, it holds that: 
\begin{align*}
	0 \le \cL_0(\vupsilon^*_0,\vlambda^*_0)- \cL_0(\vupsilon^*_\omega,\vlambda^*_\omega) \le \frac{\omega}{2} \left( \|\vlambda^*_0\|_2^2 - \|\vlambda^*_\omega\|_2^2\right).
\end{align*}
\end{lemma}

\begin{proof}
From the fact that $(\vupsilon^*_{\omega}, \vlambda^*_{\omega})$ is a saddle point of $\cL_\omega$, we have  for every ${\vupsilon} \in \mathcal{V}$ and ${\vlambda} {\in \Lambda}$:
\begin{align}
    & \cL_{\omega}({\vupsilon}, \vlambda^*_{\omega}) \ge \cL_{\omega}(\vupsilon^*_{\omega}, \vlambda^*_{\omega}) \ge \cL_{\omega}(\vupsilon^*_{\omega}, {\vlambda}) \\
    & \iff \cL_0({\vupsilon}, \vlambda^*_{\omega}) - \frac{\omega}{2} \left\| \vlambda^*_{\omega} \right\|_2^2 \ge \cL_0(\vupsilon^*_{\omega}, \vlambda^*_{\omega}) - \frac{\omega}{2} \left\| \vlambda^*_{\omega} \right\|_2^2 \ge \cL_0(\vupsilon^*_{\omega}, {\vlambda}) - \frac{\omega}{2} \left\| {\vlambda} \right\|_2^2 \\
    & \iff \cL_0({\vupsilon}, \vlambda^*_{\omega}) \ge \cL_0(\vupsilon^*_{\omega}, \vlambda^*_{\omega}) \ge \cL_0(\vupsilon^*_{\omega}, {\vlambda}) + \frac{\omega}{2} \left( \left\| \vlambda^*_{\omega} \right\|_2^2 - \left\| {\vlambda} \right\|_2^2 \right). \label{eq:-002}
\end{align}
From the fact that $(\vupsilon^*_{0}, \vlambda^*_{0})$ is a saddle point of $\cL_0$, we have  for every ${\vupsilon} \in \mathcal{V}$ and ${\vlambda} {\in \Lambda}$:
\begin{align}\label{eq:-001}
    \cL_0({\vupsilon}, \vlambda^*_0) \ge &\cL_0(\vupsilon^*_0, \vlambda^*_0) \ge \cL_0(\vupsilon^*_0, {\vlambda}).
\end{align}
By setting $(\vupsilon,\vlambda) \leftarrow (\vupsilon^*_\omega, \vlambda^*_\omega)$ in Equation~\eqref{eq:-001} and $(\vupsilon,\vlambda) \leftarrow (\vupsilon^*_0, \vlambda^*_0)$
in Equation~\eqref{eq:-002}, we obtain:
\begin{align}
     \cL_0(\vupsilon^*_{\omega}, \vlambda^*_0) & \ge \cL_0(\vupsilon^*_0, \vlambda^*_0) \ge \cL_0(\vupsilon^*_0, \vlambda^*_{\omega})\\
    & =  \cL_0(\vupsilon^*_0, \vlambda^*_{\omega}) \ge \cL_0(\vupsilon^*_{\omega}, \vlambda^*_{\omega}) \ge \cL_0(\vupsilon^*_{\omega}, \vlambda^*_0) + \frac{\omega}{2} \left( \left\| \vlambda^*_{\omega} \right\|_2^2 - \left\| \vlambda^*_0 \right\|_2^2 \right) \\
    &\ge \cL_0(\vupsilon^*_0, \vlambda^*_0) + \frac{\omega}{2} \left( \left\| \vlambda^*_{\omega} \right\|_2^2 - \left\| \vlambda^*_0 \right\|_2^2 \right),
\end{align}
thus:
\begin{align}
    \cL_0(\vupsilon^*_0, \vlambda^*_0) \ge \cL_0(\vupsilon^*_{\omega}, \vlambda^*_{\omega}) \ge \cL_0(\vupsilon^*_0, \vlambda^*_0) + \frac{\omega}{2} \left( \left\| \vlambda^*_{\omega} \right\|_2^2 - \left\| \vlambda^*_0 \right\|_2^2 \right).
\end{align}
\end{proof}

\begin{lemma}[Regularization Bias on Saddle Points - 2] \label{lemma:regularizationBias}
Under Assumption~\ref{asm:assunzione}, for every $\omega \ge  0$, it holds that:
\begin{align*}
	0 \le \min_{\vupsilon \in \mathcal{V}} \max_{\vlambda {\in \Lambda}} \cL_{0}(\vupsilon,\vlambda) - 	\min_{\vupsilon \in \mathcal{V}} \max_{\vlambda {\in \Lambda}} \cL_{\omega}(\vupsilon,\vlambda) \le \frac{\omega}{2} \| \vlambda^*_0 \|_2^2.
\end{align*}
\end{lemma}

\begin{proof}
The first inequality follows from the observation that $\cL_{0}(\vupsilon,\vlambda) \ge \cL_{\omega}(\vupsilon,\vlambda)$ for every $\omega \ge 0$.
For the second inequality, let us denote as $(\vupsilon^*_{\omega}, \vlambda^*_{\omega})$ the saddle point for $\cL_{\omega}$
and let $\Lambda^* = \left\{ \vlambda^*_{0}, \vlambda^*_{\omega} \right\}$.
We have:
\begin{align}
    \cL_{0}(\vupsilon^*_{0}, \vlambda^*_{0}) - \cL_{\omega}(\vupsilon^*_{\omega}, \vlambda^*_{\omega}) &=   \min_{\vupsilon \in \mathcal{V}} \max_{\vlambda \in \Lambda^*} \cL_{0}(\vupsilon, \vlambda) - \min_{\vupsilon \in \mathcal{V}} \max_{\vlambda \in \Lambda^*} \cL_{\omega}(\vupsilon, \vlambda) \ \\
    &\le \max_{\vupsilon \in \mathcal{V}} \left| \max_{\vlambda \in \Lambda^*} \cL_{0}(\vupsilon, \vlambda) - \max_{\vlambda \in \Lambda^*} \cL_{\omega}(\vupsilon, \vlambda) \right| \\
    &= \max_{\vupsilon \in \mathcal{V}, \vlambda \in \Lambda^*} \left| \cL_{0}(\vupsilon, \vlambda) - \cL_{\omega}(\vupsilon, \vlambda) \right| \\
    &= \frac{\omega}{2} \max \left\{ \left\| \vlambda_{0}^* \right\|_2^2; \; \left\| \vlambda_{\omega}^* \right\|_2^2 \right\} \\
    &= \frac{\omega}{2} \left\| \vlambda^*_{0}\right\|_2^2,
\end{align}
where we used Lemma~\ref{lemma:boundLagrangeMult} to conclude that $\left\| \vlambda^*_{0}\right\|_2^2 \ge \left\| \vlambda^*_{\omega}\right\|_2^2$.
\end{proof}

\begin{lemma}[Objective bound and Constraint violation]\label{lemma:33}
    Under Assumption~\ref{asm:assunzione}, for every $\omega \ge  0$, letting $(\vupsilon^*_{\omega}, \vlambda^*_{\omega})$ be a saddle point of $\cL_{\omega}$, it holds that: 
    \begin{align}
        & 0 \le J_0(\vupsilon^*_0)  - J_0(\vupsilon^*_\omega)  \le \omega \| \vlambda^*_0\|_2^2, \\
        & \| (\mathbf{J}(\vupsilon^*_\omega)- \mathbf{b})^+\|_2 \le \omega  \| \vlambda^*_0\|_2.
    \end{align}
\end{lemma}

\begin{proof}
    Since $(\vupsilon^*_{0}, \vlambda^*_{0})$ is a saddle point of $\cL_0$, it holds that $\vupsilon^*_{0}$ is feasible and, consequently, $\cL_0(\vupsilon^*_{0}, \vlambda^*_{0})=J_0(\vupsilon^*_0)$. Moreover, let $\omega > 0$: since $(\vupsilon^*_{\omega}, \vlambda^*_{\omega})$ is a saddle point of $\cL_\omega$ it holds that $\vlambda^*_\omega = \vlambda^*(\vupsilon^*_\omega) = \Pi_{\Lambda}\left(\frac{1}{\omega} (\mathbf{J}(\vupsilon^*_\omega)- \mathbf{b})  \right) = \frac{1}{\omega} (\mathbf{J}(\vupsilon^*_\omega)- \mathbf{b})^+$, \hll{since $\frac{1}{\omega} \|(\mathbf{J}(\vupsilon^*_\omega)- \mathbf{b})^+\|_2 \le \omega^{-1}\sqrt{U} J_{\max}$.} Thus, we have:
    \begin{align}
        \cL_0(\vupsilon^*_{\omega}, \vlambda^*_{\omega}) = J_0(\vupsilon^*_\omega) + \langle\vlambda^*_\omega, \mathbf{J}(\vupsilon^*_\omega)- \mathbf{b}\rangle = J_0(\vupsilon^*_\omega) + \frac{1}{\omega} \| (\mathbf{J}(\vupsilon^*_\omega)- \mathbf{b})^+\|_2^2.
    \end{align}
    From Lemma~\ref{lemma:boundLagrangeMult}, we have:
    \begin{align}
        0 \le J_0(\vupsilon^*_0)  - J_0(\vupsilon^*_\omega) - \frac{1}{\omega} \| (\mathbf{J}(\vupsilon^*_\omega)- \mathbf{b})^+\|_2^2 \le \frac{\omega}{2} \| \vlambda^*_0\|_2^2 - \frac{1}{2\omega} \| (\mathbf{J}(\vupsilon^*_\omega)- \mathbf{b})^+\|_2^2.
    \end{align}
    By summing $\frac{1}{\omega} \| (\mathbf{J}(\vupsilon^*_\omega)- \mathbf{b})^+\|_2^2 $ to all members, we have:
    \begin{align}
        \frac{1}{\omega} \| (\mathbf{J}(\vupsilon^*_\omega)- \mathbf{b})^+\|_2^2 \le J_0(\vupsilon^*_0)  - J_0(\vupsilon^*_\omega)  \le \frac{\omega}{2} \| \vlambda^*_0\|_2^2 + \frac{1}{2\omega} \| (\mathbf{J}(\vupsilon^*_\omega)- \mathbf{b})^+\|_2^2.\label{c3.1}
    \end{align}
    Now taking the first and last member, we conclude:
    \begin{align}
        \| (\mathbf{J}(\vupsilon^*_\omega)- \mathbf{b})^+\|_2^2 \le \omega^2  \| \vlambda^*_0\|_2^2 .
    \end{align}
    Since $\frac{1}{\omega} \| (\mathbf{J}(\vupsilon^*_\omega)- \mathbf{b})^+\|_2^2\ge 0$ and plugging the latter inequality into the third member of~\eqref{c3.1} we obtain:
    \begin{align}
        0 \le J_0(\vupsilon^*_0)  - J_0(\vupsilon^*_\omega)  \le \omega \| \vlambda^*_0\|_2^2.
    \end{align}
\end{proof}

\begin{lemma}[Weak $\psi$-Gradient Domination on $H_\omega(\vupsilon)$]\label{lemma:wgdH}
Under Assumption~\ref{asm:wgd}, if $\omega > 0$, for every $\vupsilon \in \mathcal{V}$ and $\vlambda {\in \Lambda}$, it holds that:
    \begin{align}\label{eq:wgdH}
        \left\| \nabla_{\vupsilon} H_{\omega}(\vupsilon) \right\|_2^{\psi} \ge \alpha_{1} \left( H_{\omega}(\vupsilon) - \min_{\vupsilon' \in \mathcal{V}} H_{\omega}(\vupsilon') \right) - \beta_{1}.
    \end{align}
\end{lemma}

\begin{proof}
	If $\omega > 0$, the dual variable exist finite since the maximization problem over $\vlambda$ is concave:
	\begin{align*}
	\vlambda^*(\vupsilon) = \argmax_{\vlambda {\in \Lambda}} \cL_{\omega}(\vupsilon, \vlambda).
	\end{align*}
	Thus, we have from Lemma~\ref{lemma:smoothH} that$\nabla_{\vupsilon} H_\omega(\vupsilon) = \nabla_{\vupsilon} \cL_{\omega}(\vupsilon, \vlambda)\rvert_{\vlambda=\vlambda^*(\vupsilon)}$ and by Assumption~\ref{asm:wgd} we have the following:
    \begin{align}
        \left\| \nabla_{\vupsilon} H_\omega(\vupsilon) \right\|_2 & = \left\| \nabla_{\vupsilon} \cL_{\omega}(\vupsilon, \vlambda)\rvert_{\vlambda=\vlambda^*(\vupsilon)} \right\|_2 \\
        & \ge \alpha_{1} \left( \cL_{\omega}(\vupsilon, \vlambda^*(\vupsilon)) - \min_{{\vupsilon'\in\mathcal{V}}} \cL_{\omega}({\vupsilon'}, \vlambda^*(\vupsilon)) \right) - \beta_{1} \\
        & \ge \alpha_{1} \left( H_{\omega}(\vupsilon) - \min_{{\vupsilon'\in\mathcal{V}}} \max_{\vlambda {\in \Lambda}} \cL_{\omega}({\vupsilon'}, \vlambda) \right) - \beta_{1} \\
        & =  \alpha_{1} \left( H_{\omega}(\vupsilon) - H^*_\omega  \right) - \beta_{1}.
    \end{align}
\end{proof}

\begin{lemma}\label{lemma:propLLambda}
Let $\omega> 0$ and  $\vupsilon \in \mathcal{V}$. The following statements hold:
\begin{itemize}
\item $\cL_\omega(\vupsilon,\cdot)$ is $\omega$-smooth, i.e., for every  $\vlambda,\vlambda' {\in \Lambda}$ it holds that:
\begin{align*}
\left| \nabla_{\vlambda} \cL_{\omega}(\vupsilon,\vlambda') -  \nabla_{\vlambda} \cL_{\omega}(\vupsilon,\vlambda) \right| \le \omega  \left\| \vlambda - \vlambda' \right\|_2^2 
\end{align*}
\item  $\cL_\omega(\vupsilon,\cdot)$ satisfies the PL condition, i.e., for every  $\vlambda {\in \Lambda}$ it holds that:
\begin{align*}
	\|\nabla_{\vlambda} \cL_{\omega} (\vupsilon, \vlambda) \|_2^2 \ge \omega \left( \max_{\vlambda' {\in \Lambda}} \cL_{\omega} (\vupsilon,\vlambda') - \cL_{\omega}(\vupsilon,\vlambda)  \right). 
\end{align*}
\item $\cL_\omega(\vupsilon,\cdot)$ satisfies the \emph{error bound} (EB) condition, i.e., for every  $\vlambda,\vlambda' {\in \Lambda}$ it holds that:
	\begin{align*}
		\| \nabla_{\vlambda} \cL_{\omega}(\vupsilon,\vlambda) \| \ge \frac{\omega}{2} \|\vlambda^*(\vupsilon) -  \vlambda\|_2,
	\end{align*}
	where $\vlambda^*(\vupsilon) = \argmax_{\vlambda \in \Lambda} \cL_\omega(\vupsilon, \vlambda)$.
 \item $\cL_\omega(\vupsilon,\cdot)$ satisfies the \emph{quadratic growth} (QG) condition, i.e., for every  $\vlambda,\vlambda' {\in \Lambda}$ it holds that:
	\begin{align*}
		H_{\omega}(\vupsilon)  - \cL_{\omega}(\vupsilon,\vlambda)   \ge \frac{\omega}{4} \|\vlambda^*(\vupsilon) -  \vlambda\|_2,
	\end{align*}
	where $\vlambda^*(\vupsilon) = \argmax_{\vlambda \in \Lambda} \cL_\omega(\vupsilon, \vlambda)$.
\end{itemize} 
\end{lemma}

\begin{proof}
	For the first property, it is enough to observe that $\cL_\omega$ is twice differentiable in $\vlambda$ and that its Hessian is $\omega \mathbf{I}$. For the second property, we observe that $\cL_\omega$ is quadratic in $\vlambda$ and, consequently it satisfies the PL condition with parameter $\omega$:
\begin{align*}
	\|\nabla_{\vlambda} \cL_{\omega} (\vupsilon, \vlambda) \|_2^2 \ge \omega \left( \max_{\vlambda' \in \mathbb{R}^U} \cL_{\omega} (\vupsilon,\vlambda') - \cL_{\omega}(\vupsilon,\vlambda)  \right) \ge  \omega \left( \max_{\vlambda' {\in \Lambda}} \cL_{\omega} (\vupsilon,\vlambda') - \cL_{\omega}(\vupsilon,\vlambda)  \right).
\end{align*}
For the third and fourth properties, we refer to Lemma A.1 of \cite{yang2020minimax}.
\end{proof}

\begin{lemma}\label{lemma:techtech}
     Let $\omega>0$. For every  $\vupsilon\in\mathcal{V}$, it holds that:
     \begin{align}
         H_{\omega}(\vupsilon)  - H_{\omega}^*  \ge \frac{\omega}{4} \|\vlambda^*(\vupsilon) -  \vlambda^*_\omega\|_2.
     \end{align}
\end{lemma}

\begin{proof}
    Let us consider the following derivation:
    \begin{align}
          H_{\omega}(\vupsilon)  - H_{\omega}^* &= H_{\omega}(\vupsilon) -   \cL_\omega(\vupsilon^*_\omega,\vlambda^*_\omega) \\
          & \ge H_{\omega}(\vupsilon) -   \cL_\omega(\vupsilon,\vlambda^*_\omega) \\
          & \ge \frac{\omega}{4} \| \vlambda^*(\vupsilon) - \vlambda^*_\omega\|_2.
    \end{align}
    having exploited the fact that, from the saddle point property, $\cL_\omega(\vupsilon^*_\omega,\vlambda^*_\omega) \le \cL_\omega(\vupsilon,\vlambda^*_\omega)$ and, then, Lemma~\ref{lemma:propLLambda}.
\end{proof}

\begin{lemma}\label{lemma:smoothH}
	Let $\omega > 0$. The following statements hold:
	\begin{itemize}
		\item $H_\omega$ is $L_H$-smooth, i.e., for every $\vupsilon,\vupsilon' \in \mathcal{V}$, it holds that:
	\begin{align*}
		\| \nabla_{\vupsilon} H_\omega(\vupsilon') - \nabla_{\vupsilon} H_\omega(\vupsilon) \|_2 \le L_H\| \vupsilon'-\vupsilon\|_2.
	\end{align*}
	where $L_H \coloneqq L_2 + \frac{L_1^2}{\omega}$.
	\item For every $\vupsilon,\vupsilon' \in \mathcal{V}$ we have $\nabla_{\vupsilon} H_\omega(\vupsilon) = \nabla_{\vupsilon} \cL_\omega(\vupsilon,\vlambda) \rvert_{\vlambda = \vlambda^*(\vupsilon)}$, where $\vlambda^*(\vupsilon) = \argmax_{\vlambda \in \Lambda} \cL_\omega(\vupsilon,\vlambda)$.
	\end{itemize}
\end{lemma}

\begin{proof}
	The first and second statements follow from Lemma A.5 of \cite{nouiehed2019solving}.
\end{proof}

\convergencePot*
\begin{proof}
The proof is subdivided into several parts. We will omit the $\omega$ subscript for notational easiness. Let us focus on a specific iteration $k \in \mathbb{N}$.

\textbf{Part I: bounding the $a_k$ term.}~~Let us start with the $a_k$ term:
\begin{align}
    H_\omega(\vupsilon_{k+1}) - H^* &\le H_\omega(\vupsilon_{k}) - H^* + \left< \vupsilon_{k+1} - \vupsilon_{k}, \; \nabla_{\vupsilon} H_\omega(\vupsilon_{k}) \right> + \frac{L_H}{2} \left\| \vupsilon_{k+1} - \vupsilon_{k} \right\|_2^2 \\
    &\le H_{\vupsilon}(\vupsilon_{k}) - H^* - \zeta_{\vupsilon, k} \left< \widehat{\nabla}_{\vupsilon} \cL_{\omega}(\vupsilon_{k}, \vlambda_{k}), \; \nabla_{\vupsilon} H_\omega(\vupsilon_{k}) \right>  \\
    & \qquad + \frac{L_H}{2} \zeta_{\vupsilon, k}^2 \left\| \widehat{\nabla}_{\vupsilon} \cL_{\omega}(\vupsilon_{k}, \vlambda_{k}) \right\|_2^2,
\end{align}
where the first line is due to the fact that the function $H_\omega$ is $L_H$-smooth (Lemma~\ref{lemma:smoothH}), the last inequality is due to the update rule of $\vupsilon$.
Now, we apply the expected value on both sides of the inequality and we use the fact that the gradient estimation is unbiased and has variance bounded by $V_{\vupsilon}$:
\begin{align}
    \E \left[H_\omega(\vupsilon_{k+1}) | \mathcal{F}_{k-1} \right] - H^* & \le H_\omega(\vupsilon_{k}) - H^* - \zeta_{\vupsilon, k} \left< \nabla_{\vupsilon} \cL_{\omega}(\vupsilon_{k}, \vlambda_{k}), \; \nabla_{\vupsilon} H_\omega(\vupsilon_{k}) \right>\\
    & \qquad  + \frac{L_H}{2} \zeta_{\vupsilon, k}^2 \E \left[\left\| \widehat{\nabla}_{\vupsilon} \cL_{\omega}(\vupsilon_{k}, \vlambda_{k}) \right\|_2^2  | \mathcal{F}_{k-1} \right],
\end{align}
where $\mathcal{F}_{k-1}$ is the filtration associated with all events realized up to interaction $k-1$.
We recall that:
\begin{align}
    \E \left[\left\| \widehat{\nabla}_{\vupsilon} \cL_{\omega}(\vupsilon_{k}, \vlambda_{k}) \right\|_2^2  | \mathcal{F}_{k-1} \right] = \Var\left[ \widehat{\nabla}_{\vupsilon} \cL_{\omega}(\vupsilon_{k}, \vlambda_{k})  | \mathcal{F}_{k-1} \right] + \left\| \nabla_{\vupsilon} \cL_{\omega}(\vupsilon_{k}, \vlambda_{k}) \right\|_2^2,
\end{align}
and that $\Var\left[ \widehat{\nabla}_{\vupsilon} \cL_{\omega}(\vupsilon_{k}, \vlambda_{k}) \right] \le V_{\vupsilon}$ by Assumption~\ref{ass:boundedVariance}. Thus, {selecting $\zeta_{\vupsilon, k} \le 1/L_H$}, we have that:
\begin{align}
    &\E \left[H_\omega(\vupsilon_{k+1})  | \mathcal{F}_{k-1} \right] - H^* \\
    &\le H_\omega(\vupsilon_{k}) - H^* - \zeta_{\vupsilon, k} \left< \nabla_{\vupsilon} \cL_{\omega}(\vupsilon_{k}, \vlambda_{k}) \; \nabla_{\vupsilon} H_\omega(\vupsilon_{k}) \right> \\
    &\quad + \frac{\zeta_{\vupsilon, k}}{2} \left\| \nabla_{\vupsilon} \cL_{\omega} (\vupsilon_{k}, \vlambda_{k}) \right\|_2^2 + \frac{L_H}{2} \zeta_{\vupsilon, k}^2 V_{\vupsilon} \\
    &= H_\omega(\vupsilon_{k}) - H^* - \zeta_{\vupsilon, k} \left< \nabla_{\vupsilon} \cL_{\omega}(\vupsilon_{k}, \vlambda_{k}) \; \nabla_{\vupsilon} H_\omega(\vupsilon_{k}) \right> \\
    &\quad + \frac{\zeta_{\vupsilon, k}}{2} \left\| \nabla_{\vupsilon} \cL_{\omega} (\vupsilon_{k}, \vlambda_{k}) \pm \nabla_{\vupsilon} H_\omega(\vupsilon_{k}) \right\|_2^2 + \frac{L_H}{2} \zeta_{\vupsilon, k}^2 V_{\vupsilon}.
\end{align}
Consider that:
\begin{align}
    &\frac{\zeta_{\vupsilon, k}}{2} \left\| \nabla_{\vupsilon} \cL_{\omega} (\vupsilon_{k}, \vlambda_{k}) - \nabla_{\vupsilon} H_\omega(\vupsilon_{k}) + \nabla_{\vupsilon} H_\omega(\vupsilon_{k}) \right\|_2^2 \\
    &= \frac{\zeta_{\vupsilon, k}}{2} \left\| \nabla_{\vupsilon} \cL_{\omega} (\vupsilon_{k}, \vlambda_{k}) - \nabla_{\vupsilon} H_\omega(\vupsilon_{k}) \right\|_2^2 - \frac{\zeta_{\vupsilon, k}}{2} \left\| \nabla_{\vupsilon} H_\omega(\vupsilon_{k}) \right\|_2^2\\
    & \quad  + \zeta_{\vupsilon, k} \left<\nabla_{\vupsilon} \cL_{\omega} (\vupsilon_{k}, \vlambda_{k}), \; \nabla_{\vupsilon} H_\omega(\vupsilon_{k})\right>.
\end{align}
Thus, the following holds:
\begin{align}
    &\E \left[H_\omega(\vupsilon_{k+1}) | \mathcal{F}_{k-1}\right] - H^* \\
    &\le H_\omega(\vupsilon_{k}) - H^* - \zeta_{\vupsilon, k} \left< \nabla_{\vupsilon} \cL_{\omega}(\vupsilon_{k}, \vlambda_{k}), \; \nabla_{\vupsilon} H_\omega(\vupsilon_{k}) \right> + \frac{L_H}{2} \zeta_{\vupsilon, k}^2 V_{\vupsilon} \\
    &\quad + \frac{\zeta_{\vupsilon, k}}{2} \left\| \nabla_{\vupsilon} \cL_{\omega} (\vupsilon_{k}, \vlambda_{k}) - \nabla_{\vupsilon} H_\omega(\vupsilon_{k}) \right\|_2^2 - \frac{\zeta_{\vupsilon, k}}{2} \left\| \nabla_{\vupsilon} H_\omega(\vupsilon_{k}) \right\|_2^2 \\
    &\quad + \zeta_{\vupsilon, k} \left<\nabla_{\vupsilon} \cL_{\omega} (\vupsilon_{k}, \vlambda_{k}), \; \nabla_{\vupsilon} H_\omega(\vupsilon_{k}) \right> \\
    &= H_\omega(\vupsilon_{k}) - H^* - \frac{\zeta_{\vupsilon, k}}{2} \left\| \nabla_{\vupsilon} H_\omega(\vupsilon_{k}) \right\|_2^2 + \frac{\zeta_{\vupsilon, k}}{2} \left\| \nabla_{\vupsilon} \cL_{\omega} (\vupsilon_{k}, \vlambda_{k}) - \nabla_{\vupsilon} H_\omega(\vupsilon_{k}) \right\|_2^2 \\
    &\quad + \frac{L_H}{2} \zeta_{\vupsilon, k}^2 V_{\vupsilon}.
\end{align}

Thus, we have obtained:
\begin{align}
   \text{\textsf{A}} &\coloneqq \E \left[H_\omega(\vupsilon_{k+1})  | \mathcal{F}_{k-1} \right] - H^* \\
   &\le H_\omega(\vupsilon_{k}) - H^* - \frac{\zeta_{\vupsilon, k}}{2} \left\| \nabla_{\vupsilon} H_\omega(\vupsilon_{k}) \right\|_2^2 + \frac{\zeta_{\vupsilon, k}}{2} \left\| \nabla_{\vupsilon} \cL_{\omega} (\vupsilon_{k}, \vlambda_{k}) - \nabla_{\vupsilon} H_\omega(\vupsilon_{k}) \right\|_2^2 \\
   &\quad + \frac{L_H}{2} \zeta_{\vupsilon, k}^2 V_{\vupsilon},
\end{align}
holding via the selection of $\zeta_{\vupsilon, k} \le 1/L_H$.
Notice that, from \textsf{A} the following directly follows:
\begin{align}\label{eq:D}
     \text{\textsf{D}} &\coloneqq \E \left[H_\omega(\vupsilon_{k+1})  | \mathcal{F}_{k-1} \right] - H_\omega(\vupsilon_{k}) \\
     &\le - \frac{\zeta_{\vupsilon, k}}{2} \left\| \nabla_{\vupsilon} H_\omega(\vupsilon_{k}) \right\|_2^2 + \frac{\zeta_{\vupsilon, k}}{2} \left\| \nabla_{\vupsilon} \cL_{\omega} (\vupsilon_{k}, \vlambda_{k}) - \nabla_{\vupsilon} H_\omega(\vupsilon_{k}) \right\|_2^2 + \frac{L_H}{2} \zeta_{\vupsilon, k}^2 V_{\vupsilon}.\label{bound_D}
\end{align}

\textbf{Part II: bounding the $b_k$ term.}~~We are ready to analyze the $b_k$ term.
Recall that for ridge regularization of the Lagrangian function presented in the main paper, we have that $\cL_\omega$ is $\omega$-smooth and fulfills the PL condition with constant $\omega$, as shown in Lemma~\ref{lemma:propLLambda}. Since $\cL$ is a quadratic function of $\vlambda$ and $\vlambda^*(\vupsilon_{k+1}) \in \Lambda$, we have that considering the non-projected $\vlambda_{k+1}$ can only increase the distance. Thus, we will ignore projection for the rest of the proof.
We have:
\begin{align}
    &H_\omega(\vupsilon_{k+1}) - \cL_{\omega}(\vupsilon_{k+1}, \vlambda_{k+1}) \\
    &\le H_\omega(\vupsilon_{k+1}) - \cL_{\omega}(\vupsilon_{k+1}, \vlambda_{k}) - \left< \vlambda_{k+1} - \vlambda_{k}, \; \nabla_{\vlambda} \cL_{\omega}(\vupsilon_{k+1}, \vlambda_{k}) \right> + \frac{\omega}{2} \left\| \vlambda_{k+1} - \vlambda_{k} \right\|_2^2 \\
    &= H_\omega(\vupsilon_{k+1}) - \cL_{\omega}(\vupsilon_{k+1}, \vlambda_{k}) - \zeta_{\vlambda, k} \left< \widehat{\nabla}_{\vlambda} \cL_{\omega}(\vupsilon_{k+1}, \vlambda_{k}) , \; \nabla_{\vlambda} \cL_{\omega}(\vupsilon_{k+1}, \vlambda_{k}) \right> \\
    & \qquad + \frac{\omega}{2} \zeta_{\vlambda, k}^2 \left\| \widehat{\nabla}_{\vlambda} \cL_{\omega}(\vupsilon_{k+1}, \vlambda_{k}) \right\|_2^2,
\end{align}
that is possible under Assumption~\ref{asm:L_grad_lip} (\ie $\cL_{\omega}$ is $L_2$-smooth) and due to the update rules we are considering.
Now, by applying the expectation on both sides,  we obtain the following:
\begin{align}
    \E& \left[ H_\omega(\vupsilon_{k+1}) - \cL_{\omega}(\vupsilon_{k+1}, \vlambda_{k+1}) | \mathcal{F}_{k-1} \right] \\
    &\le \E \left[ H_\omega(\vupsilon_{k+1}) - \cL_{\omega}(\vupsilon_{k+1}, \vlambda_{k})  | \mathcal{F}_{k-1}\right] - \zeta_{\vlambda, k} \left\| \nabla_{\vlambda} \cL_{\omega}(\vupsilon_{k+1}, \vlambda_{k}) \right\|_2^2 \\
    & \qquad + \frac{\omega}{2} \zeta_{\vlambda, k}^2 \E \left[ \left\| \widehat{\nabla}_{\vlambda} \cL_{\omega}(\vupsilon_{k+1}, \vlambda_{k})   \right\|_2^2| \mathcal{F}_{k-1} \right] \\
    &\le \E \left[ H_\omega(\vupsilon_{k+1}) - \cL_{\omega}(\vupsilon_{k+1}, \vlambda_{k})  | \mathcal{F}_{k-1}\right] - \zeta_{\vlambda, k} \left\| \nabla_{\vlambda} \cL_{\omega}(\vupsilon_{k+1}, \vlambda_{k}) \right\|_2^2 + \frac{\omega}{2} \zeta_{\vlambda, k}^2 V_{\vlambda} \\
    & \qquad + \frac{\omega}{2} \zeta_{\vlambda, k}^2 \left\| \widehat{\nabla}_{\vlambda} \cL_{\omega}(\vupsilon_{k+1}, \vlambda) \right\|_2^2 \\
    &\le \E \left[ H_\omega(\vupsilon_{k+1}) - \cL_{\omega}(\vupsilon_{k+1}, \vlambda_{k})  | \mathcal{F}_{k-1}\right] - \frac{\zeta_{\vlambda, k}}{2} \left\| \nabla_{\vlambda} \cL_{\omega}(\vupsilon_{k+1}, \vlambda_{k}) \right\|_2^2 + \frac{\omega}{2} \zeta_{\vlambda, k}^2 V_{\vlambda},
\end{align}
where the last line follows by selecting $\zeta_{\vlambda, k} \le 1/\omega$. Since $\cL_{\omega}$ enjoys the PL condition w.r.t $\vlambda$ with constant $\omega$, for every pair $(\vupsilon, \vlambda)$ we have:
\begin{align}
    \left\| \nabla_{\vlambda} \cL_{\omega}(\vupsilon, \vlambda) \right\|_2^2 \ge \omega \left( \max_{\overline{\vlambda} \in \mathbb{R}^U} \cL_{\omega}(\vupsilon, \overline{\vlambda}) - \cL_{\omega}(\vupsilon, \vlambda) \right) \ge \omega \left( \max_{\overline{\vlambda}  \in \Lambda} \cL_{\omega}(\vupsilon, \overline{\vlambda}) - \cL_{\omega}(\vupsilon, \vlambda) \right).
\end{align}

By applying the PL condition:
\begin{align}
    &\E \left[ H_\omega(\vupsilon_{k+1}) - \cL_{\omega}(\vupsilon_{k+1}, \vlambda_{k+1}) | \mathcal{F}_{k-1} \right] \\
    &\le \E \left[ H_\omega(\vupsilon_{k+1}) - \cL_{\omega}(\vupsilon_{k+1}, \vlambda_{k}) | \mathcal{F}_{k-1} \right] - \frac{\zeta_{\vlambda, k}}{2} \left\| \nabla_{\vlambda} \cL_{\omega}(\vupsilon_{k+1}, \vlambda_{k}) \right\|_2^2 + \frac{\omega}{2} \zeta_{\vlambda, k}^2 V_{\vlambda} \\
    &\le \E \left[ H_\omega(\vupsilon_{k+1}) - \cL_{\omega}(\vupsilon_{k+1}, \vlambda_{k}) | \mathcal{F}_{k-1} \right] - \frac{\zeta_{\vlambda, k}}{2} \omega \E\left[ H_\omega(\vupsilon_{k+1}) - \cL_{\omega}(\vupsilon_{k+1}, \vlambda_{k}) | \mathcal{F}_{k-1}\right] \\
    & \qquad + \frac{\omega}{2} \zeta_{\vlambda, k}^2 V_{\vlambda} \\
    &= \left( 1 - \frac{\zeta_{\vlambda, k}}{2} \omega \right) \E \left[ H_\omega(\vupsilon_{k+1}) - \cL_{\omega}(\vupsilon_{k+1}, \vlambda_{k})| \mathcal{F}_{k-1} \right] + \frac{\omega}{2} \zeta_{\vlambda, k}^2 V_{\vlambda},
\end{align}
where we enforce $1 - \frac{\zeta_{\vlambda, k}}{2} \omega  \ge 0$, i.e., $\zeta_{\vlambda, k} \le 2/\omega$. However, we do not have a proper recursive term, thus consider the following:
\begin{align}
    H_\omega(\vupsilon_{k+1}) - \cL_{\omega}(\vupsilon_{k+1}, \vlambda_{k}) & = \underbrace{H_\omega(\vupsilon_{k}) - \cL_{\omega}(\vupsilon_{k}, \vlambda_{k})}_{\text{Recursive Term}} + \underbrace{\cL_{\omega}(\vupsilon_{k}, \vlambda_{k}) - \cL_{\omega}(\vupsilon_{k+1}, \vlambda_{k})}_{\text{\textsf{C}}} \\
    & \qquad + \underbrace{H_\omega(\vupsilon_{k+1}) - H_\omega(\vupsilon_{k})}_{\text{\textsf{D}}}.
\end{align}

A bound on \textsf{D} has already been derived, so let us bound the term \textsf{C}:
\begin{align}
    \cL_{\omega}(\vupsilon_{k}, \vlambda_{k}) & - \cL_{\omega}(\vupsilon_{k+1}, \vlambda_{k}) \le - \left< \vupsilon_{k+1} - \vupsilon_{k}, \; \nabla_{\vupsilon} \cL_{\omega}(\vupsilon_{k}, \vlambda_{k})\right> + \frac{L_2}{2} \left\| \vupsilon_{k+1} - \vupsilon_{k} \right\|_2^2 \\
    &\le \zeta_{\vupsilon, k} \left< \widehat{\nabla}_{\vupsilon}\cL_{\omega}(\vupsilon_{k}, \vlambda_{k}), \; \nabla_{\vupsilon} \cL_{\omega}(\vupsilon_{k}, \vlambda_{k}) \right> + \frac{L_2}{2} \zeta_{\vupsilon, k}^2 \left\| \widehat{\nabla}_{\vupsilon}\cL_{\omega}(\vupsilon_{k}, \vlambda_{k}) \right\|_2^2,
\end{align}
again because of Assumption~\ref{asm:L_grad_lip} and the update rule.
Now, as usual, we consider the expectation conditioned to the filtration $\mathcal{F}_{k-1}$ and the properties of the variance, to obtain:
\begin{align}
    \text{\textsf{C}} &\coloneqq \cL_{\omega}(\vupsilon_{k}, \vlambda_{k}) - \E\left[\cL_{\omega}(\vupsilon_{k+1}, \vlambda_{k}) | \mathcal{F}_{k-1} \right] \\
    &\le \zeta_{\vupsilon, k} \left(1 + \frac{L_2}{2} \zeta_{\vupsilon, k}\right) \left\| \nabla_{\vupsilon} \cL_{\omega}(\vupsilon_{k}, \vlambda_{k}) \right\|_2^2 + \frac{L_2}{2} \zeta_{\vupsilon, k}^2 V_{\vupsilon},
\end{align}
having set $\zeta_{\vupsilon, k} \le 1/L_2$.
We are finally able to conclude the bound of the term \textsf{B}:
\begin{align}
    \text{\textsf{B}} & \coloneqq \E \left[ H_\omega(\vupsilon_{k+1}) - \cL_{\omega}(\vupsilon_{k+1}, \vlambda_{k+1})  | \mathcal{F}_{k-1}  \right] \\
    &\le \left( 1 - \frac{\zeta_{\vlambda, k}}{2} \omega  \right) \E \left[ H_\omega(\vupsilon_{k+1}) - \cL_{\omega}(\vupsilon_{k+1}, \vlambda_{k})  | \mathcal{F}_{k-1} \right] + \frac{\omega}{2} \zeta_{\vlambda, k}^2 V_{\vlambda} \\
    &= \left( 1 - \frac{\zeta_{\vlambda, k}}{2} \omega  \right) \left( H_\omega(\vupsilon_{k}) - \cL_{\omega}(\vupsilon_{k}, \vlambda_{k}) \right) \\
    &\quad + \left( 1 - \frac{\zeta_{\vlambda, k}}{2} \omega  \right) \left( \cL_{\omega}(\vupsilon_{k}, \vlambda_{k}) - \E\left[\cL_{\omega}(\vupsilon_{k+1}, \vlambda_{k})  | \mathcal{F}_{k-1} 
 \right] \right) \\
    &\quad + \left( 1 - \frac{\zeta_{\vlambda, k}}{2} \omega  \right) \left( \E\left[H_\omega(\vupsilon_{k+1}) | \mathcal{F}_{k-1} \right] - H_\omega(\vupsilon_{k}) \right) + \frac{\omega}{2} \zeta_{\vlambda, k}^2 V_{\vlambda}.
\end{align}

Now we apply the bounds on \textsf{C} and \textsf{D} (the latter is from Eq.~\ref{bound_D}), obtaining:
\begin{align}
    &\E \left[ H_\omega(\vupsilon_{k+1}) - \cL_{\omega}(\vupsilon_{k+1}, \vlambda_{k+1})  | \mathcal{F}_{k-1} \right] \\
    &\le \left( 1 - \frac{\zeta_{\vlambda, k}}{2} \omega \right) \left( H_\omega(\vupsilon_{k}) - \cL_{\omega}(\vupsilon_{k}, \vlambda_{k}) \right) \\
    &\quad + \left( 1 - \frac{\zeta_{\vlambda, k}}{2} \omega  \right) \left(\zeta_{\vupsilon, k} \left(1 + \frac{L_2}{2} \zeta_{\vupsilon, k}\right) \left\| \nabla_{\vupsilon} \cL_{\omega}(\vupsilon_{k}, \vlambda_{k}) \right\|_2^2 + \frac{L_2}{2} \zeta_{\vupsilon, k}^2 V_{\vupsilon}\right) \\
    &\quad + \left( 1 - \frac{\zeta_{\vlambda, k}}{2} \omega  \right) \left( - \frac{\zeta_{\vupsilon, k}}{2} \left\| \nabla_{\vupsilon} H_\omega(\vupsilon_{k}) \right\|_2^2 + \frac{\zeta_{\vupsilon, k}}{2} \left\| \nabla_{\vupsilon} \cL_{\omega} (\vupsilon_{k}, \vlambda_{k}) - \nabla_{\vupsilon} H_\omega(\vupsilon_{k}) \right\|_2^2 \right. \\
    &\qquad \left.+ \frac{L_H}{2} \zeta_{\vupsilon, k}^2 V_{\vupsilon} \right) + \frac{\omega}{2} \zeta_{\vlambda, k}^2 V_{\vlambda},
\end{align}
that is the second fundamental term.

\textbf{Part III: bounding the potential function $P_k(\chi)$.}~~Before going on, we recall that so far we enforced:  $\zeta_{\vupsilon, k} \le1/L_H$ (since $L_H \ge L_2$) and $\zeta_{\vlambda, k} \le 1/\omega$, for every $t \in \dsb{K}$.
What we want to bound here is the potential function $P_{k+1}(\chi) = a_{k+1} + \chi b_{k+1}$. Using the final results of Part I and Part II:
\begin{align}
    a_{k+1}& + \chi b_{k+1} = \E\left[ H_\omega(\vupsilon_{k+1}) - H^* \right] + \chi \E\left[ H_\omega(\vupsilon_{k+1}) - \cL_{\omega}(\vupsilon_{k+1}, \vlambda_{k+1}) \right] \\
    &\le \E\left[H_\omega(\vupsilon_{k}) - H^*\right] - \frac{\zeta_{\vupsilon, k}}{2} \E\left[\left\| \nabla_{\vupsilon} H_\omega(\vupsilon_{k}) \right\|_2^2\right] \\
    & \quad + \frac{\zeta_{\vupsilon, k}}{2} \E\left[\left\| \nabla_{\vupsilon} \cL_{\omega} (\vupsilon_{k}, \vlambda_{k}) - \nabla_{\vupsilon} H_\omega(\vupsilon_{k}) \right\|_2^2\right] + \frac{L_H}{2} \zeta_{\vupsilon, k}^2 V_{\vupsilon} \\
    &\quad + \chi \left( 1 - \frac{\zeta_{\vlambda, k}}{2} \omega  \right) \E \left[ H_\omega(\vupsilon_{k}) - \cL_{\omega}(\vupsilon_{k}, \vlambda_{k}) \right] \\
    &\quad + \chi \left( 1 - \frac{\zeta_{\vlambda, k}}{2} \omega  \right) \left(\zeta_{\vupsilon, k} \left(1 + \frac{L_2}{2} \zeta_{\vupsilon, k}\right) \E\left[\left\| \nabla_{\vupsilon} \cL_{\omega}(\vupsilon_{k}, \vlambda_{k}) \right\|_2^2\right] + \frac{L_2}{2} \zeta_{\vupsilon, k}^2 V_{\vupsilon}\right) \\
    &\quad + \chi \left( 1 - \frac{\zeta_{\vlambda, k}}{2} \omega  \right) \Bigg( - \frac{\zeta_{\vupsilon, k}}{2} \E\left[\left\| \nabla_{\vupsilon} H_\omega(\vupsilon_{k}) \right\|_2^2\right] \\
    & \qquad\quad + \frac{\zeta_{\vupsilon, k}}{2} \E\left[\left\| \nabla_{\vupsilon} \cL_{\omega} (\vupsilon_{k}, \vlambda_{k}) - \nabla_{\vupsilon} H_\omega(\vupsilon_{k}) \right\|_2^2\right] + \frac{L_H}{2} \zeta_{\vupsilon, k}^2 V_{\vupsilon} \Bigg) \\
    &\quad + \chi \frac{\omega}{2} \zeta_{\vlambda, k}^2 V_{\vlambda} \\
    &= a_{k} + \chi \left( 1 - \frac{\zeta_{\vlambda, k}}{2} \omega  \right) b_{k} \\
    &\quad - \frac{\zeta_{\vupsilon, k}}{2} \left( 1 + \chi\left( 1 - \frac{\zeta_{\vlambda, k}}{2}\omega \right) \right) \E\left[\left\| \nabla_{\vupsilon} H_\omega(\vupsilon_{k}) \right\|_2^2\right] \\
    &\quad + \frac{\zeta_{\vupsilon, k}}{2} \left( 1 + \chi\left( 1 - \frac{\zeta_{\vlambda, k}}{2}\omega \right) \right) \E\left[\left\| \nabla_{\vupsilon} \cL_{\omega} (\vupsilon_{k}, \vlambda_{k}) - \nabla_{\vupsilon} H_\omega(\vupsilon_{k}) \right\|_2^2\right] \\
    &\quad + \zeta_{\vupsilon, k} \left(1 + \frac{L_2}{2} \zeta_{\vupsilon, k}\right) \chi \left( 1 - \frac{\zeta_{\vlambda, k}}{2} \omega  \right) \E \left[ \left\| \nabla_{\vupsilon} \cL_{\omega}(\vupsilon_{k}, \vlambda_{k}) \right\|_2^2 \right] \\
    &\quad + \frac{\zeta_{\vupsilon, k}^2}{2} \left( L_H  + \chi \left( 1 - \frac{\zeta_{\vlambda, k}}{2} \omega  \right) (L_H + L_2) \right) V_{\vupsilon} + \chi \frac{\omega}{2} \zeta_{\vlambda, k}^2 V_{\vlambda}.
\end{align}
Now we can re-arrange the terms by noticing that:
\begin{align}
    \left\| \nabla_{\vupsilon} \cL_{\omega}(\vupsilon_{k}, \vlambda_{k}) \right\|_2^2 &= \left\| \nabla_{\vupsilon} \cL_{\omega}(\vupsilon_{k}, \vlambda_{k}) - \nabla_{\vupsilon} H_\omega(\vupsilon_{k}) + \nabla_{\vupsilon} H_\omega(\vupsilon_{k}) \right\|_2^2\\
    &= \left\| \nabla_{\vupsilon} \cL_{\omega}(\vupsilon_{k}, \vlambda_{k}) - \nabla_{\vupsilon} H_\omega(\vupsilon_{k}) \right\|_2^2 \\
    & \quad + \left\| \nabla_{\vupsilon} H_\omega(\vupsilon_{k}) \right\|_2^2 + 2 \left< \nabla_{\vupsilon} \cL_{\omega}(\vupsilon_{k}, \vlambda_{k}) - \nabla_{\vupsilon} H_\omega(\vupsilon_{k}), \; \nabla_{\vupsilon} H_\omega(\vupsilon_{k}) \right> \\
    &\le 2\left\| \nabla_{\vupsilon} \cL_{\omega}(\vupsilon_{k}, \vlambda_{k}) - \nabla_{\vupsilon} H_\omega(\vupsilon_{k}) \right\|_2^2 + 2\left\| \nabla_{\vupsilon} H_\omega(\vupsilon_{k}) \right\|_2^2,
\end{align}
where the last inequality holds by Young's inequality. Then we can write what follows:
\begin{align}
    &a_{k+1} + \chi b_{k+1} \\
    &\le a_{k} + \chi \left( 1 - \frac{\zeta_{\vlambda, k}}{2} \omega \right) b_{k} \\
    &\quad + \left(2 \zeta_{\vupsilon, k} \left(1 + \frac{L_2}{2} \zeta_{\vupsilon, k}\right) \chi \left( 1 - \frac{\zeta_{\vlambda, k}}{2} \omega \right) \right. \\
    &\qquad \left. - \frac{\zeta_{\vupsilon, k}}{2} \left( 1 + \chi\left( 1 - \frac{\zeta_{\vlambda, k}}{2}\omega\right) \right) \right) \E\left[\left\| \nabla_{\vupsilon} H_\omega(\vupsilon_{k}) \right\|_2^2\right] \\
    &\quad + \left( 2 \zeta_{\vupsilon, k} \left(1 + \frac{L_2}{2} \zeta_{\vupsilon, k}\right) \chi \left( 1 - \frac{\zeta_{\vlambda, k}}{2} \omega \right) + \frac{\zeta_{\vupsilon, k}}{2} \left( 1 + \chi\left( 1 - \frac{\zeta_{\vlambda, k}}{2}\omega\right) \right) \right) \\
    & \qquad \cdot \E\left[\left\| \nabla_{\vupsilon} \cL_{\omega} (\vupsilon_{k}, \vlambda_{k}) - \nabla_{\vupsilon} H_\omega(\vupsilon_{k}) \right\|_2^2\right] \\
    &\quad + \frac{\zeta_{\vupsilon, k}^2}{2} \left( L_H  + \chi \left( 1 - \frac{\zeta_{\vlambda, k}}{2} \omega \right) (L_H + L_2) \right) V_{\vupsilon} + \chi \frac{\omega}{2} \zeta_{\vlambda, k}^2 V_{\vlambda}.
\end{align}

Let us now proceed to bound $\left\| \nabla_{\vupsilon} \cL_{\omega} (\vupsilon_{k}, \vlambda_{k}) - \nabla_{\vupsilon} H_\omega(\vupsilon_{k}) \right\|_2^2$. 
By Lemma~\ref{lemma:smoothH}, we have that $\nabla_{\vupsilon} H_\omega(\vupsilon) = \nabla_{\vupsilon} \cL_{\omega}(\vupsilon, \vlambda^*(\vupsilon))$ for every $\vlambda^*(\vupsilon) \in \argmax_{\overline{\vlambda} \in \Lambda} \cL_{\omega}(\vupsilon, \overline{\vlambda})$, thus we can write:
\begin{align}
    \left\| \nabla_{\vupsilon} \cL_{\omega} (\vupsilon_{k}, \vlambda_{k}) - \nabla_{\vupsilon} H_\omega(\vupsilon_{k}) \right\|_2^2 &= \left\| \nabla_{\vupsilon} \cL_{\omega} (\vupsilon_{k}, \vlambda_{k}) - \nabla_{\vupsilon} \cL_{\omega} (\vupsilon_{k}, \vlambda^*(\vupsilon_{k})) \right\|_2^2 \\
    &\le L_3^2 \left\|\vlambda^*(\vupsilon_{k}) - \vlambda_{k} \right\|_2^2,
\end{align}
since Assumption~\ref{asm:L_grad_lip} holds.

For a fixed value of $\vupsilon$, by Lemma~\ref{lemma:propLLambda} it follows that $\cL_{\omega}(\vupsilon, \cdot)$ satisfies the quadratic growth condition (since it satisfies the PL condition), for which the following holds:
\begin{align}
    \left\| \vlambda^*(\vupsilon_{k}) - \vlambda_{k} \right\|_2^2 \le \frac{4}{\omega} \left( H_\omega(\vupsilon_{k}) - \cL_{\omega}(\vupsilon_{k}, \vlambda_{k})\right),
\end{align}
and thus we have:
\begin{align}
    \left\| \nabla_{\vupsilon} \cL_{\omega} (\vupsilon_{k}, \vlambda_{k}) - \nabla_{\vupsilon} H_\omega(\vupsilon_{k}) \right\|_2^2 \le \frac{4 L_3^2}{\omega} \left( H_\omega(\vupsilon_{k}) - \cL_{\omega}(\vupsilon_{k}, \vlambda_{k})\right).
\end{align}

By applying the total expectation, it trivially follows:
\begin{align}
    \E\left[ \left\| \nabla_{\vupsilon} \cL_{\omega} (\vupsilon_{k}, \vlambda_{k}) - \nabla_{\vupsilon} H_\omega(\vupsilon_{k}) \right\|_2^2 \right] \le \frac{4 L_3^2}{\omega} \E\left[ H_\omega(\vupsilon_{k}) - \cL_{\omega}(\vupsilon_{k}, \vlambda_{k}) \right] = \frac{4 L_3^2}{\omega} b_{k}.
\end{align}
Thus, we have:
\begin{align}\label{eq:finalPoint1}
    &a_{k+1} + \chi b_{k+1} \\
    &\le a_{k} + \chi \left( 1 - \frac{\zeta_{\vlambda, k}}{2} \omega \right) b_{k} \\
    &\quad + \left(2 \zeta_{\vupsilon, k} \left(1 + \frac{L_2}{2} \zeta_{\vupsilon, k}\right) \chi \left( 1 - \frac{\zeta_{\vlambda, k}}{2} \omega \right) \right. \\
    &\qquad \left. - \frac{\zeta_{\vupsilon, k}}{2} \left( 1 + \chi\left( 1 - \frac{\zeta_{\vlambda, k}}{2}\omega\right) \right) \right) \E\left[\left\| \nabla_{\vupsilon} H_\omega(\vupsilon_{k}) \right\|_2^2\right] \\
    &\quad + \left( 2 \zeta_{\vupsilon, k} \left(1 + \frac{L_2}{2} \zeta_{\vupsilon, k}\right) \chi \left( 1 - \frac{\zeta_{\vlambda, k}}{2} \omega \right) + \frac{\zeta_{\vupsilon, k}}{2} \left( 1 + \chi\left( 1 - \frac{\zeta_{\vlambda, k}}{2}\omega\right) \right) \right)\frac{4 L_3^2}{\omega} b_{k}  \\
    &\quad + \frac{\zeta_{\vupsilon, k}^2}{2} \left( L_H  + \chi \left( 1 - \frac{\zeta_{\vlambda, k}}{2} \omega \right) (L_H + L_2) \right) V_{\vupsilon} + \chi \frac{\omega}{2} \zeta_{\vlambda, k}^2 V_{\vlambda}.
\end{align}

\textbf{Part IV: apply the $\psi$-gradient domination.}~~Now we need to bound the term $\left\| \nabla H_\omega(\vupsilon_{k}) \right\|_2^2$. We consider Assumption~\ref{asm:wgd} and we get: $ \left\| \nabla_{\vupsilon} H_\omega(\vupsilon_{k}) \right\|_2^\psi \ge \alpha_{1} \left( H_\omega(\vupsilon_{k}) - H^* \right) - \beta_{1}$. By defining $\widetilde{H}^* \coloneqq H^* +  \beta_1/\alpha_1$, we also have:
\begin{align}
    \left\| \nabla_{\vupsilon} H_\omega(\vupsilon) \right\|_2^\psi \ge \alpha_1 &\max\left\{0, \; H_\omega(\vupsilon) - \widetilde{H}^* \right\} \\
    &\implies \nonumber \\
    \left\| \nabla_{\vupsilon} H_\omega(\vupsilon) \right\|_2^2 \ge \alpha_1^{\frac{2}{\psi}} &\max\left\{0, \; H_\omega(\vupsilon) - \widetilde{H}^* \right\}^{\frac{2}{\psi}}. \nonumber
\end{align}
If we apply the total expectation on both sides of the inequality, we get:
\begin{align}
    \E \left[\left\| \nabla_{\vupsilon} H_\omega(\vupsilon) \right\|_2^2\right] &\ge \alpha_1^{\frac{2}{\psi}} \E\left[\max\left\{0, \; H_\omega(\vupsilon) - \widetilde{H}^* \right\}^{\frac{2}{\psi}}\right] \\
    &\ge \alpha_1^{\frac{2}{\psi}} \E\left[\max\left\{0, \; H_\omega(\vupsilon) - \widetilde{H}^* \right\}\right]^{\frac{2}{\psi}} \\
    &\ge \alpha_1^{\frac{2}{\psi}} \max\left\{0, \; \E \left[H_\omega(\vupsilon) - \widetilde{H}^*\right] \right\}^{\frac{2}{\psi}},
\end{align}
which is achieved by a double application of Jensen's inequality, since $z^{2/\psi}$ is convex for $\psi\in[1,2]$ and $z\ge 0$, and the maximum is convex.
Let us start from Equation~\eqref{eq:finalPoint1}:
  \begin{align}
    &a_{k+1} + \chi b_{k+1} \\
    &\le a_{k} + \chi \left( 1 - \frac{\zeta_{\vlambda, k}}{2} \omega \right) b_{k} \\
    &\quad + \underbrace{\left(2 \zeta_{\vupsilon, k} \left(1 + \frac{L_2}{2} \zeta_{\vupsilon, k}\right) \chi \left( 1 - \frac{\zeta_{\vlambda, k}}{2} \omega \right) - \frac{\zeta_{\vupsilon, k}}{2} \left( 1 + \chi\left( 1 - \frac{\zeta_{\vlambda, k}}{2}\omega\right) \right) \right)}_{\eqqcolon -C} \\ 
    &\qquad \cdot \E\left[\left\| \nabla H_\omega(\vupsilon_{k}) \right\|_2^2\right] \\
    &\quad + \left( 2 \zeta_{\vupsilon, k} \left(1 + \frac{L_2}{2} \zeta_{\vupsilon, k}\right) \chi \left( 1 - \frac{\zeta_{\vlambda, k}}{2} \omega \right) + \frac{\zeta_{\vupsilon, k}}{2} \left( 1 + \chi\left( 1 - \frac{\zeta_{\vlambda, k}}{2}\omega\right) \right) \right)\frac{4 L_3^2}{\omega} b_{k}  \\
    &\quad + \underbrace{\frac{\zeta_{\vupsilon, k}^2}{2} \left( L_H  + \chi \left( 1 - \frac{\zeta_{\vlambda, k}}{2} \omega \right) (L_H + L_2) \right) V_{\vupsilon} + \chi \frac{\omega}{2} \zeta_{\vlambda, k}^2 V_{\vlambda}}_{\eqqcolon V}.
\end{align}
We first enforce the negativity of $-C$. To this end:
\begin{align}
    &- C =  \left(2 \zeta_{\vupsilon, k} \underbrace{\left(1 + \frac{L_2}{2} \zeta_{\vupsilon, k}\right)}_{\le 3/2} \chi {\left( 1 - \frac{\zeta_{\vlambda, k}}{2} \omega \right)} - \frac{\zeta_{\vupsilon, k}}{2} \left( 1 + \chi{\left( 1 - \frac{\zeta_{\vlambda, k}}{2}\omega\right)} \right) \right)\\
    & \le \zeta_{\vupsilon, k} \left(3 \chi {\left( 1 - \frac{\zeta_{\vlambda, k}}{2} \omega \right)} - \frac{1}{2} \left( 1 + \chi{\left( 1 - \frac{\zeta_{\vlambda, k}}{2}\omega\right)} \right) \right) \\
    & \le  \frac{\zeta_{\vupsilon, k} }{2} \left( 5 \chi \underbrace{\left( 1 -  \frac{\zeta_{\vlambda, k}}{2} \omega \right)}_{\le 1} - 1 \right) \le \frac{\zeta_{\vupsilon, k} }{2} \left( 5 \chi  - 1 \right) \le 0.
\end{align}
Thus, it is enough to enforce $5\chi -1 \le 0 \implies \chi \le 1/5$.  We now plug in the gradient domination inequalities:
  \begin{align}
    &a_{k+1} + \chi b_{k+1} \\
    &\le a_{k} + \chi \left( 1 - \frac{\zeta_{\vlambda, k}}{2} \omega \right) b_{k} -C \alpha_1^{\frac{2}{\psi}} \max\left\{0; \; \E \left[H_\omega(\vupsilon) - \widetilde{H}^*\right] \right\}^{\frac{2}{\psi}} + V\\
    &\quad + \left( 2 \zeta_{\vupsilon, k} \left(1 + \frac{L_2}{2} \zeta_{\vupsilon, k}\right) \chi \left( 1 - \frac{\zeta_{\vlambda, k}}{2} \omega \right) + \frac{\zeta_{\vupsilon, k}}{2} \left( 1 + \chi\left( 1 - \frac{\zeta_{\vlambda, k}}{2}\omega\right) \right) \right)\frac{4 L_3^2}{\omega} b_{k} .
\end{align}
Now we introduce the symbol $\widetilde{a}_k \coloneqq \E \left[ H_\omega(\vupsilon_{k}) - \widetilde{H}^* \right] = a_k -\beta_1/\alpha_1$, to get:
  \begin{align}
    &\widetilde{a}_{k+1} + \chi b_{k+1} \\
    &\le \widetilde{a}_{k} -C \alpha_1^{\frac{2}{\psi}} \max\left\{0, \widetilde{a}_{k} \right\}^{\frac{2}{\psi}} + V \\
    & +\left(\chi \left( 1 - \frac{\zeta_{\vlambda, k}}{2} \omega \right)  + \left( 2 \zeta_{\vupsilon, k} \left(1 + \frac{L_2}{2} \zeta_{\vupsilon, k}\right) \chi \left( 1 - \frac{\zeta_{\vlambda, k}}{2} \omega \right) \right. \right. \\
    &\qquad \left. \left. + \frac{\zeta_{\vupsilon, k}}{2} \left( 1 + \chi\left( 1 - \frac{\zeta_{\vlambda, k}}{2}\omega\right) \right) \right)\frac{4 L_3^2}{\omega} \right) b_{k}.
\end{align}

For what follows, we call $B$ the term that is multiplying $b_{k}$:
\begin{align}
    B &\coloneqq \chi \left( 1 - \frac{\zeta_{\vlambda, k}}{2} \omega \right)  + \left( 2 \zeta_{\vupsilon, k} \left(1 + \frac{L_2}{2} \zeta_{\vupsilon, k}\right) \chi \left( 1 - \frac{\zeta_{\vlambda, k}}{2} \omega \right) \right. \\
    &\quad \left. + \frac{\zeta_{\vupsilon, k}}{2} \left( 1 + \chi\left( 1 - \frac{\zeta_{\vlambda, k}}{2}\omega\right) \right) \right)\frac{4 L_3^2}{\omega}
\end{align}
Let refer to $\widetilde{a}_{k} + \chi b_{k}$ as $\widetilde{P}_t(\chi)$ with $\chi \in (0,1)$. For the sake of clarity, we re-write our main inequality as:
\begin{align}
    \widetilde{P}_t(\chi) = \widetilde{a}_{k+1} + \chi b_{k+1} &\le \widetilde{a}_{k} + B b_{k} - C \max\left\{ 0; \; \widetilde{a}_{k} \right\}^{\frac{2}{\psi}} + V.
\end{align}
Then, from Lemma~\ref{lemma:ineqMax}, having set $a \leftarrow \widetilde{a}_k$ and $b \leftarrow \chi b_k$, we have:
\begin{align}
    \widetilde{P}_{t+1}(\chi) = \widetilde{a}_{k+1} + \chi b_{k+1} & \le \widetilde{a}_{k} + B b_{k} +  C (\chi b_k)^{\frac{2}{\psi}}   - 2^{1- \frac{2}{\psi}} C \max\left\{ 0, \; \widetilde{a}_{k} + \chi b_k \right\}^{\frac{2}{\psi}} + V.
\end{align}
By choosing $\chi$ so that $\chi b_k \le 1$, i.e., $\chi \le 1/ \max_{k \in [K]} b_k$, we have:
\begin{align}
    \widetilde{P}_{t+1}(\chi) = \widetilde{a}_{k+1} + \chi b_{k+1} & \le \widetilde{a}_{k} + B b_{k} +  C (\chi b_k)^{\frac{2}{\psi}}   - 2^{1- \frac{2}{\psi}} C \max\left\{ 0, \; \widetilde{a}_{k} + \chi b_k \right\}^{\frac{2}{\psi}} + V \\
    & \le \widetilde{a}_{k} + (B +  \chi C) b_k   - 2^{1- \frac{2}{\psi}} C \max\left\{ 0, \; \widetilde{a}_{k} + \chi b_k \right\}^{\frac{2}{\psi}} + V\\
    & = \widetilde{P}_t(B +  \chi C)  - 2^{1- \frac{2}{\psi}} C \max\left\{ 0, \; \widetilde{P}_{t}(\chi) \right\}^{\frac{2}{\psi}} + V. 
\end{align}

To unfold the recursion, we need to ensure that $B+\chi C \le \chi$, which leads to a condition relating the two learning rates:
\begin{align}
    &B + \chi C \\
    &= \chi \left( 1 - \frac{\zeta_{\vlambda, k}}{2} \omega \right)  + \left( 2 \zeta_{\vupsilon, k} \underbrace{\left(1 + \frac{L_2}{2} \zeta_{\vupsilon, k}\right)}_{\le 3/2} \chi \underbrace{\left( 1 - \frac{\zeta_{\vlambda, k}}{2} \omega \right)}_{\le 1} \right. \\
    &\qquad \left. + \frac{\zeta_{\vupsilon, k}}{2} \left( 1 + \chi\underbrace{ \left( 1 - \frac{\zeta_{\vlambda, k}}{2}\omega\right)}_{\le 1} \right) \right)\frac{4 L_3^2}{\omega} \\
    &\quad + \chi {{\left(\underbrace{-2 \zeta_{\vupsilon, k} \left(1 + \frac{L_2}{2} \zeta_{\vupsilon, k}\right) \chi \left( 1 - \frac{\zeta_{\vlambda, k}}{2} \omega \right)}_{\le 0} + \frac{\zeta_{\vupsilon, k}}{2} \left( 1 + \chi\underbrace{\left( 1 - \frac{\zeta_{\vlambda, k}}{2}\omega\right)}_{\le 1} \right) \right)} \alpha_1^{\frac{2}{\psi}}} \\
    &\le \chi - \chi \frac{\zeta_{\vlambda, k}}{2} \omega + \zeta_{\vupsilon, k} \left(\frac{2L_3^2}{\omega} \left( 1+7\chi \right) + \frac{1+\chi}{2} \alpha_2^{\frac{2}{\psi}} \right) \le \chi \\
    &\implies \zeta_{\vupsilon, k} \le \frac{\omega^2 \chi \zeta_{\vlambda, k}}{(1+\chi)\omega \alpha_1^{\frac{2}{\psi}}  + 4L^2_3(1+7\chi)},
\end{align}
where we exploited $\zeta_{\vupsilon,k} \le 1/L_2$ and $\zeta_{\vlambda, k} \le 2/\omega$. Thus, we have:
\begin{align}
    \widetilde{P}_{k+1}(\chi) \le \widetilde{P}_k(\chi)  - 2^{1- \frac{2}{\psi}} C \max\left\{ 0, \; \widetilde{P}_{k}(\chi) \right\}^{\frac{2}{\psi}} + V.
\end{align}
Collecting all conditions on the learning rates, we have:
\begin{align}
    & \zeta_{\vupsilon, k} \le \min\left\{ \frac{1}{L_H} , \frac{1}{L_2}, \frac{\omega^2 \chi \zeta_{\vlambda, k}}{(1+\chi)\omega \alpha_1^{\frac{2}{\psi}}  + 4L^2_3(1+7\chi)} \right\} , \\
    & \zeta_{\vlambda, k} \le \min \left\{  \frac{1}{\omega} , \frac{2}{\omega} \right\} = \frac{1}{\omega}.
\end{align}
As a further simplification, let us observe that:
\begin{align}
    C & = \left(-2 \zeta_{\vupsilon, k} {\underbrace{\left(1 + \frac{L_2}{2} \zeta_{\vupsilon, k}\right)}_{\le 3/2} } \chi {\left( 1 - \frac{\zeta_{\vlambda, k}}{2} \omega \right)} + \frac{\zeta_{\vupsilon, k}}{2} \left( 1 + \chi{\left( 1 - \frac{\zeta_{\vlambda, k}}{2}\omega\right)} \right) \right)\alpha_1^{\frac{2}{\psi}} \\
    & \ge \frac{\zeta_{\vupsilon, k}}{2} \left(1 + 5 \left( 1 - \frac{\zeta_{\vlambda, k}}{2} \omega \right) \chi \right) \alpha_1^{\frac{2}{\psi}}\ge  \frac{\zeta_{\vupsilon, k}\alpha_1^{\frac{2}{\psi}}}{2}.
\end{align}

\begin{align}
    V & = \frac{\zeta_{\vupsilon, k}^2}{2} \left( L_H  + \chi \left( 1 - \frac{\zeta_{\vlambda, k}}{2} \omega \right) (L_H + L_2) \right) V_{\vupsilon} + \chi \frac{\omega}{2} \zeta_{\vlambda, k}^2 V_{\vlambda} \\
    & \le \frac{\zeta_{\vupsilon, k}^2}{2}\left((1+2\chi)L_2 + \hll{(1+\chi)}\frac{L_1^2}{\omega}\right) V_{\vupsilon} + \chi  \frac{\omega}{2}\zeta_{\vlambda, k}^2 V_{\vlambda} \eqqcolon \widetilde{V}.
\end{align}
Denoting with $\widetilde{C} \eqqcolon  2^{1- \frac{1}{\psi}} \frac{\zeta_{\vupsilon, k}\alpha_1^{\frac{2}{\psi}}}{2} $, we are going to study the recurrence:
\begin{align}
    \widetilde{P}_{k+1}(\chi) \le \widetilde{P}_k(\chi)  - \widetilde{C} \max\left\{ 0, \; \widetilde{P}_{k}(\chi) \right\}^{\frac{2}{\psi}} + \widetilde{V}.
\end{align}

\textbf{Part V: Rates Computation}

\emph{Part V(a): Exact gradients}~~We consider the case $\widetilde{V}=0$. Let us start with $\psi = 2$. From Lemma~\ref{lemma:recVar0}, we have:
\begin{align}
    & \widetilde{P}_{K}(\xi) \le \left(1 - \widetilde{C} \right)^K \widetilde{P}_{0}(\xi) \le \epsilon \\
    & \implies K \le \frac{\log \frac{\widetilde{P}_{0}(\xi)}{\epsilon}}{\log \frac{1}{1 - \widetilde{C}}} \le \widetilde{C}^{-1} \log \frac{\widetilde{P}_{0}(\xi)}{\epsilon} = \frac{2\log\frac{\widetilde{P}_{0}(\xi)}{\epsilon} }{ 2^{1- \frac{1}{\psi}} \zeta_{\vrho, t}\alpha_1^{\frac{2}{\psi}}}
\end{align}
The inequality on $K$ holds under the conditions:
\begin{align}
    \widetilde{C} &\le \frac{2}{\psi \widetilde{P}_0(\chi)^{\frac{2}{\psi}-1}} \implies  \zeta_{\vupsilon, k} \le  \frac{2^{1+\frac{2}{\psi}}}{\psi \alpha_1^{\frac{2}{\psi}} \widetilde{P}_0(\chi)^{\frac{2}{\psi}-1}}, \\
    \zeta_{\vupsilon, k} &\le \min\left\{ \frac{1}{L_H} , \frac{1}{L_2}, \frac{\omega^2 \chi \zeta_{\vlambda, k}}{(1+\chi)\omega \alpha_1^{\frac{2}{\psi}}  + 4L^2_3(1+7\chi)} \right\} \\
    &= \min\left\{ \frac{1}{L_2 + \frac{L_1^2}{\omega}}, \frac{\omega^2 \chi \zeta_{\vlambda, k}}{(1+\chi)\omega \alpha_1^{\frac{2}{\psi}}  + 4L^2_3(1+7\chi)} \right\}, \\
    \zeta_{\vlambda, k} &\le  \frac{1}{\omega},
\end{align}
where the first one derives from the hypothesis of Lemma~\ref{lemma:recVar0} and the other two from the conditions on the learning rates derived in the previous parts. We set:
\begin{align*}
    \zeta_{\vlambda, k} &= \omega^{-1}, \\
    \zeta_{\vupsilon, k} &= \min\left\{ \frac{2^{1+\frac{2}{\psi}}}{\psi \alpha_1^{\frac{2}{\psi}} \widetilde{P}_0(\chi)^{\frac{2}{\psi}-1}}, \frac{1}{L_2 + \frac{L_1^2}{\omega}} , \frac{\omega \chi }{(1+\chi)\omega \alpha_1^{\frac{2}{\psi}}  + 4L^2_3(1+7\chi)} \right\} = \cO(\omega).
\end{align*}
Thus, the sample complexity becomes $K = \cO \left( \omega^{-1} \log \frac{1}{\epsilon} \right)$.

Consider now $\psi \in [1,2)$.  We have from Lemma~\ref{lemma:recVar0}:
\begin{align}
    &\widetilde{P}_{K}(\chi) \le \left( \left(\frac{2}{\psi} - 1\right) \widetilde{C} K \right)^{-\frac{\psi}{2-\psi}} \le \epsilon \\
    & \implies  K \le \frac{\psi}{2-\psi} \widetilde{C}^{-1} \epsilon^{-\frac{2}{\psi} + 1} = \frac{2\psi}{(2-\psi) { 2^{1- \frac{1}{\psi}} \zeta_{\vupsilon, k}\alpha_1^{\frac{2}{\psi}}} } \epsilon^{-\frac{2}{\psi} + 1},
\end{align}
holding under the same conditions as before. With the same choices of learning rates, we obtain the sample complexity $K = \cO \left( \omega^{-1} \epsilon^{-\frac{2}{\psi}+1} \right)$ as sample complexity.

\emph{Part V(b): Estimated gradients}~~We consider $\widetilde{V} >0$. In this case, from Lemma~\ref{lemma:recVg}, we have: 
\begin{align}
    &\widetilde{P}_{K}(\chi) \le \left( 1 - \widetilde{C}^{1-\frac{\psi}{2}} \widetilde{V}^{\frac{\psi}{2}} \right)^K \widetilde{P}_{0}(\chi)  + \left( \frac{\widetilde{V}}{\widetilde{C}} \right)^{\frac{\psi}{2}}. 
\end{align}
We enforce both terms to be smaller or equal to $\epsilon/2$. With the first one, we can evaluate the sample complexity:
\begin{align}
    & \left( 1 - \widetilde{V}^{1-\frac{\psi}{2}} \widetilde{C}^{\frac{\psi}{2}} \right)^K \widetilde{P}_{0}(\chi) \le \frac{\epsilon}{2} \\
    & \implies K \le \frac{\log \frac{2\widetilde{P}_{0}(\chi)}{\epsilon}}{\widetilde{V}^{1-\frac{\psi}{2}} \widetilde{C}^{\frac{\psi}{2}}} \\
    & \qquad \qquad = \frac{\log \frac{2\widetilde{P}_{0}(\chi)}{\epsilon}}{\left(\frac{\zeta_{\vupsilon, k}^2}{2}\left((1+2\chi)L_2 + \hll{(1+\chi)}\frac{L_1^2}{\omega}\right) V_{\vupsilon} + \chi \frac{\omega}{2}\zeta_{\vlambda, k}^2 V_{\vlambda} \right)^{1-\frac{\psi}{2}} \left(   2^{1- \frac{1}{\psi}} \frac{\zeta_{\vupsilon, k}\alpha_1^{\frac{2}{\psi}}}{2}   \right)^{\frac{\psi}{2}}}
    \label{eq:K_stochastic}
\end{align}
Regarding the second one, we have:
\begin{align}
    \left( \frac{\widetilde{V}}{\widetilde{C}} \right)^{\frac{\psi}{2}} \le \frac{\epsilon}{2} \implies   \left( \frac{\frac{\zeta_{\vupsilon, k}^2}{2}\left((1+2\chi)L_2 + \hll{(1+\chi)}\frac{L_1^2}{\omega}\right) V_{\vupsilon} + \chi \frac{\omega}{2}\zeta_{\vlambda, k}^2 V_{\vlambda} }{   2^{1- \frac{1}{\psi}} \frac{\zeta_{\vupsilon, k}\alpha_1^{\frac{2}{\psi}}}{2}} \right)^{\frac{\psi}{2}} \le \frac{\epsilon}{2}
\end{align}
By enforcing the relation between the two learning rates, we set $\zeta_{\vupsilon, k} = \cO(\omega^2 \zeta_{\vlambda, k})$. By enforcing the previous inequality, \hll{recalling that $L_2 \le \cO(\omega^{-1})$ and $V_{\vupsilon}\le \cO(\omega^{-2})$}, we obtain $\zeta_{\vlambda} = \cO(  \omega \epsilon^{2/\psi})$, from which $\zeta_{\vupsilon} = \cO(\omega^3  \epsilon^{2/\psi})$. Substituting these values into the sample complexity upper bound, we get \hll{(highlighting the terms possibly depending on $\omega$)}:
\hll{
\begin{align}
    K & \le \cO \left(  \frac{ \log \frac{1}{\epsilon}}{((L_2 + \omega^{-1} ) V_{\vupsilon} \zeta_{\vupsilon}^2 +  \omega \zeta_{\vlambda}^2 )^{1-\psi/2} \zeta_{\vupsilon}^{\psi/2}  }\right) \\
    & = \cO \left(   \frac{ \log \frac{1}{\epsilon}}{((L_2 + \omega^{-1} ) V_{\vupsilon} (\omega^3  \epsilon^{2/\psi})^2 +  \omega (\omega \epsilon^{2/\psi})^2 )^{1-\psi/2} (\omega^3 \epsilon^{2/\psi})^{\psi/2}  } \right)\\
    & \le  \cO \left(   \frac{ \log \frac{1}{\epsilon}}{\omega^{3}  \epsilon^{4/\psi - 1}} \right),
\end{align}
having bounded the sum at the denominator with the second addendum.
}
\end{proof}

\DetSC*
\begin{proof}
    Under the considered set of assumptions, we recover the results of Theorems~\ref{thr:wgd-inherited}, \ref{thr:inherited-L-reg}, and~\ref{thr:pot-loss} and of Lemma~\ref{lem:bounded-estim-var}, matching the conditions needed to establish the sample complexity exhibited in Theorem~\ref{thr:convergencePot} for ensuring that $\cP_{\dagger,K}(\chi) \le \epsilon + \frac{\beta_{\dagger}(\sigma.\psi)}{\dalpha}$, where we employed the coefficients of the inherited weak $\psi$-GD (Theorem~\ref{thr:wgd-inherited}).

    In particular, recovering \quotes{Part V: Rates Computation} of the proof of Theorem~\ref{thr:convergencePot} with $\psi \in [1,2]$ and inexact gradients, to compute the sample complexity needed to ensure last-iterate global convergence of $\cP_{\dagger,K}(\chi)$, we have to ensure the following conditions:
    \begin{enumerate}[label=(\textit{\roman*}).]
        \item $K \le \frac{\log \frac{2\widetilde{P}_{\dagger,0}(\chi)}{\epsilon}}{\left(\frac{\zeta_{\vtheta}^2}{2}\left((1+2\chi)L_2 + \hll{(1+\chi)}\frac{L_1^2}{\omega}\right) V_{\dagger,\vtheta} + \chi \frac{\omega}{2}\zeta_{\vlambda}^2 V_{\vlambda} \right)^{1-\frac{\psi}{2}} \left(2^{1- \frac{1}{\psi}} \frac{\zeta_{\vtheta}\dalpha^{\frac{2}{\psi}}}{2} \right)^{\frac{\psi}{2}}}$,
        \item $\left( \frac{\frac{\zeta_{\vtheta}^2}{2}\left((1+2\chi)L_2 + \hll{(1+\chi)}\frac{L_1^2}{\omega}\right) V_{\dagger,\vtheta} + \chi \frac{\omega}{2}\zeta_{\vlambda}^2 V_{\vlambda} }{   2^{1- \frac{1}{\psi}} \frac{\zeta_{\vtheta}\dalpha^{\frac{2}{\psi}}}{2}} \right)^{\frac{\psi}{2}} \le \frac{\epsilon}{2}$,
    \end{enumerate}
    where $L_1=\cO(1)$ and $L_2=\cO(\omega^{-1})$ are quantified in Theorem~\ref{thr:inherited-L-reg}, while $V_{\dagger,\vtheta} = \cO(\omega^{-2} \sigma^{-2})$ and $V_{\vlambda} = \cO(1)$ are quantified in Lemma~\ref{lem:bounded-estim-var}.

    Now, considering $(ii)$, by enforcing the relation between the two learning rates specified in Theorem~\ref{thr:convergencePot}, we can set $\zeta_{\vtheta} = \cO(\omega^{2} \zeta_{\vlambda})$. Exploiting the characterization of $L_1$ and $L_2$ (Theorem~\ref{thr:inherited-L-reg}) and the one of $V_{\dagger,\vtheta}$ and $V_{\vlambda}$ (Lemma~\ref{lem:bounded-estim-var}), we have the following:
    \begin{align*}
        \left( \frac{\frac{\zeta_{\vtheta}^2}{2}\left((1+2\chi)L_2 + \hll{(1+\chi)}\frac{L_1^2}{\omega}\right) V_{\dagger,\vtheta} + \chi \frac{\omega}{2}\zeta_{\vlambda}^2 V_{\vlambda} }{   2^{1- \frac{1}{\psi}} \frac{\zeta_{\vtheta}\dalpha^{\frac{2}{\psi}}}{2}} \right)^{\frac{\psi}{2}} &= \cO \left( \frac{\omega^{4} \zeta_{\vlambda}^{2} \left( \omega^{-1} + \omega^{-1} \right) \omega^{-2} \sigma^{-2} + \omega \zeta_{\vlambda}^{2}}{\omega^{2} \zeta_{\vlambda}} \right)^{\frac{\psi}{2}} \\
        &= \cO \left( \frac{\omega \zeta_{\vlambda}^{2} \sigma^{-2} + \omega \zeta_{\vlambda}^{2}}{\omega^{2} \zeta_{\vlambda}} \right)^{\frac{\psi}{2}} \\
        &= \cO \left( \zeta_{\vlambda} \omega^{-1} \sigma^{-2} \right)^{\frac{\psi}{2}}.
    \end{align*}
    Thus, enforcing $\cO \left( \zeta_{\vlambda} \omega^{-1} \sigma^{-2} \right)^{\frac{\psi}{2}} \le \frac{\epsilon}{2}$, we have that $\zeta_{\vlambda} = \cO(\omega \sigma^{2} \epsilon^{2/\psi})$, thus implying $\zeta_{\vtheta} = \cO(\omega^{3} \sigma^{2} \epsilon^{2/\psi})$.
    
    If we now consider the term $(i)$, by exploiting the form of $L_1,L_2,V_{\dagger,\vtheta}$, and $V_{\vlambda}$ and the choice of $\zeta_{\vlambda}$ and $\zeta_{\vtheta}$, we have the following:
    \begin{align*}
        K &\le \frac{\log \frac{2\widetilde{P}_{\dagger,0}(\chi)}{\epsilon}}{\left(\frac{\zeta_{\vtheta}^2}{2}\left((1+2\chi)L_2 + \hll{(1+\chi)}\frac{L_1^2}{\omega}\right) V_{\dagger,\vtheta} + \chi \frac{\omega}{2}\zeta_{\vlambda}^2 V_{\vlambda} \right)^{1-\frac{\psi}{2}} \left(2^{1- \frac{1}{\psi}} \frac{\zeta_{\vtheta}\dalpha^{\frac{2}{\psi}}}{2} \right)^{\frac{\psi}{2}}} \\
        &= \cO\left( \frac{\log \frac{1}{\epsilon}}{\left( \omega \zeta_{\vlambda}^{2} \sigma^{-2} \right)^{1-\psi/2} \left( \omega^{2} \zeta_{\vlambda} \right)^{\psi/2}} \right) \\
        &= \cO\left( \frac{\log \frac{1}{\epsilon}}{\omega^{1+\psi/2} \zeta_{\vlambda}^{2-\psi/2} \sigma^{\psi-2}} \right) \\
        &= \cO\left( \frac{\log \frac{1}{\epsilon}}{\omega^{3} \sigma^{2} \epsilon^{-1 + 4/\psi}} \right),
    \end{align*}
    where in the last line, we exploited the fact that $\zeta_{\vlambda} = \cO(\omega \sigma^{2} \epsilon^{2/\psi})$.

    This iteration complexity, which naturally translates into a sample complexity given that we can employ a constant batch size $N$, ensures that $\cP_{\dagger,K}(\chi) \le \epsilon + \frac{\beta_{\dagger}(\sigma,\psi)}{\dalpha}$. By leveraging the result of Theorem~\ref{thr:pot-loss}, we have that the same sample complexity ensures that:
    \begin{align*}
        \cP_{\text{D},K}(\chi) \le \epsilon + \frac{\beta_{\dagger}(\sigma,\psi)}{\dalpha} + 4(1+\Lambda_{\max})L_{1\dagger} \sigma \sqrt{d_{\dagger}}.
    \end{align*}
\end{proof}

\section{Technical Lemmas}

\begin{lemma}\label{lemma:ineqMax}
Let $a \in \mathbb{R}$, $b\ge 0$, and $\psi \in [1,2]$. It holds that:
\begin{align}
    \max\{0,a\}^{\frac{2}{\psi}} \ge 2^{1-\frac{2}{\psi}} \max\{0,a + b\}^{\frac{2}{\psi}} - b^{\frac{2}{\psi}} .
\end{align}
\end{lemma}

\begin{proof}
Let us consider the following derivation:
\begin{align}
    \max\{0,a\}^{\frac{2}{\psi}} & = \begin{cases}
        a^{\frac{2}{\psi}} & \text{if } a > 0\\
        0 & \text{otherwise}
    \end{cases}\\
    & \ge \begin{cases}
        2^{1-\frac{2}{\psi}} (a + b)^{\frac{2}{\psi}} -  b^{\frac{2}{\psi}} & \text{if } a > 0\\
        0 & \text{otherwise}
    \end{cases} \\
    & = \begin{cases}
        2^{1-\frac{2}{\psi}} (a + b)^{\frac{2}{\psi}} -  b^{\frac{2}{\psi}} & \text{if } a > 0\\
        0 & \text{if } - b < a  \le 0\\
        0 & \text{otherwise}
    \end{cases} \\
    & \ge \begin{cases}
        2^{1-\frac{2}{\psi}} (a + b)^{\frac{2}{\psi}} - b^{\frac{2}{\psi}} & \text{if } a > 0\\
        2^{1-\frac{2}{\psi}} (a + b)^{\frac{2}{\psi}} - b^{\frac{2}{\psi}} & \text{if } - b < a  \le 0\\
        - b^{\frac{2}{\psi}}& \text{otherwise}
    \end{cases}  \\
    & = \begin{cases}
        2^{1-\frac{2}{\psi}} (a + b)^{\frac{2}{\psi}} - b^{\frac{2}{\psi}} & \text{if } a+b>0\\
        - b^{\frac{2}{\psi}}& \text{otherwise}
    \end{cases}\\
    & = 2^{1-\frac{2}{\psi}} \max\{0,a + b\}^{\frac{2}{\psi}} - b^{\frac{2}{\psi}},  \\
\end{align}
where the first inequality follows from $(x+y)^{\frac{2}{\psi}} \le 2^{\frac{2}{\psi}-1} (x^{\frac{2}{\psi}}+y^{\frac{2}{\psi}})$ for $x,y \ge 0$, from Holder's inequality; the second inequality from observing that $2^{1-\frac{2}{\psi}} (a + b)^{\frac{2}{\psi}} - b^{\frac{2}{\psi}} \le (2^{1- \frac{2}{\psi}} - 1) b^{\frac{2}{\psi}} \le 0$ for $-b<a\le 0$. 
\end{proof}

\section{Recurrences}
In this section, we provide auxiliary results about convergence rate of a certain class of recurrences that will be employed for the convergence analysis of the proposed algorithms. Specifically, we study the recurrence:
\begin{align}
    r_{k+1} \le r_k - a \max\{0,r_k\}^{\phi} + b
\end{align}
for $a > 0$, $b \ge 0$, and $\phi \in [1,2]$. To this end, we consider the helper sequence:
\begin{align}
     \begin{cases}
        \rho_0 = r_0 \\
        \rho_{k+1} =  \rho_k - a \max\{0,\rho_k\}^{\phi} + b
    \end{cases}
\end{align}
The line of the proof follows that of \cite{montenegro2024learning}. Let us start showing that for sufficiently small $a$, the sequence $\rho_k$ upper bounds $r_k$.

\begin{lemma}\label{lem:aux_1}
    If $a \le \frac{1}{\phi \rho_k^{\phi-1}}$ for every $k \ge 0$, then, $r_k \le \rho_k$ for every $k \ge 0$.
\end{lemma}

\begin{proof}
    By induction on $k$. For $k=0$, the statement holds since $\rho_0 = r_0$. Suppose the statement holds for every $j \le k$, we prove that it holds for $k+1$:
    \begin{align}
         \rho_{k+1} & =  \rho_k - a\max\{0,\rho_k\}^\phi + b \\
         & \ge r_k - a\max\{0,r_k\}^\phi + b \\
         & \ge r_{k+1},
    \end{align}
    where the first inequality holds by the inductive hypothesis and by observing that the function $f(x) = x -  a\max\{0,x\}^\phi$ is non-decreasing in $x$ when $a \le \frac{1}{\phi \rho_k^{\phi-1}}$. Indeed, if $x < 0$, then $f(x) = x$, which is non-decreasing; if $x \ge 0$, we have $f(x) = x -  a x^\phi$, that can be proved to be non-decreasing in the interval $\left[0,(a\phi)^{-\frac{1}{\phi-1}}\right]$ simply by studying the sign of the derivative. Thus, we enforce the following requirement to ensure that $\rho_k$ falls in the non-decreasing region: 
    \begin{align}
        \rho_{k} \le (a\phi)^{-\frac{1}{\phi-1}} \implies a \le \frac{1}{\phi \rho_k^{\phi-1}}.
    \end{align}
    So does $r_k$ by the inductive hypothesis. 
\end{proof}

Thus, from now on, we study the properties of the sequence $\rho_{k}$. Let us note that, if $\rho_k$ is convergent, then it converges to the fixed-point $\overline{\rho}$ computed as follows:
\begin{align}
    \overline{\rho} =  \overline{\rho} - a\max\{0,\overline{\rho}\}^\phi + b \implies \overline{\rho}  = \left(\frac{b}{a} \right)^{\frac{1}{\phi}},
\end{align}
having retained the positive solution of the equation only, since the negative one never attains the maximum $\max\{0,\overline{\rho}\}$.
Let us now study the monotonicity properties of the sequence $\rho_k$. 

\begin{lemma}\label{lem:aux_2}
    The following statements hold:
    \begin{itemize}
        \item If $r_0 > \overline{\rho}$ and $a \le \frac{1}{\phi r_0^{\phi-1}}$, then for every $k \ge 0$ it holds that: $\overline{\rho} \le \rho_{k+1} \le \rho_k$.
        \item If $r_0 < \overline{\rho}$ and $a \le \frac{1}{\phi \overline{\rho}^{\phi-1}}$, then for every $k \ge 0$ it holds that: $\overline{\rho} \ge \rho_{k+1} \ge \rho_k$.
    \end{itemize}
\end{lemma}

\begin{proof}
    The proof is analogous to that of \cite[][Lemma F.3]{montenegro2024learning}.
\end{proof}

From now on, we focus on the case in which $r_0 \ge \overline{\rho}$, since, as we shall see later, the opposite case is irrelevant for the convergence guarantees.
We now consider two cases: $b=0$ and $b>0$.

\subsection{Analysis when $b=0$}
From the policy optimization perspective, this case corresponds to the one in which the gradients are exact (no variance). Recall that here $\overline{\rho}=0$. We have the following convergence result.

\begin{lemma}\label{lemma:recVar0}
    If $a \le \frac{1}{\phi r_0^{\phi-1}}$, $r_0 \ge 0$, and $b=0$ it holds that:
    \begin{align}
        \rho_{k+1} \le \begin{cases}
            (1-a)^{k+1} r_0 & \text{if } \phi = 1 \\
             \min\left\{r_0,((\phi-1)a(k+1))^{-\frac{1}{\phi-1}}\right\} &\text{if } \phi \in (1,2]
        \end{cases}.
    \end{align}
\end{lemma}

\begin{proof}
Since $r_0 \ge 0 = \overline{\rho}$, from Lemma~\ref{lem:aux_2}, we know that $\rho_k \ge 0$ and, thus, $\max\{0,\rho_{k}\} = \rho_k$. For $\phi = 1$, we have:
\begin{align}
    \rho_{k+1} = \rho_k - a \rho_k = (1-a)\rho_k = (1-a)^{k+1} \rho_0=(1-a)^{k+1} r_0.
\end{align}
For $\phi \in (1,2]$, we have:
\begin{align}
    \rho_{k+1} = \rho_k - a \rho_k^\phi.
\end{align}
We proceed by induction. For $k=0$, the statement hold since $\rho_0=r_0$ and $r_0 \le (\phi a)^{-\frac{1}{\psi-1}} \le ((\phi-1) a)^{-\frac{1}{\psi-1}}$ from the condition on the learning rate. Suppose the thesis holds for $j\le k$, we prove it for $k+1$. $\rho_{k+1} \le r_0$ by monotonicity, and, from the inductive hypothesis:
\begin{align}
    \rho_{k+1} = \rho_k - a \rho_k^\phi & \le (\phi a k)^{-\frac{1}{\phi-1}} - a  (\phi a k)^{-\frac{\phi}{\phi-1}} \\
    & =  \underbrace{(\phi a k)^{-\frac{1}{\phi-1}} - (\phi a (k+1))^{-\frac{1}{\phi-1}} - a  (\phi a k)^{-\frac{\phi}{\phi-1}}}_{(*)} + (\phi a (k+1))^{-\frac{1}{\phi-1}}.
\end{align}
We now prove that $(*)$ is non-positive:
\begin{align}
    (*) & = ((\phi-1) a k)^{-\frac{1}{\phi-1}} - ((\phi-1) a (k+1))^{-\frac{1}{\phi-1}} - a  ((\phi-1) a k)^{-\frac{\phi}{\phi-1}} \\
    & = ((\phi-1) a
    )^{-\frac{1}{\phi-1}} k^{-\frac{\phi}{\phi-1}} \underbrace{\left( k - (k+1) \left(\frac{k}{k+1} \right)^{\frac{\phi}{\phi-1}} \right)}_{\le \frac{1}{\phi-1}} - a^{-\frac{1}{\phi-1}} ((\phi-1) k)^{-\frac{\phi}{\phi-1}} \\
    & \le a^{-\frac{1}{\phi-1}}  k^{-\frac{\phi}{\phi-1}} (\phi-1)^{-\frac{1}{\phi-1}} \left( \frac{1}{\phi-1} - \frac{1}{\phi-1} \right) \le 0,
\end{align}
having observed that:
\begin{align}
    \sup_{k \ge 1} \left( k - (k+1) \left(\frac{k}{k+1} \right)^{\frac{\phi}{\phi-1}} \right) = \lim_{k \rightarrow + \infty} \left( k - (k+1) \left(\frac{k}{k+1} \right)^{\frac{\phi}{\phi-1}} \right) = \frac{1}{\phi-1}.
\end{align}

\end{proof}

\subsection{Analysis for $b>0$}
From the policy optimization perspective, this corresponds to the case in which the gradients are estimated, i.e., the variance is positive. In this case, we proceed considering the helper sequence:
\begin{align}
     \begin{cases}
        \eta_{0}= \rho_0 \\
        \eta_{k+1} =   \left(1- a\overline{\rho}^{\phi-1} \right) \eta_k + b & \text{if } k \ge 0
    \end{cases}.
\end{align}

We show that the sequence $\eta_{k}$ upper bounds $\rho_k$ when $\rho_0 = r_0 \ge \overline{\rho}$.
\begin{lemma}
    If $r_0> \overline{\rho}$ and $a \le \frac{1}{\phi r_0^{\phi-1}}$, then, for every $k \ge 0$, it holds that $\eta_k \ge \rho_k$.
\end{lemma}

\begin{proof}
    The proof is analogous to that of \cite[][Lemma F.4]{montenegro2024learning}.
\end{proof}

Thus, we can provide the convergence guarantee.

\begin{lemma}\label{lemma:recVg}
    If $a \le \frac{1}{\phi r_0^{\phi-1}}$, $r_0 \ge 0$, and $b>0$ it holds that:
    \begin{align}
        \eta_{k+1} \le \left(1 - b^{1-\frac{1}{\phi}}a^{\frac{1}{\phi}}  \right)^{k+1} + \left(\frac{b}{a}\right)^{\frac{1}{\phi}}.
    \end{align}
\end{lemma}

\begin{proof}
    By unrolling the recursion:
    \begin{align}
        \eta_{k+1} & =   \left(1- a\overline{\rho}^{\phi-1} \right) \eta_k + b \\
        & = \left(1 - a\overline{\rho}^{\phi-1} \right)^{k+1} r_0 + b \sum_{j=0}^k\left(1 - a\overline{\rho}^{\phi-1} \right)^j \\
        & \le \left(1 - a \overline{\rho}^{\phi-1} \right)^{k+1} r_0 + b \sum_{j=0}^{+\infty}\left(1 - a\overline{\rho}^{\phi-1} \right)^j \\
        & = \left(1 - b^{1-\frac{1}{\phi}}a^{\frac{1}{\phi}}  \right)^{k+1} + \frac{b}{a\overline{\rho}^{\phi-1}} \\
        & = \left(1 - b^{1-\frac{1}{\phi}}a^{\frac{1}{\phi}}  \right)^{k+1} + \left(\frac{b}{a}\right)^{\frac{1}{\phi}}.
    \end{align}
\end{proof}

%% file: apx/add_exp.tex
\section{Experimental Details} \label{apx:exp}

\subsection{Employed Policies and Hyperpolicies}

\paragraph{Linear Gaussian Policy.}
A \emph{linear parametric gaussian} policy $\pi_{\vtheta}: \cS \times \cA \to \Delta(\cA)$ with variance $\sigma^2$ samples the actions as $a_t \sim \mathcal{N}(\vtheta^\top \bm{s}_t, \sigma^2 I_{d_{\cS}})$, where $\bm{s}_t$ is the observed state at time $t$ and $\vtheta$ is the parameter vector.

\paragraph{Tabular Softmax Policy.}
A \emph{tabular softmax} policy $\pi_{\vtheta}: \cS \times \cA \to \Delta(\cA)$ with a constant temperature $\tau$ is such that:
\begin{align*}
    \pi_{\vtheta}(\bm{a}_{j} | \bm{s}_{i}) = \frac{\exp\left( \frac{\vtheta_{i,j}}{\tau} \right)}{\sum_{z=1}^{|\cA|} \exp\left( \frac{\vtheta_{i,z}}{\tau} \right)},
\end{align*}
where $\vtheta_{i,j}$ is the parameter associated with the $i$-th state and the $j$-th action. Notice that the total number of parameters for this kind of policy is $|\cS| |\cA|$.

\paragraph{Linear Deterministic Policy.}
A \emph{linear parametric deterministic} policy $\mu_{\vtheta}: \cS \to \cA$ samples the actions as $\bm{a}_t = \vtheta^\top \bm{s}_t$, where $\bm{s}_t$ is the observed state at time $t$ and $\vtheta$ is the parameter vector.

\paragraph{Gaussian Hyperpolicy.}
A \emph{parametric gaussian} hyperpolicy $\nu_{\vrho} \in \Delta(\Theta)$ with variance $\sigma^2$ samples the parameters $\vtheta$ for the underlying generic parametric policy $\pi_{\vtheta}$ as $\vtheta_{t} \sim \mathcal{N}(\vrho, \sigma^2 I_{\dt})$, where $\vrho$ is the parameter vector for the hyperpolicy.

\subsection{Environments}
\paragraph{Discrete Grid World with Walls.}
Discrete Grid World with Walls (DGWW) is a simple discrete environment we employed to compare \cpgae against the sample-based versions of NPG-PD~\citep[][Appendix H]{ding2020natural} and RPG-PD~\citep[][Appendix C.9]{ding2024last}.
DGWW is a grid-like bidimensional environment in which an agent can assume only integer coordinate positions and in which an agent can play four actions stating whether to go up, right, left, or down.
The goal is to reach the center of the grid performing the minimum amount of steps, begin the initial state uniformly sampled among the four vertices of the grid.
The agent is rewarded negatively and proportionally to its distance from the center, where the reward is $0$.
Around the goal state there is a \quotes{U-shaped} obstacle with an opening on the top side.
In particular, when the agent lands in a state in which the wall is present, it receives a cost of $1$, otherwise the cost signal is always equal to $0$.
In our experiments, we employed a DGWW environment of such a kind, with $|\cS| = 49$, \ie with each dimension with length equal to $7$.

\paragraph{Linear Quadratic Regulator with Costs.}
The Linear Quadratic Regulator~\citep[LQR,][]{anderson2007optimal} is a continuous environment we employed in the regularization sensitivity study of \cpgae and \cpgpe, and in the comparison among the same algorithms against the sample-based version of NPG-PD2~\cite[][Algorithm 1]{ding2022convergence} and its ridge-regularized version RPG-PD2 (not provided by the authors, but designed by us).
LQR is a dynamical system governed by the following state evolution:
\begin{align*}
    \bm{s}_{t+1} = A \bm{s}_t + B \bm{a}_{t},
\end{align*}
where $A \in \mR^{d_{\cS} \times d_{\cS}}$ and $B \in \mR^{d_{\cS} \times d_{\cA}}$.

In the standard version of the environment, the reward is computed at each step as:
\begin{align*}
    r_t = - \bm{s}_t^\top R \bm{s}_t - \bm{a}_t^\top Q \bm{a}_t,
\end{align*}
where $R \in \mR^{d_{\cS} \times d_{\cS}}$ and $Q \in \mR^{d_{\cA} \times d_{\cA}}$.

We modified this version of the LQR environment introducing costs.
In particular, in our \emph{CostLQR}, the state evolution is treated as in the original case, while the reward at step $t$ is computed as:
\begin{align*}
    r_t = - \bm{s}_t^\top R \bm{s}_t,
\end{align*}
where $R \in \mR^{d_{\cS} \times d_{\cS}}$.
Moreover, we added a cost signal $c$ which is computed as follows at every time step $t$:
\begin{align*}
    c_t = \bm{a}_t^\top Q \bm{a}_t,
\end{align*}
where $Q \in \mR^{d_{\cA} \times d_{\cA}}$.

In our experiments, we consider a \emph{CostLQR} environment whose main characteristics are reported in Table~\ref{tab:env_spec}.

Additionally, we considered a uniform initial state distribution in $[-3,3]$ and the following matrices:
\begin{align*}
    A = B = 0.9 \left[\begin{matrix} 1 & 0 \\ 0 & 1\end{matrix}\right], \qquad Q = \left[\begin{matrix} 0.9 & 0 \\ 0 & 0.1\end{matrix}\right], \qquad R = \left[\begin{matrix} 0.1 & 0 \\ 0 & 0.9\end{matrix}\right].
\end{align*}

\paragraph{MuJoCo with Costs.}
For our experiments on risk minimization, we utilized environments from the MuJoCo control suite~\citep{todorov2012mujoco}, which offers a variety of continuous control environments. To tailor these environments to our specific requirements, we introduced a cost function that represents the energy associated with the control actions. In standard MuJoCo environments, a portion of the reward is typically calculated as the cost of the control action, which is proportional to the deviation of the chosen action from predefined action bounds.
In our MuJoCo modification, at each time step, we make the environment return a cost computed as:
\begin{align*}
    \left\| \bm{a}_t - \min\left\{ \max\left\{\bm{a}_t, a_{\min}\right\}, a_{\max} \right\}\right\|_2,
\end{align*}
where $a_{\min}$ and $a_{\max}$ are, respectively, the bounds for the minimum and maximum value for each component of the action vector.
Then, the action $\min\left\{ \max\left\{\bm{a}_t, a_{\min}\right\}, a_{\max} \right\}$ is passed to the environment.
In our experiment, we consider \emph{Swimmer-v4} and \emph{Hopper-v4} MuJoCo environments, whose main features are summarized in Table~\ref{tab:env_spec}.

\paragraph{RobotWorld.}
For the comparison among \cpg, AD-PGPD~\citep{rozada2024deterministic}, and PGDual~\citep{zaho2021pgdual,brunke2021pgdual}, we employed the \emph{Robot World} environment~\citep{rozada2024deterministic}. This environment is a modification of the \emph{CostLQR} one in which both reward and cost functions are given by the following quadratic functions:
\begin{align*}
    r(\bm{s}, \bm{a}) &=  \left<G_1; \left| \bm{s} \right|\right> + \left<R_1; \left| \bm{a} \right|\right> - \frac{1}{2} \|\bm{a}\|^2_2 , \\
    c(\bm{s}, \bm{a}) &=  \left< G_2; \vs^{2} \right> + \left< R_2; \va^{2} \right>,
\end{align*}
where for every $\bm{x} \in \mR^{n}$ we define $|\bm{x}| \coloneqq (|x_{1}|,\ldots,|x_{n}|)^{\top}$ and $\bm{x}^2 \coloneqq (x_{1}^2,\ldots,x_{n}^2)^{\top}$. In the setting we considered for our experiments, we employed the following values for $G_{1}, G_{2}, R_{1}$, and $R_{2}$:

\begin{align*}
    G_1 = - (1, 1, 0.001, 0.001), \quad 
    G_2 = - (0.001, 0.001, 1, 1), \quad \text{and} \quad R_1 = R_2 = - (0.01, 0.01).
\end{align*}

Furthermore, differently from the usual LQR environment, in \emph{RobotWorld} the agent is allowed to control both velocity and acceleration of the agent.

\begin{table}[t]
    \centering
    
    \renewcommand{\arraystretch}{1.2}
    \begin{tabular}{|c||c|c|c|c|}
         \hline
         \rowcolor{gray!20}
         \textbf{Environment} & \textbf{State Dim.} $d_{\cS}$ & \textbf{Action Dim.} $d_{\cA}$ & \textbf{Action Range} $[a_{\min}, a_{\max}]$ & \textbf{State Range} $[s_{\min}, s_{\max}]$\\
         \hline
         \hline
         CostLQR & $2$ & $2$ & $(-\infty,+\infty)$ & $(-\infty,+\infty)$\\
         \hline
         Swimmer-v4 & $8$ & $2$ & $[-1,1]$ & $(-\infty,+\infty)$\\
         \hline
         Hopper-v4 & $11$ & $3$ & $[-1,1]$ & $(-\infty,+\infty)$\\
         \hline
         RobotWorld & $4$ & $2$ & $(-\infty,+\infty)$ & $(-\infty,+\infty)$\\
         \hline
    \end{tabular}
    \caption{Main features of \emph{CostLQR}, \emph{Swimmer-v4}, \emph{Hopper-v4}, and \emph{RobotWorld}.}
    \label{tab:env_spec}
\end{table}

\subsection{Experimental Details}
\begin{table}[t]
    \centering
    \renewcommand{\arraystretch}{1.2}
    \begin{tabular}{|l|l|}
        \hline
        \rowcolor{gray!20}
        \multicolumn{2}{|c|}{\bfseries Details for the Comparison in \emph{DGWW} Experiment} \\
        \hline \hline
        Environment & \emph{DGWW} \\
        Horizon & $T=100$ \\
        \hline
        Policy & Tabular Softmax \\
        \hline
        Constraint Threshold & $b=0.2$ \\
        \hline
        Iterations & $K=3000$ \\
        Batch Size & $N=10$ \\
        \hline
        Learning Rates \cpgae & $\zeta_{\vtheta} = 0.01$ and $\zeta_{\vlambda} = 0.1$ \\
        Learning Rates NPG-PD & $\zeta_{\vtheta} = 0.01$ and $\zeta_{\vlambda} = 0.1$ \\
        Learning Rates RPG-PD & $\zeta_{\vtheta} = 0.01$ and $\zeta_{\vlambda} = 0.01$ \\
        \hline
        Regularization \cpgae & $\omega = 10^{-4}$ \\
        Regularization RPG-PD & $\omega = 10^{-4}$ \\
        \hline
    \end{tabular}

    \vspace{0.2cm}

    \caption{Details for the comparison of \cpgae against NPG-PD and RPG-PD in DGWW (Section~\ref{subsec:comparison}).}
    \label{tab:exp-dgww}
\end{table}

\begin{table}[t]
    \centering
    \renewcommand{\arraystretch}{1.2}
    \begin{tabular}{|l|l|}
        \hline
        \rowcolor{gray!20}
        \multicolumn{2}{|c|}{\bfseries Details for the Comparison in \emph{CostLQR} Experiment} \\
        \hline \hline
        Environment & Bidimensional \emph{CostLQR} \\
        Horizon & $T=50$ \\
        \hline
        (Hyper)Policy & Linear Gaussian with $\sigma^{2} = 10^{-3}$ \\
        \hline
        Constraint Threshold & $b=0.9$ \\
        \hline
        Iterations (\cpgpe and \cpgae) & $K=6000$ \\
        Iterations (NPG-PD2 and RPG-PD2) & $K=1000$ \\
        Batch Size (\cpgpe and \cpgae) & $N=100$ \\
        Batch Size (NPG-PD2 and RPG-PD2) & $N=600$ \\
        \hline
        Learning Rate (Adam) \cpgpe & $\zeta_{\vrho,0} = 10^{-3}$ and $\zeta_{\vlambda,0} = 10^{-2}$ \\
        Learning Rates (Adam) \cpgae & $\zeta_{\vtheta,0} = 10^{-3}$ and $\zeta_{\vlambda,0} = 10^{-2}$ \\
        Learning Rates (Adam) NPG-PD2 & $\zeta_{\vtheta,0} = 3 \cdot 10^{-3}$ and $\zeta_{\vlambda,0} = 10^{-2}$ \\
        Learning Rates (Adam) RPG-PD2 & $\zeta_{\vtheta,0} = 3 \cdot 10^{-3}$ and $\zeta_{\vlambda,0} = 10^{-2}$ \\
        \hline
        Regularization \cpgpe & $\omega = 10^{-4}$ \\
        Regularization \cpgae & $\omega = 10^{-4}$ \\
        Regularization RPG-PD2 & $\omega = 10^{-4}$ \\
        \hline
    \end{tabular}

    \vspace{0.2cm}

    \caption{Details for the comparison of \cpgpe and \cpgae against NPG-PD2 and RPG-PD2 in a bidimensional \emph{CostLQR} (Section~\ref{subsec:comparison}).}
    \label{tab:exp-cost-lqr}
\end{table}

\begin{table}[t]
    \centering
    \renewcommand{\arraystretch}{1.2}
    \begin{tabular}{|l|l|}
        \hline
        \rowcolor{gray!20}
        \multicolumn{2}{|c|}{\bfseries Details for the Comparison in \emph{RobotWorld} Experiment} \\
        \hline \hline
        Environment & \emph{RobotWorld} \\
        Horizon & $T=100$ \\
        \hline
        Hyperpolicy (\cpgpe) & Linear Gaussian with $\sigma^{2} = 10^{-6}$ \\
        Policy (\cpgae) & Linear Gaussian with $\sigma^{2} = 5 \cdot 10^{-2}$ \\
	Policy (AD-PGPD and PGDual) & Linear Deterministic \\
        \hline
        Constraint Threshold & $b=1000$ \\
        \hline
        Iterations (\cpgpe and \cpgae) & $K=10^{3}$ \\
        Iterations (AD-PGPD and PGDual) & $K=4 \cdot 10^{4}$ \\
        Batch Size (\cpgpe and \cpgae) & $N=100$ \\
        Batch Size (AD-PGPD and PGDual) & $N=400$ \\
        \hline
        Learning Rates (Adam) \cpgpe & $\zeta_{\vrho,0} = 5 \cdot 10^{-6}$ and $\zeta_{\vlambda,0} = 5 \cdot10^{-3}$ \\
        Learning Rates (Adam) \cpgae & $\zeta_{\vtheta,0} = 5 \cdot 10^{-6}$ and $\zeta_{\vlambda,0} = 10^{-4}$ \\
        Learning Rates (Adam) AD-PGPD & $\zeta_{\vtheta,0} = 10^{-5}$ and $\zeta_{\vlambda,0} = 10^{-5}$ \\
        Learning Rates (Adam) PGDual & $\zeta_{\vtheta,0} = 10^{-4}$ and $\zeta_{\vlambda,0} = 10^{-5}$ \\
        \hline
         Regularization (\cpgpe and \cpgae)  & $\omega = 10^{-4}$  \\
        Regularization AD-PGPD & $\omega = 2 \cdot 10^{-1}$ \\
        \hline
    \end{tabular}

    \vspace{0.2cm}

    \caption{Details for the comparison of \cpgpe and \cpgae against AD-PGPD and PGDual in the RobotWorld Environment (Section~\ref{subsec:comparison}).}
    \label{tab:exp-cost-robotworld}
\end{table}

\begin{table}[t]
    \centering
    \renewcommand{\arraystretch}{1.2}
    \begin{tabular}{|l|l|}
        \hline
        \rowcolor{gray!20}
        \multicolumn{2}{|c|}{\bfseries Details for the Deterministic Deployment in \emph{CostSwimmer-v4}} \\
        \hline \hline
        Environment & \emph{CostSwimmer-v4} \\
        Horizon & $T=100$ \\
        \hline
        (Hyper)policy (\cpgpe and \cpgae) & Linear Gaussian with $\sigma^{2} \in \{10^{-2}, 5\cdot10^{-2}, 10^{-1}, 5 \cdot 10^{-1}, 1\}$ \\
       
        \hline
        Constraint Threshold & $b=50$ \\
        \hline
        Iterations (\cpgpe and \cpgae) & $K=3 \cdot 10^{3}$ \\
        Batch Size (\cpgpe and \cpgae) & $N=100$ \\
        \hline
        Learning Rates (Adam) \cpgpe & $\zeta_{\vrho,0} = 10^{-3}$ and $\zeta_{\vlambda,0} = 10^{-2}$ \\
        Learning Rates (Adam) \cpgae & $\zeta_{\vtheta,0} = 10^{-3}$ and $\zeta_{\vlambda,0} = 10^{-2}$ \\
        \hline
         Regularization (\cpgpe and \cpgae)  & $\omega = 10^{-4}$  \\
        \hline
    \end{tabular}

    \vspace{0.2cm}

    \caption{Details for the comparison of \cpgpe and \cpgae in \emph{CostSwimmer-v4} (Section~\ref{subsec:exp_deterministic}).}
    \label{tab:exp-cost-swimmer}
\end{table}

\begin{table}[t]
    \centering
    \renewcommand{\arraystretch}{1.2}
    \begin{tabular}{|l|l|}
        \hline
        \rowcolor{gray!20}
        \multicolumn{2}{|c|}{\bfseries Details for the Deterministic Deployment in \emph{CostHopper-v4}} \\
        \hline \hline
        Environment & \emph{CostHopper-v4} \\
        Horizon & $T=100$ \\
        \hline
        (Hyper)policy (\cpgpe and \cpgae) & Linear Gaussian with $\sigma^{2} \in \{10^{-2}, 5\cdot10^{-2}, 10^{-1}, 5 \cdot 10^{-1}, 1\}$ \\
       
        \hline
        Constraint Threshold & $b=50$ \\
        \hline
        Iterations (\cpgpe and \cpgae) & $K=3 \cdot 10^{3}$ \\
        Batch Size (\cpgpe and \cpgae) & $N=100$ \\
        \hline
        Learning Rates (Adam) \cpgpe & $\zeta_{\vrho,0} = 10^{-2}$ and $\zeta_{\vlambda,0} = 10^{-1}$ \\
        Learning Rates (Adam) \cpgae & $\zeta_{\vtheta,0} = 10^{-2}$ and $\zeta_{\vlambda,0} = 10^{-1}$ \\
        \hline
        Regularization (\cpgpe and \cpgae)  & $\omega = 10^{-4}$  \\
        \hline
    \end{tabular}

    \vspace{0.2cm}

    \caption{Details for the comparison of \cpgpe and \cpgae in \emph{CostHopper-v4} (Section~\ref{subsec:exp_deterministic}).}
    \label{tab:exp-cost-hopper}
\end{table}

\begin{table}[t]
    \centering
    \renewcommand{\arraystretch}{1.2}
    \begin{tabular}{|l|l|}
        \hline
        \rowcolor{gray!20}
        \multicolumn{2}{|c|}{\bfseries Details for the Regularization Sensitivity Study in \emph{CostLQR}} \\
        \hline \hline
        Environment & Bidimensional \emph{CostLQR} \\
        Horizon & $T=50$ \\
        \hline
        (Hyper)Policy & Linear Gaussian with $\sigma^{2} = 10^{-3}$ \\
        \hline
        Constraint Threshold & $b=0.2$ \\
        \hline
        Iterations & $K=10^{4}$ \\
        Batch Size & $N=100$ \\
        \hline
        Learning Rate (Adam) \cpgpe & $\zeta_{\vrho,0} = 10^{-3}$ and $\zeta_{\vlambda,0} = 10^{-2}$ \\
        Learning Rates (Adam) \cpgae & $\zeta_{\vtheta,0} = 10^{-3}$ and $\zeta_{\vlambda,0} = 10^{-2}$ \\
        \hline
        Regularization & $\omega \in \left\{0, 10^{-4}, 10^{-2}\right\} 10^{-4}$ \\
        \hline
    \end{tabular}

    \vspace{0.2cm}

    \caption{Details for the regularization sensitivity study of \cpgpe and \cpgae in a bidimensional \emph{CostLQR} (Section~\ref{subsec:sensitivity}).}
    \label{tab:exp-reg-lqr}
\end{table}

%% file: main.bbl
\begin{thebibliography}{61}
\expandafter\ifx\csname natexlab\endcsname\relax\def\natexlab#1{#1}\fi
\providecommand{\url}[1]{\texttt{#1}}
\providecommand{\href}[2]{#2}
\providecommand{\path}[1]{#1}
\providecommand{\DOIprefix}{doi:}
\providecommand{\ArXivprefix}{arXiv:}
\providecommand{\URLprefix}{URL: }
\providecommand{\Pubmedprefix}{pmid:}
\providecommand{\doi}[1]{\href{http://dx.doi.org/#1}{\path{#1}}}
\providecommand{\Pubmed}[1]{\href{pmid:#1}{\path{#1}}}
\providecommand{\bibinfo}[2]{#2}
\ifx\xfnm\relax \def\xfnm[#1]{\unskip,\space#1}\fi
\bibitem[{Achiam et~al.(2017)Achiam, Held, Tamar and Abbeel}]{achiam2017constrained}
\bibinfo{author}{Achiam, J.}, \bibinfo{author}{Held, D.}, \bibinfo{author}{Tamar, A.}, \bibinfo{author}{Abbeel, P.}, \bibinfo{year}{2017}.
\newblock \bibinfo{title}{Constrained policy optimization}, in: \bibinfo{booktitle}{Proceedings of the International Conference on Machine Learning (ICML)}, \bibinfo{publisher}{{PMLR}}. pp. \bibinfo{pages}{22--31}.
\bibitem[{Altman(1999)}]{altman1999constrained}
\bibinfo{author}{Altman, E.}, \bibinfo{year}{1999}.
\newblock \bibinfo{title}{Constrained Markov Decision Processes}.
\newblock \bibinfo{publisher}{CRC Press}.
\bibitem[{Anderson and Moore(2007)}]{anderson2007optimal}
\bibinfo{author}{Anderson, B.D.}, \bibinfo{author}{Moore, J.B.}, \bibinfo{year}{2007}.
\newblock \bibinfo{title}{Optimal control: linear quadratic methods}.
\newblock \bibinfo{publisher}{Courier Corporation}.
\bibitem[{Azizzadenesheli et~al.(2018)Azizzadenesheli, Yue and Anandkumar}]{azizzadenesheli2018policy}
\bibinfo{author}{Azizzadenesheli, K.}, \bibinfo{author}{Yue, Y.}, \bibinfo{author}{Anandkumar, A.}, \bibinfo{year}{2018}.
\newblock \bibinfo{title}{Policy gradient in partially observable environments: Approximation and convergence}.
\newblock \bibinfo{journal}{arXiv preprint arXiv:1810.07900} .
\bibitem[{Bai et~al.(2022)Bai, Bedi, Agarwal, Koppel and Aggarwal}]{bai2022achieving}
\bibinfo{author}{Bai, Q.}, \bibinfo{author}{Bedi, A.S.}, \bibinfo{author}{Agarwal, M.}, \bibinfo{author}{Koppel, A.}, \bibinfo{author}{Aggarwal, V.}, \bibinfo{year}{2022}.
\newblock \bibinfo{title}{Achieving zero constraint violation for constrained reinforcement learning via primal-dual approach}, in: \bibinfo{booktitle}{{AAAI} Conference on Artificial Intelligence}, \bibinfo{publisher}{{AAAI} Press}. pp. \bibinfo{pages}{3682--3689}.
\bibitem[{Bai et~al.(2023)Bai, Bedi and Aggarwal}]{bai2023achieving}
\bibinfo{author}{Bai, Q.}, \bibinfo{author}{Bedi, A.S.}, \bibinfo{author}{Aggarwal, V.}, \bibinfo{year}{2023}.
\newblock \bibinfo{title}{Achieving zero constraint violation for constrained reinforcement learning via conservative natural policy gradient primal-dual algorithm}, in: \bibinfo{booktitle}{{AAAI} Conference on Artificial Intelligence}, \bibinfo{publisher}{{AAAI} Press}. pp. \bibinfo{pages}{6737--6744}.
\bibitem[{Baxter and Bartlett(2001)}]{baxter2001infinite}
\bibinfo{author}{Baxter, J.}, \bibinfo{author}{Bartlett, P.L.}, \bibinfo{year}{2001}.
\newblock \bibinfo{title}{Infinite-horizon policy-gradient estimation}.
\newblock \bibinfo{journal}{Journal of Artificial Intelligence Research} \bibinfo{volume}{15}, \bibinfo{pages}{319--350}.
\bibitem[{Bertsekas(2014)}]{bertsekas2014constrained}
\bibinfo{author}{Bertsekas, D.P.}, \bibinfo{year}{2014}.
\newblock \bibinfo{title}{Constrained optimization and Lagrange multiplier methods}.
\newblock \bibinfo{publisher}{Academic press}.
\bibitem[{Bhandari and Russo(2024)}]{bhandari2024global}
\bibinfo{author}{Bhandari, J.}, \bibinfo{author}{Russo, D.}, \bibinfo{year}{2024}.
\newblock \bibinfo{title}{Global optimality guarantees for policy gradient methods}.
\newblock \bibinfo{journal}{Operations Research} .
\bibitem[{Bhatnagar and Lakshmanan(2012)}]{bhatnagar2012online}
\bibinfo{author}{Bhatnagar, S.}, \bibinfo{author}{Lakshmanan, K.}, \bibinfo{year}{2012}.
\newblock \bibinfo{title}{An online actor-critic algorithm with function approximation for constrained markov decision processes}.
\newblock \bibinfo{journal}{Journal of Optimization Theory and Applications} \bibinfo{volume}{153}, \bibinfo{pages}{688--708}.
\bibitem[{Brunke et~al.(2022)Brunke, Greeff, Hall, Yuan, Zhou, Panerati and Schoellig}]{brunke2021pgdual}
\bibinfo{author}{Brunke, L.}, \bibinfo{author}{Greeff, M.}, \bibinfo{author}{Hall, A.W.}, \bibinfo{author}{Yuan, Z.}, \bibinfo{author}{Zhou, S.}, \bibinfo{author}{Panerati, J.}, \bibinfo{author}{Schoellig, A.P.}, \bibinfo{year}{2022}.
\newblock \bibinfo{title}{Safe learning in robotics: From learning-based control to safe reinforcement learning}.
\newblock \bibinfo{journal}{Annual Review of Control, Robotics, and Autonomous Systems} \bibinfo{volume}{5}, \bibinfo{pages}{411--444}.
\bibitem[{Chow et~al.(2017)Chow, Ghavamzadeh, Janson and Pavone}]{chow2017risk}
\bibinfo{author}{Chow, Y.}, \bibinfo{author}{Ghavamzadeh, M.}, \bibinfo{author}{Janson, L.}, \bibinfo{author}{Pavone, M.}, \bibinfo{year}{2017}.
\newblock \bibinfo{title}{Risk-constrained reinforcement learning with percentile risk criteria}.
\newblock \bibinfo{journal}{Journal of Machine Learning Research} \bibinfo{volume}{18}, \bibinfo{pages}{6070--6120}.
\bibitem[{Chow et~al.(2018)Chow, Nachum, Du{\'{e}}{\~{n}}ez{-}Guzm{\'{a}}n and Ghavamzadeh}]{chow2018lyapunov}
\bibinfo{author}{Chow, Y.}, \bibinfo{author}{Nachum, O.}, \bibinfo{author}{Du{\'{e}}{\~{n}}ez{-}Guzm{\'{a}}n, E.A.}, \bibinfo{author}{Ghavamzadeh, M.}, \bibinfo{year}{2018}.
\newblock \bibinfo{title}{A lyapunov-based approach to safe reinforcement learning}, in: \bibinfo{booktitle}{Advances in Neural Information Processing Systems (NeurIPS)}, pp. \bibinfo{pages}{8103--8112}.
\bibitem[{Dalal et~al.(2018)Dalal, Dvijotham, Vecerik, Hester, Paduraru and Tassa}]{dalal2018safe}
\bibinfo{author}{Dalal, G.}, \bibinfo{author}{Dvijotham, K.}, \bibinfo{author}{Vecerik, M.}, \bibinfo{author}{Hester, T.}, \bibinfo{author}{Paduraru, C.}, \bibinfo{author}{Tassa, Y.}, \bibinfo{year}{2018}.
\newblock \bibinfo{title}{Safe exploration in continuous action spaces}.
\newblock \bibinfo{journal}{arXiv preprint arXiv:1801.08757} .
\bibitem[{Deisenroth et~al.(2013)Deisenroth, Neumann and Peters}]{deisenroth13survey}
\bibinfo{author}{Deisenroth, M.P.}, \bibinfo{author}{Neumann, G.}, \bibinfo{author}{Peters, J.}, \bibinfo{year}{2013}.
\newblock \bibinfo{title}{A survey on policy search for robotics}.
\newblock \bibinfo{journal}{Foundations and Trends in Robotics} \bibinfo{volume}{2}, \bibinfo{pages}{1--142}.
\bibitem[{Ding et~al.(2024)Ding, Wei, Zhang and Ribeiro}]{ding2024last}
\bibinfo{author}{Ding, D.}, \bibinfo{author}{Wei, C.Y.}, \bibinfo{author}{Zhang, K.}, \bibinfo{author}{Ribeiro, A.}, \bibinfo{year}{2024}.
\newblock \bibinfo{title}{Last-iterate convergent policy gradient primal-dual methods for constrained mdps}.
\newblock \bibinfo{journal}{Advances in Neural Information Processing Systems (NeurIPS)} \bibinfo{volume}{36}.
\bibitem[{Ding et~al.(2021)Ding, Wei, Yang, Wang and Jovanovic}]{ding2021provably}
\bibinfo{author}{Ding, D.}, \bibinfo{author}{Wei, X.}, \bibinfo{author}{Yang, Z.}, \bibinfo{author}{Wang, Z.}, \bibinfo{author}{Jovanovic, M.R.}, \bibinfo{year}{2021}.
\newblock \bibinfo{title}{Provably efficient safe exploration via primal-dual policy optimization}, in: \bibinfo{booktitle}{Proceedings of the International Conference on Artificial Intelligence and Statistics (AISTATS)}, \bibinfo{publisher}{{PMLR}}. pp. \bibinfo{pages}{3304--3312}.
\bibitem[{Ding et~al.(2020)Ding, Zhang, Basar and Jovanovic}]{ding2020natural}
\bibinfo{author}{Ding, D.}, \bibinfo{author}{Zhang, K.}, \bibinfo{author}{Basar, T.}, \bibinfo{author}{Jovanovic, M.}, \bibinfo{year}{2020}.
\newblock \bibinfo{title}{Natural policy gradient primal-dual method for constrained markov decision processes}.
\newblock \bibinfo{journal}{Advances in Neural Information Processing Systems (NeurIPS)} \bibinfo{volume}{33}, \bibinfo{pages}{8378--8390}.
\bibitem[{Ding et~al.(2022)Ding, Zhang, Duan, Ba{\c{s}}ar and Jovanovi{\'c}}]{ding2022convergence}
\bibinfo{author}{Ding, D.}, \bibinfo{author}{Zhang, K.}, \bibinfo{author}{Duan, J.}, \bibinfo{author}{Ba{\c{s}}ar, T.}, \bibinfo{author}{Jovanovi{\'c}, M.R.}, \bibinfo{year}{2022}.
\newblock \bibinfo{title}{Convergence and sample complexity of natural policy gradient primal-dual methods for constrained mdps}.
\newblock \bibinfo{journal}{arXiv preprint arXiv:2206.02346} .
\bibitem[{Fatkhullin et~al.(2023)Fatkhullin, Barakat, Kireeva and He}]{fatkhullin2023stochastic}
\bibinfo{author}{Fatkhullin, I.}, \bibinfo{author}{Barakat, A.}, \bibinfo{author}{Kireeva, A.}, \bibinfo{author}{He, N.}, \bibinfo{year}{2023}.
\newblock \bibinfo{title}{Stochastic policy gradient methods: Improved sample complexity for fisher-non-degenerate policies}, in: \bibinfo{booktitle}{Proceedings of the International Conference on Machine Learning (ICML)}, \bibinfo{publisher}{{PMLR}}. pp. \bibinfo{pages}{9827--9869}.
\bibitem[{Fujimoto et~al.(2018)Fujimoto, van Hoof and Meger}]{fujimoto2018deterministic}
\bibinfo{author}{Fujimoto, S.}, \bibinfo{author}{van Hoof, H.}, \bibinfo{author}{Meger, D.}, \bibinfo{year}{2018}.
\newblock \bibinfo{title}{Addressing function approximation error in actor-critic methods}, in: \bibinfo{booktitle}{Proceedings of the International Conference on Machine Learning (ICML)}, \bibinfo{publisher}{PMLR}. pp. \bibinfo{pages}{1587--1596}.
\bibitem[{Ghavamzadeh and Engel(2006)}]{ghavamzadeh2006bayesian}
\bibinfo{author}{Ghavamzadeh, M.}, \bibinfo{author}{Engel, Y.}, \bibinfo{year}{2006}.
\newblock \bibinfo{title}{Bayesian policy gradient algorithms}.
\newblock \bibinfo{journal}{Advances in Neural Information Processing Systems (NIPS)} \bibinfo{volume}{19}.
\bibitem[{Gladin et~al.(2023)Gladin, Lavrik{-}Karmazin, Zainullina, Rudenko, Gasnikov and Tak{\'{a}}c}]{gladin2023algorithm}
\bibinfo{author}{Gladin, E.}, \bibinfo{author}{Lavrik{-}Karmazin, M.}, \bibinfo{author}{Zainullina, K.}, \bibinfo{author}{Rudenko, V.}, \bibinfo{author}{Gasnikov, A.V.}, \bibinfo{author}{Tak{\'{a}}c, M.}, \bibinfo{year}{2023}.
\newblock \bibinfo{title}{Algorithm for constrained markov decision process with linear convergence}, in: \bibinfo{booktitle}{Proceedings of the International Conference on Artificial Intelligence and Statistics (AISTATS)}, \bibinfo{publisher}{{PMLR}}. pp. \bibinfo{pages}{11506--11533}.
\bibitem[{Gravell et~al.(2020)Gravell, Esfahani and Summers}]{gravell2020learning}
\bibinfo{author}{Gravell, B.}, \bibinfo{author}{Esfahani, P.M.}, \bibinfo{author}{Summers, T.}, \bibinfo{year}{2020}.
\newblock \bibinfo{title}{Learning optimal controllers for linear systems with multiplicative noise via policy gradient}.
\newblock \bibinfo{journal}{IEEE Transactions on Automatic Control} \bibinfo{volume}{66}, \bibinfo{pages}{5283--5298}.
\bibitem[{Hsieh et~al.(2019)Hsieh, Iutzeler, Malick and Mertikopoulos}]{hsieh2019convergence}
\bibinfo{author}{Hsieh, Y.}, \bibinfo{author}{Iutzeler, F.}, \bibinfo{author}{Malick, J.}, \bibinfo{author}{Mertikopoulos, P.}, \bibinfo{year}{2019}.
\newblock \bibinfo{title}{On the convergence of single-call stochastic extra-gradient methods}, in: \bibinfo{booktitle}{Advances in Neural Information Processing Systems (NeurIPS)}, pp. \bibinfo{pages}{6936--6946}.
\bibitem[{Kingma and Ba(2015)}]{kingma2015adam}
\bibinfo{author}{Kingma, D.P.}, \bibinfo{author}{Ba, J.}, \bibinfo{year}{2015}.
\newblock \bibinfo{title}{Adam: {A} method for stochastic optimization}, in: \bibinfo{booktitle}{International Conference on Learning Representations (ICLR)}.
\bibitem[{Likmeta et~al.(2020)Likmeta, Metelli, Tirinzoni, Giol, Restelli and Romano}]{likmeta2020autonomous}
\bibinfo{author}{Likmeta, A.}, \bibinfo{author}{Metelli, A.M.}, \bibinfo{author}{Tirinzoni, A.}, \bibinfo{author}{Giol, R.}, \bibinfo{author}{Restelli, M.}, \bibinfo{author}{Romano, D.}, \bibinfo{year}{2020}.
\newblock \bibinfo{title}{Combining reinforcement learning with rule-based controllers for transparent and general decision-making in autonomous driving}.
\newblock \bibinfo{journal}{Robotics and Autonomous Systems} \bibinfo{volume}{131}, \bibinfo{pages}{103568}.
\bibitem[{Lillicrap et~al.(2015)Lillicrap, Hunt, Pritzel, Heess, Erez, Tassa, Silver and Wierstra}]{lillicrap2016deterministic}
\bibinfo{author}{Lillicrap, T.P.}, \bibinfo{author}{Hunt, J.J.}, \bibinfo{author}{Pritzel, A.}, \bibinfo{author}{Heess, N.M.O.}, \bibinfo{author}{Erez, T.}, \bibinfo{author}{Tassa, Y.}, \bibinfo{author}{Silver, D.}, \bibinfo{author}{Wierstra, D.}, \bibinfo{year}{2015}.
\newblock \bibinfo{title}{Continuous control with deep reinforcement learning}.
\newblock \bibinfo{journal}{CoRR} \bibinfo{volume}{abs/1509.02971}.
\bibitem[{Liu et~al.(2021)Liu, Zhou, Kalathil, Kumar and Tian}]{liu2021policy}
\bibinfo{author}{Liu, T.}, \bibinfo{author}{Zhou, R.}, \bibinfo{author}{Kalathil, D.}, \bibinfo{author}{Kumar, P.}, \bibinfo{author}{Tian, C.}, \bibinfo{year}{2021}.
\newblock \bibinfo{title}{Policy optimization for constrained mdps with provable fast global convergence}.
\newblock \bibinfo{journal}{arXiv preprint arXiv:2111.00552} .
\bibitem[{Liu et~al.(2020)Liu, Ding and Liu}]{liu2020ipo}
\bibinfo{author}{Liu, Y.}, \bibinfo{author}{Ding, J.}, \bibinfo{author}{Liu, X.}, \bibinfo{year}{2020}.
\newblock \bibinfo{title}{{IPO:} interior-point policy optimization under constraints}, in: \bibinfo{booktitle}{{AAAI} Conference on Artificial Intelligence}, \bibinfo{publisher}{{AAAI} Press}. pp. \bibinfo{pages}{4940--4947}.
\bibitem[{Masiha et~al.(2022)Masiha, Salehkaleybar, He, Kiyavash and Thiran}]{masiha2022stochastic}
\bibinfo{author}{Masiha, S.}, \bibinfo{author}{Salehkaleybar, S.}, \bibinfo{author}{He, N.}, \bibinfo{author}{Kiyavash, N.}, \bibinfo{author}{Thiran, P.}, \bibinfo{year}{2022}.
\newblock \bibinfo{title}{Stochastic second-order methods improve best-known sample complexity of sgd for gradient-dominated functions}.
\newblock \bibinfo{journal}{Advances in Neural Information Processing Systems (NeurIPS)} \bibinfo{volume}{35}, \bibinfo{pages}{10862--10875}.
\bibitem[{Mei et~al.(2020)Mei, Xiao, Szepesvari and Schuurmans}]{mei2020global}
\bibinfo{author}{Mei, J.}, \bibinfo{author}{Xiao, C.}, \bibinfo{author}{Szepesvari, C.}, \bibinfo{author}{Schuurmans, D.}, \bibinfo{year}{2020}.
\newblock \bibinfo{title}{On the global convergence rates of softmax policy gradient methods}, in: \bibinfo{booktitle}{Proceedings of the International Conference on Machine Learning (ICML)}, \bibinfo{organization}{PMLR}. pp. \bibinfo{pages}{6820--6829}.
\bibitem[{Metelli et~al.(2018)Metelli, Papini, Faccio and Restelli}]{metelli2018policy}
\bibinfo{author}{Metelli, A.M.}, \bibinfo{author}{Papini, M.}, \bibinfo{author}{Faccio, F.}, \bibinfo{author}{Restelli, M.}, \bibinfo{year}{2018}.
\newblock \bibinfo{title}{Policy optimization via importance sampling}, in: \bibinfo{booktitle}{Advances in Neural Information Processing Systems (NeurIPS)}, pp. \bibinfo{pages}{5447--5459}.
\bibitem[{Mondal and Aggarwal(2024)}]{mondal2024last}
\bibinfo{author}{Mondal, W.U.}, \bibinfo{author}{Aggarwal, V.}, \bibinfo{year}{2024}.
\newblock \bibinfo{title}{Last-iterate convergence of general parameterized policies in constrained mdps}.
\newblock \bibinfo{journal}{arXiv preprint arXiv:2408.11513} .
\bibitem[{Montenegro et~al.(2024a)Montenegro, Mussi, Metelli and Papini}]{montenegro2024learning}
\bibinfo{author}{Montenegro, A.}, \bibinfo{author}{Mussi, M.}, \bibinfo{author}{Metelli, A.M.}, \bibinfo{author}{Papini, M.}, \bibinfo{year}{2024}a.
\newblock \bibinfo{title}{Learning optimal deterministic policies with stochastic policy gradients}, in: \bibinfo{booktitle}{Proceedings of the International Conference on Machine Learning (ICML)}.
\bibitem[{Montenegro et~al.(2025)Montenegro, Mussi, Metelli and Papini}]{montenegro2025convergence}
\bibinfo{author}{Montenegro, A.}, \bibinfo{author}{Mussi, M.}, \bibinfo{author}{Metelli, A.M.}, \bibinfo{author}{Papini, M.}, \bibinfo{year}{2025}.
\newblock \bibinfo{title}{Convergence analysis of policy gradient methods with dynamic stochasticity}, in: \bibinfo{booktitle}{Proceedings of the International Conference on Machine Learning (ICML)}.
\bibitem[{Montenegro et~al.(2024b)Montenegro, Mussi, Papini and Metelli}]{montenegro2024constraints}
\bibinfo{author}{Montenegro, A.}, \bibinfo{author}{Mussi, M.}, \bibinfo{author}{Papini, M.}, \bibinfo{author}{Metelli, A.M.}, \bibinfo{year}{2024}b.
\newblock \bibinfo{title}{Last-iterate global convergence of policy gradients for constrained reinforcement learning}, in: \bibinfo{booktitle}{Advances in Neural Information Processing Systems (NeurIPS)}.
\bibitem[{Moskovitz et~al.(2023)Moskovitz, O'Donoghue, Veeriah, Flennerhag, Singh and Zahavy}]{moskovitz2023reload}
\bibinfo{author}{Moskovitz, T.}, \bibinfo{author}{O'Donoghue, B.}, \bibinfo{author}{Veeriah, V.}, \bibinfo{author}{Flennerhag, S.}, \bibinfo{author}{Singh, S.}, \bibinfo{author}{Zahavy, T.}, \bibinfo{year}{2023}.
\newblock \bibinfo{title}{Reload: Reinforcement learning with optimistic ascent-descent for last-iterate convergence in constrained mdps}, in: \bibinfo{booktitle}{Proceedings of the International Conference on Machine Learning (ICML)}, \bibinfo{publisher}{{PMLR}}. pp. \bibinfo{pages}{25303--25336}.
\bibitem[{Nouiehed et~al.(2019)Nouiehed, Sanjabi, Huang, Lee and Razaviyayn}]{nouiehed2019solving}
\bibinfo{author}{Nouiehed, M.}, \bibinfo{author}{Sanjabi, M.}, \bibinfo{author}{Huang, T.}, \bibinfo{author}{Lee, J.D.}, \bibinfo{author}{Razaviyayn, M.}, \bibinfo{year}{2019}.
\newblock \bibinfo{title}{Solving a class of non-convex min-max games using iterative first order methods}.
\newblock \bibinfo{journal}{Advances in Neural Information Processing Systems (NeurIPS)} \bibinfo{volume}{32}.
\bibitem[{Papini et~al.(2022)Papini, Pirotta and Restelli}]{papini2022smoothing}
\bibinfo{author}{Papini, M.}, \bibinfo{author}{Pirotta, M.}, \bibinfo{author}{Restelli, M.}, \bibinfo{year}{2022}.
\newblock \bibinfo{title}{Smoothing policies and safe policy gradients}.
\newblock \bibinfo{journal}{Machine Learning} \bibinfo{volume}{111}, \bibinfo{pages}{4081--4137}.
\bibitem[{Paternain et~al.(2019)Paternain, Chamon, Calvo{-}Fullana and Ribeiro}]{paternain2019constrained}
\bibinfo{author}{Paternain, S.}, \bibinfo{author}{Chamon, L.F.O.}, \bibinfo{author}{Calvo{-}Fullana, M.}, \bibinfo{author}{Ribeiro, A.}, \bibinfo{year}{2019}.
\newblock \bibinfo{title}{Constrained reinforcement learning has zero duality gap}, in: \bibinfo{booktitle}{Advances in Neural Information Processing Systems (NeurIPS)}, pp. \bibinfo{pages}{7553--7563}.
\bibitem[{Peters and Schaal(2006)}]{peters2006policy}
\bibinfo{author}{Peters, J.}, \bibinfo{author}{Schaal, S.}, \bibinfo{year}{2006}.
\newblock \bibinfo{title}{Policy gradient methods for robotics}, in: \bibinfo{booktitle}{IEEE/RSJ International Conference on Intelligent Robots and Systems}, \bibinfo{organization}{IEEE}. pp. \bibinfo{pages}{2219--2225}.
\bibitem[{Rozada et~al.(2025)Rozada, Ding, Marques and Ribeiro}]{rozada2024deterministic}
\bibinfo{author}{Rozada, S.}, \bibinfo{author}{Ding, D.}, \bibinfo{author}{Marques, A.G.}, \bibinfo{author}{Ribeiro, A.}, \bibinfo{year}{2025}.
\newblock \bibinfo{title}{Deterministic policy gradient primal-dual methods for continuous-space constrained mdps}, in: \bibinfo{booktitle}{{AAAI} Conference on Artificial Intelligence}, pp. \bibinfo{pages}{20200--20208}.
\bibitem[{Sehnke et~al.(2010)Sehnke, Osendorfer, Rückstieß, Graves, Peters and Schmidhuber}]{SEHNKE2010551}
\bibinfo{author}{Sehnke, F.}, \bibinfo{author}{Osendorfer, C.}, \bibinfo{author}{Rückstieß, T.}, \bibinfo{author}{Graves, A.}, \bibinfo{author}{Peters, J.}, \bibinfo{author}{Schmidhuber, J.}, \bibinfo{year}{2010}.
\newblock \bibinfo{title}{Parameter-exploring policy gradients}.
\newblock \bibinfo{journal}{Neural Networks} \bibinfo{volume}{23}, \bibinfo{pages}{551--559}.
\bibitem[{Silver et~al.(2014)Silver, Lever, Heess, Degris, Wierstra and Riedmiller}]{silver2014deterministic}
\bibinfo{author}{Silver, D.}, \bibinfo{author}{Lever, G.}, \bibinfo{author}{Heess, N.M.O.}, \bibinfo{author}{Degris, T.}, \bibinfo{author}{Wierstra, D.}, \bibinfo{author}{Riedmiller, M.A.}, \bibinfo{year}{2014}.
\newblock \bibinfo{title}{Deterministic policy gradient algorithms}, in: \bibinfo{booktitle}{Proceedings of the International Conference on Machine Learning (ICML)}.
\bibitem[{Stooke et~al.(2020)Stooke, Achiam and Abbeel}]{stooke2020responsive}
\bibinfo{author}{Stooke, A.}, \bibinfo{author}{Achiam, J.}, \bibinfo{author}{Abbeel, P.}, \bibinfo{year}{2020}.
\newblock \bibinfo{title}{Responsive safety in reinforcement learning by {PID} lagrangian methods}, in: \bibinfo{booktitle}{Proceedings of the International Conference on Machine Learning (ICML)}, \bibinfo{publisher}{{PMLR}}. pp. \bibinfo{pages}{9133--9143}.
\bibitem[{Sutton and Barto(2018)}]{sutton2018reinforcement}
\bibinfo{author}{Sutton, R.S.}, \bibinfo{author}{Barto, A.G.}, \bibinfo{year}{2018}.
\newblock \bibinfo{title}{Reinforcement learning: an introduction}.
\newblock \bibinfo{publisher}{MIT Press}.
\bibitem[{Tessler et~al.(2019)Tessler, Mankowitz and Mannor}]{tessler2018reward}
\bibinfo{author}{Tessler, C.}, \bibinfo{author}{Mankowitz, D.J.}, \bibinfo{author}{Mannor, S.}, \bibinfo{year}{2019}.
\newblock \bibinfo{title}{Reward constrained policy optimization}.
\bibitem[{Todorov et~al.(2012)Todorov, Erez and Tassa}]{todorov2012mujoco}
\bibinfo{author}{Todorov, E.}, \bibinfo{author}{Erez, T.}, \bibinfo{author}{Tassa, Y.}, \bibinfo{year}{2012}.
\newblock \bibinfo{title}{Mujoco: {A} physics engine for model-based control}, in: \bibinfo{booktitle}{{IEEE/RSJ} International Conference on Intelligent Robots and Systems (IROS)}, \bibinfo{publisher}{{IEEE}}. pp. \bibinfo{pages}{5026--5033}.
\bibitem[{Uchibe and Doya(2007)}]{uchibe2007constrained}
\bibinfo{author}{Uchibe, E.}, \bibinfo{author}{Doya, K.}, \bibinfo{year}{2007}.
\newblock \bibinfo{title}{Constrained reinforcement learning from intrinsic and extrinsic rewards}, in: \bibinfo{booktitle}{IEEE International Conference on Development and Learning}, \bibinfo{organization}{IEEE}. pp. \bibinfo{pages}{163--168}.
\bibitem[{Vaswani et~al.(2022)Vaswani, Yang and Szepesv{\'{a}}ri}]{vaswani2022near}
\bibinfo{author}{Vaswani, S.}, \bibinfo{author}{Yang, L.}, \bibinfo{author}{Szepesv{\'{a}}ri, C.}, \bibinfo{year}{2022}.
\newblock \bibinfo{title}{Near-optimal sample complexity bounds for constrained mdps}, in: \bibinfo{booktitle}{Advances in Neural Information Processing Systems (NeurIPS)}.
\bibitem[{V{\'{a}}zquez{-}Abad et~al.(2002)V{\'{a}}zquez{-}Abad, Krishnamurthy, Martin and Baltcheva}]{felisia2002self}
\bibinfo{author}{V{\'{a}}zquez{-}Abad, F.J.}, \bibinfo{author}{Krishnamurthy, V.}, \bibinfo{author}{Martin, K.}, \bibinfo{author}{Baltcheva, I.}, \bibinfo{year}{2002}.
\newblock \bibinfo{title}{Self learning control of constrained markov chains - a gradient approach}, in: \bibinfo{booktitle}{{IEEE} Conference on Decision and Control (CDC)}, \bibinfo{publisher}{{IEEE}}. pp. \bibinfo{pages}{1940--1945}.
\bibitem[{Williams(1992)}]{williams1992simple}
\bibinfo{author}{Williams, R.J.}, \bibinfo{year}{1992}.
\newblock \bibinfo{title}{Simple statistical gradient-following algorithms for connectionist reinforcement learning}.
\newblock \bibinfo{journal}{Machine Learning} \bibinfo{volume}{8}, \bibinfo{pages}{229--256}.
\bibitem[{Xiong et~al.(2022)Xiong, Xu, Zhao, Liang and Zhang}]{xiong2022deterministic}
\bibinfo{author}{Xiong, H.}, \bibinfo{author}{Xu, T.}, \bibinfo{author}{Zhao, L.}, \bibinfo{author}{Liang, Y.}, \bibinfo{author}{Zhang, W.}, \bibinfo{year}{2022}.
\newblock \bibinfo{title}{Deterministic policy gradient: Convergence analysis}, in: \bibinfo{booktitle}{Proceedings of the Conference on Uncertainty in Artificial Intelligence (UAI)}, \bibinfo{publisher}{PMLR}. pp. \bibinfo{pages}{2159--2169}.
\bibitem[{Xu et~al.(2021)Xu, Liang and Lan}]{xu2021crpo}
\bibinfo{author}{Xu, T.}, \bibinfo{author}{Liang, Y.}, \bibinfo{author}{Lan, G.}, \bibinfo{year}{2021}.
\newblock \bibinfo{title}{{CRPO:} {A} new approach for safe reinforcement learning with convergence guarantee}, in: \bibinfo{booktitle}{Proceedings of the International Conference on Machine Learning (ICML)}, \bibinfo{publisher}{{PMLR}}. pp. \bibinfo{pages}{11480--11491}.
\bibitem[{Yang et~al.(2020)Yang, Kiyavash and He}]{yang2020minimax}
\bibinfo{author}{Yang, J.}, \bibinfo{author}{Kiyavash, N.}, \bibinfo{author}{He, N.}, \bibinfo{year}{2020}.
\newblock \bibinfo{title}{Global convergence and variance reduction for a class of nonconvex-nonconcave minimax problems}, in: \bibinfo{booktitle}{Advances in Neural Information Processing Systems (NeurIPS)}, pp. \bibinfo{pages}{1153--1165}.
\bibitem[{Ying et~al.(2022)Ying, Ding and Lavaei}]{ying2022dual}
\bibinfo{author}{Ying, D.}, \bibinfo{author}{Ding, Y.}, \bibinfo{author}{Lavaei, J.}, \bibinfo{year}{2022}.
\newblock \bibinfo{title}{A dual approach to constrained markov decision processes with entropy regularization}, in: \bibinfo{booktitle}{Proceedings of the International Conference on Artificial Intelligence and Statistics (AISTATS)}, \bibinfo{publisher}{{PMLR}}. pp. \bibinfo{pages}{1887--1909}.
\bibitem[{Yu et~al.(2019)Yu, Yang, Kolar and Wang}]{yu2019convergent}
\bibinfo{author}{Yu, M.}, \bibinfo{author}{Yang, Z.}, \bibinfo{author}{Kolar, M.}, \bibinfo{author}{Wang, Z.}, \bibinfo{year}{2019}.
\newblock \bibinfo{title}{Convergent policy optimization for safe reinforcement learning}, in: \bibinfo{booktitle}{Advances in Neural Information Processing Systems (NeurIPS)}, pp. \bibinfo{pages}{3121--3133}.
\bibitem[{Yuan et~al.(2022)Yuan, Gower and Lazaric}]{yuan2022general}
\bibinfo{author}{Yuan, R.}, \bibinfo{author}{Gower, R.M.}, \bibinfo{author}{Lazaric, A.}, \bibinfo{year}{2022}.
\newblock \bibinfo{title}{A general sample complexity analysis of vanilla policy gradient}, in: \bibinfo{booktitle}{Proceedings of the International Conference on Artificial Intelligence and Statistics (AISTATS)}, \bibinfo{organization}{PMLR}. pp. \bibinfo{pages}{3332--3380}.
\bibitem[{Zhao and You(2021)}]{zaho2021pgdual}
\bibinfo{author}{Zhao, F.}, \bibinfo{author}{You, K.}, \bibinfo{year}{2021}.
\newblock \bibinfo{title}{Primal-dual learning for the model-free risk-constrained linear quadratic regulator}, in: \bibinfo{booktitle}{Proceedings of the Conference on Learning for Dynamics and Control (CDC)}, \bibinfo{publisher}{PMLR}. pp. \bibinfo{pages}{702--714}.
\bibitem[{Zheng et~al.(2022)Zheng, You and Mallada}]{zheng2022constrained}
\bibinfo{author}{Zheng, T.}, \bibinfo{author}{You, P.}, \bibinfo{author}{Mallada, E.}, \bibinfo{year}{2022}.
\newblock \bibinfo{title}{Constrained reinforcement learning via dissipative saddle flow dynamics}, in: \bibinfo{booktitle}{Asilomar Conference on Signals, Systems, and Computers (ACSSC)}, \bibinfo{publisher}{{IEEE}}. pp. \bibinfo{pages}{1362--1366}.

\end{thebibliography}
